\documentclass{book}

\usepackage[utf8]{inputenc}
\usepackage[T1]{fontenc}
\usepackage{geometry}
\usepackage{fancyhdr}
\usepackage{titlesec}
\usepackage{amsmath}
\usepackage{graphicx}
\usepackage{dsfont}
\usepackage{amssymb}
\usepackage{amstext}
\usepackage{epstopdf}
\DeclareGraphicsExtensions{.eps}
\usepackage{pdfsync}
\usepackage{subfigure}

\input xy
\xyoption{all}

\usepackage{proof}

\titleformat{\chapter}[display]
{\normalfont\huge\bfseries}{\chaptertitlename\ \thechapter}{20pt}{\Huge}
\titlespacing*{\chapter}{0pt}{-50pt}{40pt}

\titleformat{\section}
{\normalfont\Large\bfseries}{\thesection}{1em}{}
\titlespacing*{\section}{0pt}{3.5ex plus 1ex minus .2ex}{2.3ex plus .2ex}

\titleformat{\subsection}
{\normalfont\large\bfseries}{\thesubsection}{1em}{}
\titlespacing*{\subsection}{0pt}{3.25ex plus 1ex minus .2ex}{1.5ex plus .2ex}

 \newtheorem{thm}{Theorem}

 \newtheorem{prop}{Proposition}
 \newtheorem{defn}{Definition}
 
 \newtheorem{exam}{Example}
 
 \newtheorem{alg}{Algorithm}
 \newtheorem{proof}{Proof} 

 \newcommand{\A}{\mathcal{A}}
 
 \newcommand{\C}{\mathcal{C}}
 
 \newcommand{\E}{\mathcal{E}}
 
 \newcommand{\G}{\mathcal{G}}
 \newcommand{\J}{\mathcal{J}}

 \newcommand{\U}{\mathcal{U}}

\begin{document}
\author{ Carlos Leandro\\
    SolidReturn Lda, Almeirim, Portugal \and Departamento de Matem\'{a}tica,\\ Instituto Superior de Engenharia de Lisboa,  Portugal\\ miguel.melro.leandro@gmail.com
 }

\title{Categorical semiotics: Foundations for Knowledge Integration}

\date{\today}
\frontmatter

\maketitle


\chapter*{Preface}
The integration of knowledge extracted from diverse models, whether described by domain experts or generated by machine learning algorithms, has historically been challenged by the absence of a suitable framework for specifying and integrating structures, learning processes, data transformations, and data models or rules. In this work, we extend algebraic specification methods to address these challenges within such a framework.

In our work, we tackle the challenging task of developing a comprehensive framework for defining and analyzing deep learning architectures. We believe that previous efforts have fallen short by failing to establish a clear connection between the constraints a model must adhere to and its actual implementation. 

Our methodology employs graphical structures that resemble Ehresmann's sketches \cite{RosickyAdamek94}, interpreted within a universe of fuzzy sets. This approach offers a unified theory that elegantly encompasses both deterministic and non-deterministic neural network designs. 

This approach leverages the inherent capabilities of sketches to integrate both deterministic and non-deterministic data structures. By adopting this strategy, we aim to capitalize on the suitability of graphical languages—commonly used in category theory and specifically employed for defining sketches—for reasoning about complex problems, structural description, and task specification and decomposition.

The use of graphical structures in our extended algebraic specification methods offers several advantages. Firstly, it provides a visual representation that facilitates understanding and communication among domain experts and researchers from various disciplines. Secondly, it enables the seamless integration of data and knowledge from different sources, thereby promoting interoperability and collaboration. Lastly, the flexibility inherent in our approach allows for the incorporation of machine learning algorithms, enabling the synthesis of new insights and models from diverse data sets.

Furthermore, we highlight how this theory naturally incorporates fundamental concepts from computer science and automata theory. Our extended algebraic specification framework, grounded in graphical structures akin to Ehresmann's sketches, offers a promising solution for integrating knowledge across disparate models and domains. By bridging the gap between domain-specific expertise and machine-generated insights, we pave the way for more comprehensive, collaborative, and effective approaches to knowledge integration and modeling.

----------------------------------------------------------------------

\noindent\textbf{Keywords:} $\Omega$-sets, formal specification, syntax, ML-algebras, monoidal logic, basic logic, divisible logic, fuzzy logic, graphic language, deduction, Bayes inference, analytic grammars, link grammars, ontology, multi-diagrams, libraries of components, refinement, parser graph, limit as multi-morphism, Ehresmann's sketches, accessible categories, specification system, logic semiotic, diferencial semiotic, temporal semiotic, concept description, similarity relation, consistent knowledge units, $\lambda$-consistent knowledge units, fuzzy computability, Turing computable, semiotic integration, knowledge integration, consequence relation.

\newpage
\tableofcontents

\mainmatter
\chapter*{Introduction}

A model can be conceptualized as a system of sets with associated relations that impose constraints on the set system. When considering a class of models that share similar structures, along with the structure-preserving maps between them, we enter the realm of category theory. For example, a relational database schema can be seen as a specification of a class of systems of sets. In this context, the category of models is defined by all possible database states and the transformations between these states. The relational database schema provides constraints that govern the state of the database.

At the heart of an information system lies a collection of databases and sets of data transformations, often represented as workflows. From a modern perspective, a database serves as an internal model of a fragment of the real world. Meanwhile, the transformations offer diverse methods for constructing views of this fragment and integrating its various aspects. A crucial step in designing an effective information system is to specify the universe and its views in abstract, formalized terms suitable for semantic refinement. This specification should be adaptable for low-level system specifications and capable of accommodating improvements through the introduction of new knowledge about the data stored in the system over its lifetime. Such a data structure specification is referred to as semantic modeling.

Semantic modeling aims to condense information and process descriptions into a comprehensible format suitable for communication between database tools or designs, such as between data mining processes and data analysts. Graphic languages are often the preferred choice for this purpose, and significant effort has been devoted to the development of graphic denotational systems. The history of graphic notations spans various scientific and engineering disciplines, resulting in a rich tapestry of visual modeling languages and methods. Examples include ER diagrams and their numerous variants, OOA\&D schemas in myriad versions, and UML, which itself encompasses a variety of notations.

Our objective is to elucidate the fundamental semantic foundations of graphic languages and present an integrated framework that allows for the consistent approach and integration of different languages and their semantics. By doing so, we aim to bridge the gap between theoretical models and practical applications, facilitating more effective communication, collaboration, and system design in the realm of semantic modeling and information systems.

A well-designed graphic language should be formalizable, sufficiently expressive to capture all the intricacies of the real world, and suitable for semantic refinement. We are particularly interested in using the same language to model both deterministic and nondeterministic structures, such as data structures and models generated using machine learning algorithms. In our view, Category theory offers the most expressive and formal approach to deterministic graphic specification.

Category theory generalizes the use of graphic language to specify structures and properties through diagrams. These categorical techniques provide powerful tools for formal specification, structuring, model construction, and formal verification across a wide range of systems, as evidenced by numerous papers. Data specification requires a finite, effective, and comprehensive presentation of complete structures. This type of methodology was explored in Category theory for algebraic specification by Ehresmann. He developed sketches as a specification methodology for mathematical structures and presented them as an alternative to the string-based specification employed in mathematical logic.

The functional semantics of sketches is sound in the informal sense, preserving by definition the structure given in the sketch. The analogy to the semantics of traditional model theory is close enough that sketches and their models fit the definition of an "institution" (see \cite{BurstallGoguen83}), which is an abstract notion of a logical system comprising syntactic and semantic components. The soundness of semantics appears trivial when contrasted with the inductive proof of soundness that occurs in string-based logic, as the semantics functor is not defined recursively.

Sketch specifications enjoy a unique combination of rigor, expressiveness, and comprehensibility. They can be used for data modeling, process modeling, and metadata modeling, providing a unified specification framework for system modeling. However, the sketch structure requires us to adopt a global perspective of the system, making it impossible to decompose a specification problem into subproblems. This poses challenges in specifying a large system, especially when considering the interaction between subsystems or components, which is a common practice in engineering.

To illustrate this problem further, consider the specification of workflows (for more details, see \cite{EhrigHeckelBaldanCorradini06}). A workflow describes a business process in terms of tasks and shared resources. Such descriptions are essential, for example, when interoperability of workflows from different organizations is a concern, such as when integrating applications from various enterprises over the Internet \cite{EhrigHeckelBaldanCorradini06}. A workflow is a net \cite{vanDerAalst98}, sometimes a Petri net, that satisfies certain structural constraints and corresponding soundness conditions.

Typically, the methodology used to specify workflows involves an associated library of components. An interorganizational workflow is modeled as a set of such workflow nets \cite{vanDerAalst99}, connected through additional places for asynchronous communication and synchronization requirements on transitions. The difficulty of applying sketches to this type of task stems from the way composition is defined in the category of sets.

We aim to interpret a workflow as an arrow, where the gluing of different workflows must also be interpreted as an arrow and must always be defined. The net structure defined by a workflow is called an oriented multi-graph since its components are relations linking families defined by sets of entities or data structures.

To achieve our goal, we extend the syntax of sketches to multi-graphs and define their models on a class of fuzzy relations (see \cite{Johnstone02}). In this extended framework, multi-graph nodes are interpreted as relations. To formalize this extension, we begin by defining a library as the syntactic structure of admissible configurations using components from that library. This approach is grounded in the notion of a component, which is used to define relationships between two families of entities and can be organized in a hierarchy of complexity. If a component is not atomic, it is defined by the composition of other components. We view the set of admissible configurations as the graphic language defined by the library structure.

A model of a library is a map from the library multi-graph to the structure associated with the class of relations defined using multi-valued logic. An interpretation for an admissible configuration is defined for each library model. It is the limit of the multi-graph, in the category of $\Omega$-sets, resulting from applying a library model to an admissible multi-graph. This framework appears to be suitable for the definition and study of graphic-based logics, especially when interpreting one of its nodes as the set of truth values.

We employ libraries as a means to define the lexicon of the language used in describing a domain. However, the category defined by library models and natural transformations is not an accessible category (see \cite{RosickyAdamek94}). It cannot be axiomatized by a basic theory in first-order logic. This limitation means that classical Ehresmann sketches do not have sufficient expressive power to specify the category of library models defined in multi-valued logic with more than three truth values.

To formalize linguistic structures, Chomsky proposed in \cite{chomsky57} that a language can be viewed as a set of grammatically correct sentences that are possible within that language. The objective of defining a language is thus to explicitly characterize this set of grammatical sentences through the use of a formal grammar. The two primary categories of grammar are \emph{generative grammars}, which consist of rules for generating elements of a language, and \emph{analytic grammars}, which consist of rules for analyzing a structure to determine its membership in the language. In our approach to defining a graphic language based on libraries of components, we utilize a multi-graph as the analytic grammar of the language.

It's important to note that a generative grammar does not necessarily correspond to the algorithm used to parse the generated language. In contrast, an analytic grammar aligns more directly with the structure and semantics of a parser designed for the language. Examples of analytic grammar formalisms include top-down parsing languages (TDPL) \cite{Birman70}, link grammars \cite{TemperleySleator91}, and parsing expression grammars \cite{Ford04}.

Our extension to the syntax of sketches is grounded in the principles of Link grammars, a syntactic theory introduced by Davy Temperley and Daniel Sleator in \cite{TemperleySleator91}. Unlike traditional tree-like hierarchical structures, Link grammars establish relations between pairs of words, resembling the dependency grammars pioneered by Lucien Tesni\`{e}re in the 1960s \cite{Tesniere59}. Designed primarily for linguistic applications, Link grammar serves as an analytic grammar that derives syntactic structures by examining the positional relationships between pairs of words.

The model developed by Temperley and Sleator for English showcases the system's comprehensive coverage of various linguistic phenomena in the language. In this system, a sentence is deemed grammatically correct if it is possible to establish links between all words based on the predefined link requirements for each word as specified in the lexicon. In this context, each link represents a syntactic relation between words. It has been demonstrated in \cite{TemperleySleator91} that link grammars possess the expressive power equivalent to that of context-free grammars.

We introduce the concept of a \textit{sign system} as an extension of the Ehresmann sketch notion. Goguen provides a formal definition for a closely related concept in \cite{Goguen99}. In our formulation, a sign system is defined by a library comprised of a multi-graph, a set of commutative multi-diagrams, a collection of limit cones, and a set of colimit cocones. Within this framework, both the limit and the colimit of a multi-diagram are understood as relations established within the context of fuzzy set theory. 

We define a \textit{semiotic system} as a pair comprising a sign system and one of its associated models. Such semiotic systems can be interpreted as institutions in the sense described by Goguen \cite{BurstallGoguen83}.

Our primary objective is to establish a mathematically rigorous theory of semiotics grounded in the Ehresmann sketch notion. Our approach seeks to refine and internalize the formalizations introduced by the \textit{Vienna Circle} in their \textit{International Encyclopedia of Unified Science}. To achieve this, we delineate the broader field of \textit{Semiotics} into three distinct branches:

\begin{itemize}
    \item \textit{Semantics}: Concerned with the relationship between signs and the entities or concepts they represent.
    \item \textit{Syntactics}: Focuses on the interrelations among signs within formal structures.
    \item \textit{Pragmatics}: Examines the impact of signs on the processes or contexts in which they are used.
\end{itemize}

By partitioning semiotics into these three branches, we aim to provide a comprehensive and structured framework that captures the intricate relationships between signs and their broader implications.

Signs manifest as integral components within sign systems. Many signs are intricate constructs, amalgamated from simpler, foundational signs. In this light, signs can be perceived as structures that are defined using other signs. Sign systems serve to systematically organize and represent the relationships between these signs. Assigning a sign to an entity functions as a mechanism to name or attribute properties to that entity; for instance, two entities labeled with identical signs are deemed identical. We define a model of a sign system as a coherent process for labeling entities within a predefined universe while preserving the relational attributes between signs.

Drawing inspiration from Peirce's approach to \textit{Semiotics}, we correlate entities from the universe of discourse with specific signs. This is achieved by labeling certain objects or morphisms within a topos with signs stipulated in a library specification, using a library model. If the object classifier within the topos possesses a monoidal logic structure, and its operators are annotated with signs from a particular library, we can leverage this library to articulate monoidal graphical logics. Within the context of a semiotic system, a relation is perceived as a configuration that, when interpreted, corresponds to a multi-morphism targeted towards an object annotated with an associated monoidal logic structure.

This approach facilitates the expansion of traditional logical concepts to graphical logics that are associated with a semiotic system. Graphic relations can be employed to formulate queries within the semiotic framework. A response to such queries is a fiber over its source, derived from its interpretation, defined by all the "points" transformed into truths by the semiotic-associated ML-algebra. In this scenario, a "set" of points aligns with a graphic relation if each point, as per the relation's interpretation, corresponds to a truth within the semiotic monoidal logic. Given that monoidal logics can be perceived as fuzzy or multi-valued logics, the logics definable within this framework essentially emerge as fuzzy graphical logics. This extension enriches the logical framework underpinning Ehresmann sketches and facilitates the fuzzification of the notion of relation evaluation, a critical consideration for practical applications.

Everyday activities generate a wealth of information, which is often stored across various databases within information systems. A query to an information system can be perceived as a lens through which the data stored in the system is viewed, typically presented as a dataset. Frequently, this information serves as a foundation for deriving new insights about reality. Given the vast volume of data stored, this transformation process needs to be automated, and this is a primary objective of fields within Artificial Intelligence, such as Machine Learning, as discussed in \cite{MichalskiMitchellCarbonell86}.

Within the framework of fuzzy set theory, a dataset can be represented by a relation. The interpretation of a word within a logical semiotic is termed a model for the dataset if the relation aligns with the multi-diagram limit. If we conceptualize a dataset as a method for encoding data, its model can be viewed as a representation of knowledge in a graphical language associated with a specific logical semiotic. This relationship between data and knowledge becomes valuable when we define the notion of a structure being $\lambda$-consistent with a relation, where $\lambda$ quantifies the degree of similarity between the interpretation of a word and the concept encoded within the dataset. In this context, a dataset being $\lambda$-consistent with a diagram encapsulates the idea of approximation.

If the data contained within a dataset is $\lambda$-consistent with a set of diagrams, we refer to the semiotic defined by this set of diagrams as a theory that is $\lambda$-consistent with the data. Naturally, this concept can be extended to databases. The knowledge encapsulated within a database can be encoded within a semiotic, and the quality of this knowledge is determined by how well its model provides accurate approximations to the data associated with the tables that can be generated from a database state. 

Given that different human specialists or machine learning algorithms often articulate the extracted knowledge using varied languages, the challenge of knowledge integration essentially becomes a problem of semiotic integration. The objective of this integration process is to construct a unified semiotic that harnesses all available knowledge while maintaining optimal performance. We present methodologies for merging multiple distinct theories. However, this process sometimes also necessitates the integration of languages and the associated logics, a task that is streamlined through our approach.

\textbf{Overview:}

In chapter \ref{monoidal logics}, we initiated our discussion by outlining a partially-ordered monoidal structure for the set of truth-values utilized in defining our graphic logics. The language underpinning this logic is rooted in potential circuit configurations, facilitated by libraries of components—a structure detailed in chapter \ref{specifying libraries}. To provide a framework for modeling these libraries, we employed a class of relations defined between sets and assessed within a multi-valued logic. This approach is elaborated upon in chapter \ref{multi-morphisms}. To facilitate this, we introduced a definition for the composition of relations that is compatible with circuit assembly. In this context, composition should be perceived as a comprehensive operator within the class of relations, relaxing the strictures of diagram equality as evaluated in a multi-valued logic. This allows us to present a generalized version of the concept of commutative diagrams, as explored in chapter \ref{multidiagrams}. Further, in chapter \ref{bayesian inference}, we define a version of Bayesian inference tailored for fuzzy logics.

In chapter \ref{modeling libraries}, we conceptualize the language defined by libraries as a collection of circuits that are amenable to the plug-in operation. Here, each word represents a string of component labels or signs and establishes a relation between a set of input requirements and a set of output structures, both identified through sets of signs. We delve deeper into how libraries are modeled using the class of relations in chapter \ref{modeling libraries}. The descriptive power of languages formulated using libraries enables the definition of structures that are not attainable through basic first-order theory. This limitation is discussed in chapter \ref{descritive power}, where we highlight that the category defined using library models is inaccessible. To leverage the expressive power of Ehresmann sketches, we enriched the structure of a library with a framework resembling an Ehresmann sketch. This specification tool, termed a sign system or specification system, is discussed in chapter \ref{Sings Semiotics}. We utilized multi-diagrams, respecting library constraints, to achieve this enrichment, with limits and colimits interpreted as multi-morphisms. This concept is further elaborated upon in chapter \ref{multidiagrams}, focusing on the evaluation of diagram commutativity within a multi-valued logic. A specification system equipped with a fixed model is termed a semiotic system. We conclude this chapter with an illustrative example.

In chapter \ref{logics}, we employ sign systems to specify fuzzy logics, which are semiotics endowed with unique structures. To achieve this, we impose interpretations on some of the sign system signs, enabling the interpretation of words as evaluations of relations within a monoidal logic. When certain signs are interpreted as operators of ML-algebras, the library is termed a logic library, and the associated language is referred to as a logic language. We formalize these concepts and delineate conditions under which a diagram defines a relation and an equation.

However, certain mathematical problems necessitate more expressive languages than those defined using libraries. To enhance the expressive capabilities of libraries, we introduce Lagrangian syntactic operators. An illustrative example is provided in chapter \ref{synopt}, which enables the definition of Differential Semiotics.

In chapter \ref{Concept description}, we elucidate the process of evaluating relations within a semiotic and use this framework to define the concept of the level of consistency of a relation. This notion is extended to relations that are $\lambda$-consistent with words in a semiotic. We emphasize that a diagram defining a relation can be viewed as a query to the semiotic. In this context, a $\lambda$-answer is a structure that is $\lambda$-consistent with the query. These concepts are employed in chapter \ref{fuzzy computability} to define bottom and upper presentations for a structure $A$ within a semiotic: the bottom presentation represents the most detailed structure in $A$ encoded in the semiotic language, while the upper presentation encapsulates $A$ and is coded in the fixed language. These notions can be perceived as two approximations to the concept, aligning with Pawlak's spirit in the standard version of rough set theory. They establish interior and closure operators for structures, enabling the formulation of a formal topology. This topology is instrumental in defining an inference system for word evaluations within a semiotic system, grounded in the properties of ML-algebras.

Chapter \ref{integration} is devoted to the integration of semiotics and models for concepts. The objective is to construct a unified system that harnesses all available knowledge, thereby enhancing concept description by amalgamating diverse descriptions of a concept expressed in different languages. For specification purposes, a string-based modal logic is introduced in Chapter \ref{reasoning}, where propositional variables are interpreted as diagrams defined within a semiotic. This serves as a meta-language for reasoning about models of concepts and knowledge.

\chapter{Monoidal logics}\label{monoidal logics}

The ubiquity of fuzziness in practical problems underscores its importance as a rule rather than an exception. Extensive research has been conducted on fuzzy sets and their associated fuzzy logics. In this context, we are particularly interested in exploring the potential of extending the data specification paradigm using Ehresmann sketches to accommodate fuzzy scenarios. The motivation for this paper stems from the introduction to $\Omega$-Categories provided in \cite{TholenClementinoHofmann03} and the concept of $\Omega$-Sets presented in \cite{Valverde85}.

Ulrich H\"{o}hle pioneered Monoidal Logic in 1995 with the aim of establishing a unified framework for several first-order non-classical logics. These include Linear logic, Intuitionistic logic, and Lukasiewicz logic. A Monoidal Logic can be understood as a Full Lambek calculus augmented with exchange and weakening properties. We hypothesize that problems related to the specification of structures, particularly those involving libraries of components, can be formulated within a framework grounded in set theory enriched with a monoidal logic. Such an approach offers a promising avenue for addressing the complexities and nuances inherent in fuzzy systems, thereby paving the way for more robust and flexible data specification methodologies.

Recall that an algebra $(\Omega, \otimes, \leq, 1)$ is a partially-ordered monoid if $(\Omega, \otimes, 1)$ forms a monoid and $\leq$ is a partial order on $\Omega$ such that the operator $\otimes$ is monotone increasing, meaning:
\[
x \leq x' \text{ and } y \leq y' \text{ imply } x \otimes y \leq x' \otimes y'.
\]

An algebra $(\Omega, \otimes, \setminus, /, \leq, 1)$ is termed a resituated partially-ordered monoid if $(\Omega, \otimes, \leq, 1)$ is a partially-ordered monoid. Furthermore, the following condition must hold for all \( x, y, z \in \Omega \):
\[
x \otimes y \leq z \Leftrightarrow y \leq x \setminus z \Leftrightarrow x \leq z / y.
\]
This condition is referred to as the \emph{law of residuation}, where \( / \) and \( \setminus \) denote the right and left residuals of \( \otimes \), respectively.

Any residuated partially-ordered monoid \( \Omega \) for which \( (\Omega, \leq) \) forms a lattice and \( (\Omega, \otimes) \) possesses a unit is termed a \emph{residuated lattice}. To be more precise, an algebra
\[
(\Omega, \vee, \wedge, \otimes, \setminus, /, 1)
\]
is considered a residuated lattice if it satisfies the following conditions:
\begin{enumerate}
  \item \( (\Omega, \otimes, 1) \) is a monoid, where \( \setminus \) and \( / \) serve as the right and left residuals of \( \otimes \), respectively, and
  \item \( (\Omega, \vee, \wedge) \) forms a lattice.
\end{enumerate}
When \( \otimes \) is commutative, the structure is referred to as a \emph{commutative residuated lattice}. In any commutative residuated lattice, the equation \( x \setminus y = y/x \) holds for all \( x, y \). In such instances, we employ the symbol \( \Rightarrow \) and represent \( x \Rightarrow y \) instead of \( x \setminus y \) (or \( y/x \)). Furthermore, the commutative residuated lattice is denoted as \( (\Omega, \vee, \wedge, \otimes, \Rightarrow, 1) \).

\begin{defn}
A \emph{ML-algebra} is defined as a bounded commutative residuated lattice, where \(1 = \top\). Formally, it is a system \((\Omega, \otimes, \Rightarrow, \vee, \wedge, \perp, \top)\) that satisfies the following conditions:
\begin{enumerate}
  \item \((\Omega, \otimes, \top)\) forms a commutative monoid,
  \item \(x \otimes \top = x\) for every \(x \in \Omega\),
  \item \((\Omega, \vee, \wedge, \perp, \top)\) is a bounded lattice, and
  \item The residuation property holds:
  \[
  \text{For all } x, y, z \in \Omega, \quad x \leq y \Rightarrow z \text{ if and only if } x \otimes y \leq z.
  \]
\end{enumerate}
\end{defn}

In this paper, we assume that the ML-algebra \(\Omega\) is non-trivial, meaning \(\top \neq \perp\).

An equivalent structure to an ML-algebra is presented in \cite{TholenClementinoHofmann03} as a commutative and unital quantale, where \(\Omega\) is a complete lattice equipped with a symmetric and associative tensor product \(\otimes\), having unit \(\top\) and a right adjoint \(\Rightarrow\). Considering \(\Omega\) as a thin category, \(\Omega\) is termed a \emph{symmetric monoidal-closed} category.

Logics that have refinements of ML-algebras as their models are termed \emph{monoidal logics}. In many-valued logics, such as fuzzy logics, \(\otimes\) serves as the standard truth degree function for the conjunction connective. Given that the operator \(\otimes\) is monotone and has a right adjoint, the following propositions can be derived:

\begin{prop}
For an ML-algebra, the following properties hold:
\begin{enumerate}
  \item If \(y \leq z\), then \(x \otimes y \leq x \otimes z\),
  \item \(x \leq y\) if and only if \((x \Rightarrow y) = \top\), and
  \item \(x \Rightarrow z = \bigvee\{y : x \otimes y \leq z\}\).
\end{enumerate}
\end{prop}

Additionally, we have:

\begin{prop}\cite{CastroKlawonn95}\label{prop:implic}
For any ML-algebra, the following equalities hold for all \(x, y, z \in \Omega\):
\begin{enumerate}
  \item \(x \otimes (x \Rightarrow y) \leq x \wedge y\), and
  \item \((x \Rightarrow y) \otimes (y \Rightarrow z) \leq x \Rightarrow z\).
\end{enumerate}
\end{prop}

Every non-trivial Heyting algebra, denoted by \(\Omega\), where \(\otimes = \wedge\) and \(\top\) is the top element, serves as an example of an ML-algebra. Specifically, the two-element chain \(2 = \{ \text{false} < \text{true} \}\) equipped with the monoidal structure defined by "and" and "true" is an instance of an ML-algebra.

The complete real half-line, denoted by \(P = [0, \infty]\), when endowed with the categorical structure induced by the relation \(\leq\), admits several intriguing monoidal structures. If \(\otimes = \wedge = \max\), then it forms a Heyting algebra. Another viable choice for \(\otimes\) is \(+\). Notably, in this case, the right adjoint \(\Rightarrow\) is represented by the truncated minus operation: \(x \Rightarrow y = \max\{v - u, 0\}\).

\begin{exam}[t-norm based fuzzy logic]
A t-norm, denoted by \(\otimes\), is a function used to establish an ML-algebra structure on the real unit interval \([0,1]\). By considering the unit interval as the set of truth values, we can formulate a monoidal logic using a t-norm. The residuum of \(\otimes\) is defined as the operation \(x \Rightarrow y = \max\{z | x \otimes z \leq y\}\). In addition to the aforementioned, other truth functions are deemed essential in fuzzy logic, such as the \emph{weak conjunction} \(x \cdot y = \min(x,y)\) and \emph{weak disjunction} \(x + y = \max(x,y)\). In the subsequent discussion, we focus on interpreting a fuzzy logic using an ML-algebra \(([0,1], \otimes, \Rightarrow, \vee, \wedge, 0, 1)\), where \(\otimes\) represents a continuous t-norm and \(\Rightarrow\) symbolizes its residuum. The lattice structures are further defined by \(x \vee y = \min(x,y)\) and \(x \wedge y = \max(x+y, 1)\).

The following are important examples of fuzzy logics interpreted within ML-algebras defined by specific continuous t-norms:
\begin{itemize}
  \item \textbf{{\L}ukasiewicz} logic, which is defined using the t-norm \(x\otimes y=\max(0,x+y-1)\) and its residuum:
  \[
  x\Rightarrow y=\min(1,1-x+y)
  \]
  
  \item \textbf{G\"{o}del logic}, defined by the t-norm \(x\otimes y=\min(x,y)\) and its residuum:
  \[
  x\Rightarrow y= \left\{
                                      \begin{array}{cl}
                                        1 & \text{if } x\leq y \\
                                        y & \text{otherwise}\\
                                      \end{array}
                                    \right.
  \]
  
  \item \textbf{Product logic}, defined with the t-norm \(x\otimes y=x \cdot y\) and its residuum:
  \[
  x\Rightarrow y= \left\{
                                      \begin{array}{cl}
                                        1 & \text{if } x\leq y \\
                                        \frac{y}{x} & \text{otherwise}\\
                                      \end{array}
                                    \right.
  \]
\end{itemize}
\end{exam}

Particularly crucial to this study are the fundamental logics, which correspond to instances of ML-algebras that are divisible, meaning they satisfy the property:
\[
x \otimes (x \Rightarrow y) = x \wedge y.
\]
Examples of such logics include classic Boolean logic and fuzzy logics like Product, G\"{o}del, and {\L}ukasiewicz. It's worth noting that most of the examples presented subsequently are constructed using Product logics with the natural order on the interval \([0,1]\).

For the subsequent discussion, we define:
\[
a \Leftrightarrow b := (a \Rightarrow b) \otimes (b \Rightarrow a) \quad \text{and} \quad \neg a := a \Rightarrow \bot.
\]

Consider a finite family of ML-algebras given by
\[
(\Omega_i, \otimes_i, \Rightarrow_i, \vee_i, \wedge_i, \perp_i, \top_i)_{i \in I}.
\]
The product of these ML-algebras is defined as the ML-algebra
\[
\left(\Pi_{i \in I} \Omega_i, \otimes, \Rightarrow, \vee, \wedge, \perp, \top\right)
\]
where:
\begin{enumerate}
  \item \(\Pi_{i \in I} \Omega_i\) denotes the Cartesian product of the sets of truth values,
  \item \((\lambda_1, \lambda_2, \ldots, \lambda_n) \otimes (\alpha_1, \alpha_2, \ldots, \alpha_n) = (\lambda_1 \otimes \alpha_1, \lambda_2 \otimes \alpha_2, \ldots, \lambda_n \otimes \alpha_n)\),
  \item \((\lambda_1, \lambda_2, \ldots, \lambda_n) \Rightarrow (\alpha_1, \alpha_2, \ldots, \alpha_n) = (\lambda_1 \Rightarrow \alpha_1, \lambda_2 \Rightarrow \alpha_2, \ldots, \lambda_n \Rightarrow \alpha_n)\),
  \item \((\lambda_1, \lambda_2, \ldots, \lambda_n) \vee (\alpha_1, \alpha_2, \ldots, \alpha_n) = (\lambda_1 \vee \alpha_1, \lambda_2 \vee \alpha_2, \ldots, \lambda_n \vee \alpha_n)\),
  \item \((\lambda_1, \lambda_2, \ldots, \lambda_n) \wedge (\alpha_1, \alpha_2, \ldots, \alpha_n) = (\lambda_1 \wedge \alpha_1, \lambda_2 \wedge \alpha_2, \ldots, \lambda_n \wedge \alpha_n)\),
  \item \(\bot = (\bot_1, \bot_2, \ldots, \bot_n)\), and
  \item \(\top = (\top_1, \top_2, \ldots, \top_n)\).
\end{enumerate}
This composite structure gives rise to two types of morphisms: the projections
\[
\pi_j: \Pi_{i \in I} \Omega_i \rightarrow \Omega_j,
\]
and the upper interpretations
\[
\begin{array}{cccl}
  \top_j: & \Omega_j & \rightarrow & \Pi_{i \in I} \Omega_i \\
          & \alpha_j & \mapsto & (\bot_1, \bot_2, \ldots, \alpha_j, \ldots, \bot_n)
\end{array}
\]
It's important to note that the upper interpretation serves as the right inverse to the projection:
\[
\pi_j \circ \top_j = \text{id}_{\Omega_j}.
\]
We will utilize this structure as a means for integrating ML-logics.

\chapter{Multi-morphisms}\label{multi-morphisms}

A remarkable result in Category Theory, as presented by M. Makkai in \cite{ParMakkai89}, reveals that the arrow specification language is fully expressive. This implies that any construction with a formal semantic meaning can also be articulated using the arrow language. Furthermore, when the fundamental objects of interest are described by arrows, it often transpires that many derived objects can be naturally constructed through arrows \cite{PissenDiskinKadish99}. To define the universe we intend to explore, it is both necessary and sufficient to delineate our understanding of morphisms between the objects within this universe.

Our semantic modeling universe should essentially mirror \( Set \) but possess the capability to represent soft structures defined through monoidal logics. Let \( \Omega \) be a set equipped with an ML-algebraic structure \( (\Omega, \otimes, \Rightarrow, \vee, \wedge, \perp, \top) \). We define our universe as \( Set(\Omega) \), which consists of \( \Omega \)-sets. An \( \Omega \)-set, denoted by \( \alpha: A \), refers to a set \( A \) equipped with an \( \Omega \)-valued map
\[
[\cdot=\cdot]: A \times A \rightarrow \Omega,
\]
that is symmetric and transitive. This symmetry and transitivity are expressed by the conditions
\[
[a=b] = [b=a] \quad \text{and} \quad [a=b] \otimes [b=c] \leq [a=c],
\]
for all \( a, b, c \in A \). Such a structure is termed a \emph{similarity} in \( A \). Greek letters will be employed to denote \( \Omega \)-sets. We use the notation \( \alpha: A \) to signify an \( \Omega \)-set characterized by the set \( A \) and the similarity \( [\cdot=\cdot]_\alpha \). This can be interpreted as a relation evaluated in \( \Omega \) or a distribution in \( A \times A \). The diagonal of this fuzzy relation is instrumental in defining fuzzy sets with support \( A \). For each \( \Omega \)-set \( \alpha: A \) and \( a \in A \), we define
\[
[a]_\alpha = [a=a]_\alpha,
\]
and term it the \emph{extent} of \( a \). Consequently, \( [\cdot]_\alpha: A \rightarrow \Omega \) serves as a representation for the fuzzy set \( \alpha \) encoded through the similarity \( [\cdot=\cdot]_\alpha \). An element \( a \) is deemed \emph{globally} present in \( \alpha: A \) if \( [a]_\alpha = \top \).

Note that every set \( A \) naturally possesses an \( \Omega \)-set structure defined by the equality \( = \) in \( A \). That is, the similarity is defined as:
\[
[a=b]_A = \left\{
        \begin{array}{cc}
          \top & \text{if } a=b \\
          \bot & \text{if } a\neq b \\
        \end{array}
      \right..
\]
The crisp similarity \( [a=b]_A \), characterized by the equality in \( A \), is denoted as \( 1_A \).

Entities within an \( \Omega \)-set \( \alpha: A \) are defined by a set of attributes \( (A_i)_{i \in I} \) if \( A = \Pi_{i\in I} A_i \). Given \( \bar{x} \in \Pi_{i\in I} A_i \), many values associated with some of these attributes are often "non-observable" or unknown. Consequently, we differentiate between two types of attributes: \emph{observable attributes} and \emph{non-observable attributes}. Let \( (A_i)_{i\in L} \) represent a set of observable attributes in \( A \), where \( L \subseteq I \). We define an observable \( \Omega \)-set of \( \alpha:A = \Pi_{i\in I} A_i \) as the \( \Omega \)-set \( \beta:B = \Pi_{i\in L} A_i \) such that
\[
[\bar{a}=\bar{b}]_\beta = \bigvee_{\substack{\bar{x}=(\bar{c},\bar{a}), \\ \bar{y}=(\bar{d},\bar{b}) \in A}} [\bar{x}=\bar{y}]_\alpha.
\]

\begin{defn}[Observable description]
If \( \alpha:A \) is an \( \Omega \)-set with a set of observable attributes \( (A_i)_{i\in L} \), we refer to any \( \bar{a} \in \Pi_{i\in L} A_i \) as an \emph{observable description} for an entity in \( \alpha \).
\end{defn}

We define a multi-morphism in \( Set(\Omega) \) as a tracking morphism between \( \Omega \)-sets \( \alpha:A \) and \( \beta:B \) as a map
\[
f: A \times B \rightarrow \Omega
\]
(usually referred to as an \( \Omega \)-map or an \( \Omega \)-matrix in \cite{TholenClementinoHofmann03}). If \( f \) is a multi-morphism from \( \alpha:A \) to \( \beta:B \) in \( Set(\Omega) \), we denote this by \( f: A \rightharpoonup B \), indicating \( A \) as the source and \( B \) as the target of \( f \). For observable descriptions \( \bar{a} \) and \( \bar{b} \) representing entities in \( \alpha:A \) and \( \beta:B \), respectively, we define
\[
f(\bar{a},\bar{b}) = \bigvee_{\substack{\bar{x}=(\bar{c},\bar{a}), \\ \bar{y}=(\bar{d},\bar{b})}} f(\bar{x},\bar{y}).
\]
The complete partial order on the ML-algebra \( \Omega \) induces a complete partial order on the set of multi-morphisms. Given two multi-morphisms \( f,g: A \times B \rightarrow \Omega \) between \( \Omega \)-sets \( \alpha:A \) and \( \beta:B \) in \( Set(\Omega) \), we write \( f \leq g \) if \( f(a,b) \leq g(a,b) \) for every \( (a,b) \in A \times B \).

Graphically, a multi-morphism
\[
f: A_0 \times A_1 \times A_2 \rightharpoonup A_3 \times A_4 \times A_5
\]
is represented in Fig. \ref{multiarrow} by a multi-arrow.
\begin{figure}[h]
 \[
 \small
\xymatrix @=5pt {
&&&*+[o][F-]{f}\ar `r[rd][rd]\ar `r[rrd][rrd]\ar `r[rrrd][rrrd]&&&\\
 A_0\ar `u[urrr][urrr]&A_1\ar `u[urr][urr]& A_2\ar `u[ur][ur]&&A_3&A_4&A_5
 }
\]
\caption{Multi-arrow representation of a multi-morphism.}\label{multiarrow}
\end{figure}
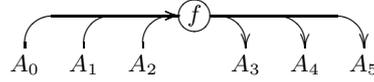
Here, the sources of the multi-arrow are \( A_0, A_1, \) and \( A_2 \), while the targets are \( A_3, A_4, \) and \( A_5 \).

We classify multi-morphisms that preserve entity evaluation in \( \Omega \) as follows:

\begin{defn}[Total multi-morphism]\label{total}
A multi-morphism \( f: A \rightharpoonup B \) is \emph{total} in \( \alpha:A \) if 
\[
[a]_\alpha = \bigvee_{b} f(a,b),
\]
for every \( a \in A \).
\end{defn}

\begin{defn}[Faithful multi-morphism]\label{faithful}
A multi-morphism \( f: A \rightharpoonup B \) is \emph{faithful} in \( \beta:B \) if 
\[
[b]_\beta = \bigvee_{a} f(a,b),
\]
for every \( b \in B \).
\end{defn}

Note that for every \( \Omega \)-set \( \alpha: \prod_{i \in I} A_i \), we can use the similarity diagonal to define a multi-morphism by selecting a set of source sets \( (A_s)_{s \in S} \) and a set of target sets \( (A_t)_{t \in T} \) with disjoint indexes, i.e., \( S \cap T = \emptyset \). This multi-morphism is given by a map
\[
g(\bar{x},\bar{y},\bar{z}) = \bigvee_{\bar{y} \in I \setminus (S \cup T)} [\bar{x},\bar{y},\bar{z}]_{\prod_{i \in I} A_i},
\]
for every \( \bar{x} \in \prod_{s \in S} A_s \) and \( \bar{z} \in \prod_{t \in T} A_t \). This defines
\[
g: \prod_{s \in S} A_s \rightharpoonup \prod_{t \in T} A_t.
\]

Composition of multi-morphisms is defined as matrix multiplication. Given two multi-morphisms 
\[
f: A \rightharpoonup B \quad \text{and} \quad g: B \rightharpoonup C,
\]
their composition, denoted as
\[ 
f \otimes g: A \rightharpoonup C,
\]
is defined by
\[
(f \otimes g)(a,c) = \bigvee_{b} (f(a,b) \otimes g(b,c)).
\]
It is noteworthy that if \( f \) and \( g \) are both total and faithful, then \( f \otimes g \) is also total and faithful. The identity for an \( \Omega \)-set \( \alpha: A \) in this composition is the multi-morphism
\[ 
1_A = [\cdot = \cdot]_\alpha: A \rightharpoonup A,
\]
defined by the equality in \( A \). This is because for \( f: A \rightharpoonup B \), we have \( 1_A \otimes f = f \otimes 1_B \).

\begin{prop}\label{prop:comprestriction} 
In \( Set(\Omega) \), let \( f: A \rightharpoonup A \) be a multi-morphism such that \( 1_A \leq f \). If in the ML-algebra \( \Omega \), for every truth value \( \alpha \), \( \alpha \otimes \alpha \leq \alpha \), then
\[ 
1_A \leq f \otimes f \leq f.
\] 
Moreover, when the logic has more than two truth values, i.e., if \( |\Omega| > 2 \), we have
\[ 
f \otimes f = 1_A \quad \text{iff} \quad f = 1_A.
\]
\end{prop}

The set of multi-morphisms defined between \( \Omega \)-sets \( \alpha: A \) and \( \beta: B \) is denoted by \( Set(\Omega)|A,B| \). Moreover, every map \( f: A \rightarrow B \) in \( Set \) induces a multi-morphism with source \( A \) and target \( B \), given by
\[
\chi_f: A \times B \rightarrow \Omega, \quad \text{where} \quad \chi_f(a,b) =
\left\{
  \begin{array}{ll}
    \top, & \text{if } f(a) = b, \\
    \bot, & \text{if } f(a) \neq b. \\
  \end{array}
\right.
\]
In this context, the hom-set \( Set[A,B] \), consisting of morphisms between \( A \) and \( B \) in \( Set \), defines a subset of \( Set(\Omega)|A,B| \). For simplicity, in the subsequent discussions, we will use the notation \( f: A \rightarrow B \) instead of \( f: A \rightharpoonup B \) to denote the multi-morphism induced by a map. 

Thus, \( f: A \rightarrow B \) defines a total multi-morphism from \( \alpha: A \), characterized by the equality similarity, and is a faithful multi-morphism to \( \beta: B \), characterized by the similarity
\[
[a=b] =
\left\{
  \begin{array}{ll}
    \top, & \text{if } a = b \text{ and } a \in f(A), \\
    \bot, & \text{otherwise.} \\
  \end{array}
\right.
\]

The formula for multi-morphism composition simplifies considerably when one of the multi-morphisms is a set-map. For maps \( f: A \rightarrow B \) and \( g: B \rightarrow C \), and multi-morphisms \( r: A \rightharpoonup B \) and \( s: B \rightharpoonup C \), we have
\[
(f \otimes s)(a,c) = s(f(a), c), \quad \text{and} \quad (r \otimes g)(a,c) = \bigvee_{b \in g^{-1}(c)} r(a,b).
\]
The composition operator for multi-morphisms can be extended to cases where the multi-morphisms are not composable in the traditional sense. Let
\[
f: A \rightharpoonup X \times W \quad \text{and} \quad g: B \times X \rightharpoonup C,
\]
then we define
\[ 
f \otimes g: A \times B \rightharpoonup W \times C,
\]
as
\[
(f \otimes g)(a,b,w,c) = \bigvee_{x} (f(a,x,w) \otimes g(b,x,c)).
\]
In particular, if \( f: A \rightharpoonup B \) and \( g: C \rightharpoonup D \) with \( B \neq C \), then 
\[ 
f \otimes g: A \times C \rightharpoonup B \times D,
\]
is given by 
\[
(f \otimes g)(a,c,b,d) = f(a,b) \otimes g(c,d).
\]
This formulation captures the independence between entities in \( B \) and \( C \), leading us to the following definition:

\begin{defn}[Independence]
Two multi-morphisms \( f \) and \( g \) are termed independent if 
\[ 
f \otimes g = g \otimes f.
\]
\end{defn}

\begin{exam}[Keys in a Relational Database]\label{def:indexproduct}
The relational model for database management is grounded in predicate logic and set theory. It operates on the foundational assumption that data is represented as mathematical \(n\)-ary relations, where an \(n\)-ary relation is a subset of the Cartesian product of \(n\) domains. Traditionally, reasoning about such data is conducted using either two-valued or three-valued logic, with operations being performed via relational calculus or relational algebra.

The relational model offers database designers the ability to craft a coherent and logical representation of information. This consistency is achieved by incorporating declared constraints into the database design, commonly referred to as the logical schema.

A weight table \( R \) in a database, defined using attributes \( (A_i)_I \), can be represented as a map in an ML-algebra:
\[
R: \prod_{i \in I} A_i \rightarrow \Omega.
\]
In this context, a weight table is a \(\Omega\)-set \( \alpha: \prod_{i \in I} A_i \).

Each weight table \( R: A \times B \rightarrow \Omega \) can be described as the multi-morphism \( R: A \rightharpoonup B \) and can be decomposed into two weight tables:
\[
D_0: A \times K \rightarrow \Omega
\]
\[
D_1: K \times B \rightarrow \Omega
\]
such that
\[
R = D_0 \otimes D_1.
\]
Here, we refer to \( K \) as a set of \emph{keys} between \( D_0 \) and \( D_1 \), denoting their joint as
\[
D_0 \otimes_K D_1.
\]

In a more general setting, if \( K_1, K_2, \ldots, K_n \) are sets of keys between \( D_0 \) and \( D_1, D_2, \ldots, D_n \) respectively, we denote their joint product as
\[
D = D_0 \otimes_{K_1, K_2, \ldots, K_n} (D_1 \otimes \ldots \otimes D_n),
\]
which can also be represented as
\[
D = (((D_0 \otimes_{K_1} D_1) \otimes_{K_2} D_2) \otimes_{K_3} \ldots) \otimes_{K_n} D_n,
\]
or simply as
\[
D = D_0 \otimes_{K_1} D_1 \otimes_{K_2} D_2 \otimes_{K_3} \ldots \otimes_{K_n} D_n.
\]
When the family \( (K_i) \) of keys is defined by the same set \( K \), the joint product is termed the \( K \)-indexed product of \( D_0, D_1, \ldots, D_n \), denoted as
\[
D = D_0 \otimes_{K} D_1 \otimes_{K} D_2 \otimes_{K} \ldots \otimes_{K} D_n.
\]
In this scenario, \( D \) is referred to as the \( K \)-\emph{indexed product} of \( D_0, D_1, D_2, \ldots, D_n \).
\end{exam}

Given the significance of multi-morphism composition in this work, it is essential to formally define what is meant by multi-morphism composition:
\begin{defn}[Multi-morphism Composition]\label{def:composition}
Let \( f \) and \( g \) be multi-morphisms defined by \(\Omega\)-maps
\[
f: \prod_{i \in I(f)} A_i \rightarrow \Omega \quad \text{and} \quad g: \prod_{j \in I(g)} B_j \rightarrow \Omega
\]
where for every \( i \in I(f) \) and \( j \in I(g) \), \( i = j \) if and only if \( A_i = B_j \). Without specifying source and target sets for \( f \) and \( g \), we define
\[
(f \otimes g)(\bar{x}, \bar{y}) = f(\bar{x}) \otimes g(\bar{y}),
\]
for every \(\bar{x} \in \prod_{i \in I(f)} A_i\) and \(\bar{y} \in \prod_{j \in I(g)} B_j\). However, when we select source sets \( S(f) \subset I(f) \) and \( S(g) \subset I(g) \), as well as target sets \( T(f) \subset I(f) \) and \( T(g) \subset I(g) \) such that \( S(f) \cap T(f) = \emptyset \) and \( S(g) \cap T(g) = \emptyset \), we define
\[
(f \otimes g)(\bar{x}, \bar{y}) = \bigvee_{\bar{z} \in \prod_{i \in T(f) \cap S(g)} A_i} f(\bar{x}, \bar{z}) \otimes g(\bar{z}, \bar{y}),
\]
for every \(\bar{x} \in \prod_{i \in S(f)} A_i \times \prod_{j \in S(g) \setminus T(f)} B_j\) and \(\bar{y} \in \prod_{i \in T(f) \setminus S(g)} A_i \times \prod_{j \in T(g)} B_j\).
\end{defn}

\begin{figure}[h]
\[
\small
\xymatrix @=7pt {
&&&&&*+[o][F-]{g}\ar `r[rdd][rdd] &\\
&&*+[o][F-]{f}\ar `r[rrd][rrd]\ar `r[rrrd][rrrd] &&&&\\
 A_0\ar `u[urr][urr]\ar `d[drrrr][drrrr]&A_1\ar `u[ur][ur]\ar `d[drrr][drrr]& &A_2\ar `u[uurr][uurr]\ar `d[dr][dr]&A_3\ar `u[uur][uur]&A_4&A_5\\
 &&&&*++[o][F-]{f\otimes g}\ar `r[ru][ru]\ar `r[rru][rru]&&\\
 }
\]
\caption{Multi-morphism composition.}\label{multimorphism composition}
\end{figure}
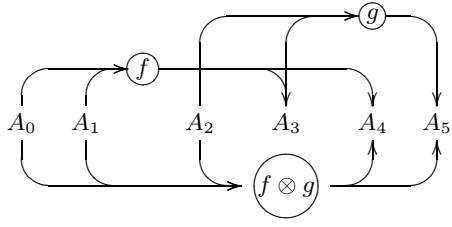

The transpose \( f^\circ: B \rightharpoonup A \) of a multi-morphism \( f: A \rightharpoonup B \) is defined as \( f^\circ(b,a) = f(a,b) \). It is evident that
\[
(\cdot)^\circ: \text{Set}(\Omega)|A,B| \rightarrow \text{Set}(\Omega)|B,A|
\]
is order-preserving, and
\[
[\cdot = \cdot]^\circ = [\cdot = \cdot], \quad (f \otimes g)^\circ = g^\circ \otimes f^\circ, \quad \text{and} \quad {(f^\circ)}^\circ = f.
\]
Thus, if \( f \) is a multi-morphism with a set of sources \( \mathcal{A} \) and a set of targets \( \mathcal{B} \), then \( \mathcal{B} \) serves as the set of sources for \( f^\circ \) and \( \mathcal{A} \) as its set of targets.
\begin{figure}[h]
\[
\small
\xymatrix @=7pt {
&&&*+[o][F-]{f}\ar `r[rd][rd]\ar `r[rrd][rrd]\ar `r[rrrd][rrrd]&&&\\
 A_0\ar `u[urrr][urrr]&A_1\ar `u[urr][urr]& A_2\ar `u[ur][ur]&&A_3\ar `d[dl][dl]&A_4\ar `d[dll][dll]&A_5\ar `d[dlll][dlll]\\
 &&&*+[o][F-]{f^\circ}\ar `l[lu][lu]\ar `l[llu][llu]\ar `l[lllu][lllu]&&&
 }
\]
\caption{Transpose.}\label{multimorphism:transpose}
\end{figure}
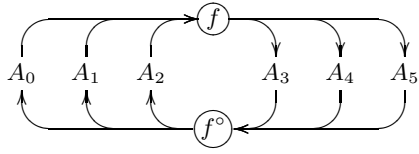

We categorize multi-morphisms based on their ability to preserve the distribution of truth values in their domain or codomain.

\begin{defn}\label{def:isomorphism}
A multi-morphism \( f: A \rightharpoonup B \) is termed an \emph{epimorphism} between \( \alpha: A \) and \( \beta: B \) if
\[ f^\circ \otimes \alpha \otimes f = \beta. \]
Similarly, it is designated a \emph{monomorphism} between \( \alpha: A \) and \( \beta: B \) when
\[ \alpha = f \otimes \beta \otimes f^\circ. \]
When both conditions hold, \( f \) is referred to as an \emph{isomorphism} between \( \Omega \)-objects \( \alpha: A \) and \( \beta: B \).
\end{defn}

For every set-map \( f: A \rightarrow B \), the following relationships hold:
\[
[\cdot = \cdot]_A = 1_A \leq f \otimes f^\circ \quad \text{and} \quad f^\circ \otimes f \leq 1_B = [\cdot = \cdot]_B,
\]
indicating that \( f \) is left adjoint to \( f^\circ \), denoted by \( f \dashv f^\circ \). If \( f: A \rightharpoonup B \) is a multi-morphism satisfying \( 1_A = f \otimes f^\circ \) and \( f^\circ \otimes f = 1_B \), then the multi-morphism \( f \) is termed \emph{orthogonal}.

In a general context, consider multi-morphisms \( f: A \rightharpoonup B \) and \( g: B \rightharpoonup A \). We define \( f \) as the \emph{left adjoint} to \( g \) for \( \alpha: A \) and \( \beta: B \) if the following conditions hold:
\[
[\cdot = \cdot]_\alpha \leq f \otimes g \quad \text{and} \quad g \otimes f \leq [\cdot = \cdot]_\beta.
\]

The tensor product in \( \Omega \) can be naturally extended to \( \Omega \)-sets. For \( \Omega \)-sets \( \alpha: A \) and \( \beta: B \), we define \( \alpha \otimes \beta \) as the \( \Omega \)-set associated with the Cartesian product \( A \times B \) in \( \text{Set}(\Omega) \). This is equipped with the following similarity relation:
\[
[(a_1, b_1) = (a_2, b_2)]_{\alpha \otimes \beta} = [a_1 = a_2]_\alpha \otimes [b_1 = b_2]_\beta.
\]

For each \( \Omega \)-set \( \alpha: A \), the functor
\[
\alpha \otimes \cdot: \text{Set}(\Omega) \rightarrow \text{Set}(\Omega),
\]
possesses an adjoint known as the hom functor \( (\cdot)^\alpha: \text{Set}(\Omega) \rightarrow \text{Set}(\Omega) \). This is defined by \( \beta^\alpha = \text{Set}(\Omega)[\alpha, \beta] \), with the similarity relation given by
\[
[f = g]_{\beta^\alpha} = \bigwedge_{a \in A} \bigwedge_{b \in B} (f(a, b) \Leftrightarrow g(a, b)).
\]

For any \( f \in \text{Set}(\Omega)[\beta, \gamma] \), we have
\[
f^\alpha = \text{Set}(\Omega)[\alpha, f]: \text{Set}(\Omega)[\alpha, \beta] \rightarrow \text{Set}(\Omega)[\alpha, \gamma],
\]
such that \( (f^\alpha)(g) = g \otimes f \).

Being monoidal-closed, \( \Omega \) possesses a natural structure as an \( \Omega \)-set, given by
\[
[x=y]_\Omega = (x \Leftrightarrow y) = (x \Rightarrow y) \otimes (y \Rightarrow x).
\]

Given the similarity relations \( [\cdot=\cdot]_\alpha:A\rightharpoonup A \) and \( [\cdot=\cdot]_\beta:B\rightharpoonup B \), if the sets \( A \) and \( B \) are distinct, we can apply the composition definition to obtain the similarity relation
\[ [\cdot=\cdot]_\alpha\otimes[\cdot=\cdot]_\beta:A\times B\rightharpoonup A\times B, \]
defined by
\[ ([\cdot=\cdot]_\alpha\otimes[\cdot=\cdot]_\beta)(a_1,b_1,a_2,b_2) = [a_1=a_2]_\alpha\otimes[b_1=b_2]_\beta. \]
This relation defines the \( \Omega \)-object \( \alpha\otimes\beta:A\times B \). More generally, we define:

\begin{defn}[Product of \( \Omega \)-sets]\label{ProdSimil}
Given \( \Omega \)-sets \( \alpha:A \) and \( \beta:B \), we define the product \( \alpha\otimes\beta \) as the \( \Omega \)-set \( \alpha\otimes\beta:A\times B \) equipped with the similarity relation
\[ [\cdot=\cdot]_{\alpha\otimes\beta}:A\times B\times A\times B\rightharpoonup \Omega, \]
such that
\[ [(a_1,b_1)=(a_2,b_2)]_{\alpha\otimes\beta} = [a_1=a_2]_{\alpha}\otimes[b_1=b_2]_{\beta}. \]
\end{defn}

Given the transitivity imposed on the definition of similarity, we have
\[ [\cdot=\cdot]_{\alpha\otimes\beta} = [\cdot=\cdot]_\alpha\otimes[\cdot=\cdot]_\beta \leq [\cdot=\cdot]_\alpha. \]

\chapter{Bayesian Inference in Basic Logic}\label{bayesian inference}

The definition presented for multi-morphism composition \( \otimes \) is compatible with Bayes' theorem when the \( Set(\Omega) \) logic is considered a basic logic.

\begin{prop}[Bayes' Rule]
Let \( \Omega \) be a divisible ML-algebra. Given a faithful and total multi-morphism \( f:A\rightharpoonup B \), and observable descriptions \( a \) and \( b \) of entities in \( \alpha:A \) and \( \beta:B \), respectively, the equations
\begin{enumerate}
  \item \( [a]_\alpha\otimes f(\beta|a) = f(a,\_) \), and
  \item \( [b]_\beta\otimes f(\alpha|b) = f(\_,b) \),
\end{enumerate}
have solutions, and they define \( \Omega \)-maps \( f(\alpha|b):A\rightarrow \Omega \) and \( f(\beta|a):B\rightarrow \Omega \). These are given by
\[ f(\beta|a) = [a]_\alpha\Rightarrow f(a,\_) \]
and
\[ f(\alpha|b) = [b]_\beta\Rightarrow f(\_,b) \].
\end{prop}

\begin{proof}
In a divisible ML-algebra \( \Omega \), we have \( x\otimes(x\Rightarrow y) = x\wedge y \). Since \( f \) is faithful, \( [a]_\alpha = \bigvee_c f(a,c) \geq f(a,c) \), then \( [a]_\alpha \geq f(a,\_) \). Because \( [a]_\alpha \wedge f(a,\_) = f(a,\_) \), we have
\[
[a]_\alpha\otimes([a]_\alpha\Rightarrow f(a,\_)) = [a]_\alpha\wedge f(a,\_) = f(a,\_).
\]
Similarly, we can use the same strategy to prove \( f(\alpha|b) = [b]_\beta\Rightarrow f(\_,b) \).
\end{proof}

We interpret the \( \Omega \)-map \( f(\beta|a) \) as a classifier in \( B \), defined by relation \( f \), for an entity described by \( a \) using the basic monoidal logic \( \Omega \).

Applying the principles of Bayesian inference in the multi-morphism context: For faithful and total multi-morphisms \( f:A\rightharpoonup B \) and \( g:B\rightharpoonup C \), and \( \Omega \)-sets \( \alpha:A \), \( \beta:B \), \( \gamma:D \), we have
\begin{center}
\begin{tabular}{rcl}
  \( [a]_\alpha\otimes (f\otimes g)(\gamma|a)(c) \) & = & \( (f\otimes g)(a,c) \) \\
   & = & \( \bigvee_b f(a,b)\otimes g(b,c) \)\\
   & = & \( \bigvee_b [a]_\alpha\otimes f(\beta|a)(b)\otimes g(b,c) \), \\
\end{tabular}
then
\[
(f\otimes g)(\gamma|a)(c) = [a]_\alpha\Rightarrow \left([a]_\alpha\otimes\bigvee_b  f(\beta|a)(b)\otimes g(b,c)\right),
\]
i.e.,
\[
(f\otimes g)(\gamma|a) = \bigvee_b f(\beta|a)(b)\otimes g(b,\_),
\]
since in a divisible ML-algebra \( \Omega \), we have \( x\Rightarrow(x\otimes y) = x\wedge y \).
\end{center}

When \( f:A\rightharpoonup C \) and \( g:B\rightharpoonup D \) are independent, we have
\[
(f\otimes g)(\gamma\otimes \delta|a,b)(c,d) = f(\gamma|a)(c) \otimes g(\delta|b)(d).
\]
Naturally, if \( C=D \), we write \( (f\otimes g)(\delta|a,b)(d) \) for \( (f\otimes g)(\delta\otimes \delta|a,b)(d,d) \). In this scenario, we interpret the classifier \( (f\otimes g)(\delta|a,b) \) as the combination of two classifiers: \( f(\delta|a) \) and \( g(\delta|b) \), which define entities of \( \delta:D \) described by \( a \) and \( b \).

\begin{exam}[Binding]
Drugs are typically small organic molecules that achieve their desired activity by binding to a specific target site on a receptor. The initial step in the discovery of a new drug often involves identifying and isolating the receptor to which it should bind. Subsequently, numerous small molecules are tested for their ability to bind to this target site. This process leaves researchers facing the challenge of discerning the distinguishing factors between active (binding) compounds and inactive (non-binding) ones. Such determinations are crucial as they inform the design of new compounds that not only bind effectively but also possess other essential properties required for a drug, including solubility, oral absorption, absence of side effects, appropriate duration of action, and toxicity, among others.

DuPont Pharmaceuticals contributed a dataset to the KDD Cup 2001, comprising 1,909 compounds evaluated for their ability to bind to a specific site on thrombin, a pivotal receptor involved in blood clotting. Out of these compounds, 42 were found to be active (i.e., they bind effectively), while the remaining compounds were deemed inactive. Each compound is represented by a feature vector, which includes an observable class value (A for active, I for inactive) and 139,351 binary features detailing the molecule's three-dimensional properties. Notably, the specific definitions of these individual binary features were not provided. They are essentially non-observable descriptions of each compound; researchers were unaware of the exact meaning of each bit, but it was understood that these bits were generated consistently across all 1,909 compounds. 

It is widely acknowledged that biological activity, particularly receptor binding affinity, is closely associated with various structural and physical properties of small organic molecules. The challenge set forth by DuPont Pharmaceuticals for the KDD Cup 2001 was to ascertain which of these three-dimensional properties are critical in determining the binding affinity in this context. The ultimate goal was to develop an accurate predictive model capable of determining the class value (active or inactive) of new compounds based on their structural features.

Let \( S \) denote the set of available compounds. Suppose that the process of classifying a compound in the laboratory results in the proposition "Compound \( a \) is active," modeled within the framework of fuzzy logic, denoted by \( \Omega = [0,1] \). A classification of each compound as either active or inactive can be represented as a multi-morphism in \( Set([0,1]) \):
\[
c: S \rightharpoonup \{A, I\}
\]
Here, \( c(a, I) \) and \( c(a, A) \) define the truth values corresponding to the propositions "Compound \( a \) is inactive" and "Compound \( a \) is active," respectively, within the interval \( \Omega = [0,1] \).

Each compound in \( S \) is characterized by a set of observable three-dimensional properties, which are determined through laboratory processes and subsequently encoded in the dataset. The similarity between compounds must be represented by an \( \Omega \)-set, denoted as \( \alpha: S \). 

In the context of \( \alpha: S \), if \( \bar{x} \) represents an observable three-dimensional structure of a compound, it can be regarded as a generalization of the compound's structure. A description \( x \) defines a class of compounds within \( \alpha: S \). We can then define the truth value of the proposition "Compounds that satisfy the description \( \overline{x} \) are active" as:
\[
c(\beta | \overline{x})(A)
\]
Here, the \( \Omega \)-set \( \beta: \{A, I\} \) encodes the similarity between the states of a compound being "active" or "inactive".

\end{exam}

In this example, the optimal description \( \bar{x} \) for an active compound in \( S \) can be viewed as the description that maximizes \( c(\beta | \bar{x})(A) \). However, this notion highlights the necessity of having a systematic approach to encode observable descriptions and the establishment of a framework for selecting the best observable description. In the following chapters, we introduce a methodology to describe entities using a graphical language. The vocabulary of this language consists of terms designed to encode relationships between the observable characteristics of entities within a multi-valued logic fr


\chapter{Multi-diagrams}\label{multidiagrams}

A \emph{multi-diagram} in \( Set(\Omega) \) is a multi-graph homomorphism \( D: \mathcal{G} \rightarrow Set(\Omega) \) defined by mapping the multi-graph vertices to \( \Omega \)-sets and multi-arrows to multi-morphisms in \( Set(\Omega) \).

Formally, if the multi-graph \( \mathcal{G} \) is defined using nodes \( (v_i)_{i \in L} \) and by a family of multi-arrows \( (a_{IJ}) \), where the multi-arrow \( a_{IJ} \) has a source
\[
\{ v_i : i \in I \subset L \},
\]
and a target
\[
\{ v_j : j \in J \subset L \},
\]
a multi-graph homomorphism \( D: \mathcal{G} \rightarrow Set(\Omega) \) transforms every node \( v_i \) into an \( \Omega \)-set \( D(v_i) \) and each multi-arrow 
\[
a_{IJ}: \{ v_i : i \in I \subset L \} \rightharpoonup \{ v_j : j \in J \subset L \},
\]
into a multi-morphism 
\[
D(a_{IJ}): \prod_{i \in I} D(v_i) \rightharpoonup \prod_{j \in J} D(v_j).
\]

The standard definition of limit in \( Set \) for a diagram can be extended to multi-diagrams in \( Set(\Omega) \). For this purpose, we consider the category \( Set \) as the topos \( Set(\{false, true\}) \), where \( \{false, true\} \) defines a two-element chain with the monoidal structure given by the logical operators "and" and "true". Recall that the limit for a diagram or multi-diagram \( D \) is defined as a \( \{false, true\} \)-set, denoted by \( Lim \; D \), which is a subobject of a Cartesian product defined by the diagram vertices (see \cite{maclane71} or \cite{Borceux94}). We utilize this definition of limit extension to multi-diagrams in \( Set(\Omega) \). Since the Cartesian product of \( \Omega \)-sets \( (\alpha_i: A_i) \) was defined in \ref{ProdSimil} as the \( \Omega \)-set \( \otimes_i \alpha_i: \prod_i A_i \) given by
\[
[\cdot = \cdot]_{\otimes_i \alpha_i} = \bigotimes_i [\cdot = \cdot]_{\alpha_i},
\]
we define
\begin{defn}[Limit of a multi-diagram]\label{lim}
Let \( D: \mathcal{G} \rightarrow Set(\Omega) \) be a multi-diagram where \( \mathcal{G} \) has vertices \( (v_i)_{i \in L} \). Its limit \( Lim \; D \) is a subobject of the multi-diagram vertices' Cartesian product:
\[
Lim \; D \leq \prod_{i \in L} M(v_i)
\]
given by
\[
Lim \; D: \prod_{i \in L} M(v_i) \rightarrow \Omega
\]
such that
\[
(Lim \; D)(\bar{x}_1, a_{i}, \bar{x}_2, a_{j}, \bar{x}_3) = [(\bar{x}_1, a_{i}, \bar{x}_2, a_{j}, \bar{x}_3)]_{\prod_{i \in L} M(v_i)} \otimes \bigotimes_{f: v_i \rightharpoonup v_j \in \mathcal{G}} D(f)(a_{i}, a_{j}).
\]
\end{defn}

We view a limit as the result of applying the pattern used in the definition of each multi-morphism to the Cartesian product of its vertices. This definition satisfies the usual universal property when the object classifier used in \( Set(\Omega) \) is the two-element chain, \( 2 = \{false, true\} \), with the monoidal structure given by "and" and "true". In other words, this definition coincides with the classical one in the context of classical logic.

It is noteworthy that we can interpret the limit for a multi-diagram \( D \) as a multi-morphism by selecting a set of source \( \Omega \)-sets and a set of targets. The canonical multi-morphism associated with a multi-diagram \( D \) has a source \( s(D) \) that is the union of the sources used to define the diagram multi-morphisms and a target \( t(D) \) that is the union of the targets of \( D \)'s multi-arrows.

\begin{exam}
By definition, the multi-diagram \( D \) presented below is represented as:
\begin{figure}[h]
\[
\small
\xymatrix @=7pt {
&&&*+[o][F-]{f}\ar `r[rd][rd]\ar `r[rrd][rrd]\ar `r[rrrd][rrrd]&&&\\
 A_0\ar `u[urrr][urrr]&A_1\ar `u[urr][urr]\ar `d[drr][drr]& A_2\ar `d[dr][dr]\ar[r]&*+[o][F-]{h}\ar[r] &A_3&A_4&A_5\\
 &&&*+[o][F-]{g}\ar `r[rru][rru]\ar `r[rrru][rrru]&&&
 }
\]
\caption{Multi-diagram.}\label{multidiagram1}
\end{figure}
The limit of this multi-diagram \( D \) is the \( \Omega \)-map
\[
Lim\;D:A_0\times A_1\times A_2\times A_3\times A_4\times A_5\rightarrow \Omega
\]
defined as:
\[
(Lim\;D)(a_0,a_1,a_2,a_3,a_4,a_5)=[a_0,a_1,a_2,a_3,a_4,a_5]\otimes f(a_0,a_1,a_3,a_4,a_5)\otimes g(a_1,a_2,a_4,a_5)\otimes h(a_2,a_3).
\]
The functional representation of the multi-diagram limit is depicted as:
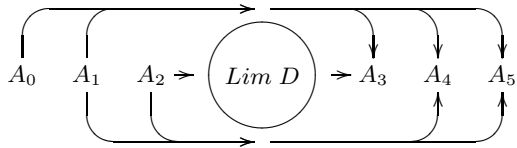
\begin{figure}[h]
\[
\small
\xymatrix @=7pt {
&&&\ar `r[rd][rd]\ar `r[rrd][rrd]\ar `r[rrrd][rrrd]&&&\\
 A_0\ar `u[urrr][urrr]&A_1\ar `u[urr][urr]\ar `d[drr][drr]& A_2\ar `d[dr][dr]\ar[r]&*++++[o][F-]{Lim\;D}\ar[r]&A_3&A_4&A_5\\
 &&&\ar `r[rru][rru]\ar `r[rrru][rrru]&&&
 }
\]
\caption{Multi-diagram limit functionality.}\label{multidiagram2}
\end{figure}
\end{exam}

The limit of a multi-diagram condenses the entire diagram into a single multi-morphism by internalizing all interconnections, thereby representing the multi-diagram as a unified entity.

In this context, the equalizer of a parallel pair of multi-morphisms \( R, S: X \rightharpoonup Y \) is defined as:
\[
Lim(R=S): X \times Y \rightarrow \Omega
\]
where
\[
Lim(R=S)(x,y) = [x,y] \otimes R(x,y) \otimes S(x,y).
\]

Similarly, the pullback of multi-morphisms \( R: X \rightharpoonup U \) and \( S: Y \rightharpoonup U \) is represented by the multi-morphism:
\[
Lim(R\otimes_U S): X \times U \times Y \rightarrow \Omega
\]
where
\[
Lim(R\otimes_U S)(x,u,y) = [x,u,y] \otimes R(x,u) \otimes S(y,u).
\]

For a discrete multi-diagram \( D \), its limit is denoted by \( \Pi_v D(v) \), given by:
\[
\Pi_v D(v)(\bar{x}) = [\bar{x}]_{\bigotimes_i D(i)},
\]
where, for \( \bar{x} = (x_1, x_2, \ldots, x_n) \), \( \Pi_v D(v)(\bar{x}) = [x_1, x_2, \ldots, x_n] \).

The provided definition for the limit streamlines the proof of the following proposition:

\begin{prop}[Existence of Limit in \( Set(\Omega) \)]
Every multi-diagram \( D: \G \rightarrow Set(\Omega) \) has a limit. In other words, there exists a multi-morphism \( f \leq \prod_{v \in \G} D(v) \) such that \( Lim\;D = f \).
\end{prop}

Even more intriguing is the reverse implication that can be demonstrated for fundamental logics:

\begin{prop}
If \( \Omega \) is a divisible ML-algebra, then for every \( \Omega \)-map
\[
g: A_0 \times A_1 \times \ldots \times A_n \rightarrow \Omega
\]
and \( \Omega \)-map \( (\alpha_i: A_i) \) such that
\[
g(x_1, x_2, \ldots, x_n) \leq [x_1, x_2, \ldots, x_n]
\]
there exists a multi-diagram \( D: \G \rightarrow Set(\Omega) \) such that
\[
Lim\;D = g.
\]
\end{prop}

This proposition can be proven by demonstrating that the multi-diagram
\[
f: A_0 \rightharpoonup A_1 \times \ldots \times A_n
\]
where
\[
f(x_1, x_2, \ldots, x_n) = [x_1, x_2, \ldots, x_n] \Rightarrow g(x_1, x_2, \ldots, x_n)
\]
has \( g \) as its limit.


A multi-diagram can be perceived as a mechanism to depict dependencies among classes of entities. When the limit of a diagram \( D \) results in an \( \Omega \)-object \( \alpha \), we aim to interpret the diagram as a representation or model for \( \alpha \). However, for such a relationship to be meaningful, we require a language to formalize the structure of the diagram. Developing this language is one of the objectives of this research.

Let \( D \) be a multi-diagram with vertices \( (v_i) \), where the arrows are represented by faithful and total multi-morphisms, and let \( a_i \) be an observable description of an entity in \( D(v_i) \). The limit in \( Set(\Omega) \) of \( D \), where \( \Omega \) is a divisible ML-algebra, defines the classifier
\[
(Lim\;D)(D(v_1)|a_2,\ldots,a_n)
\]
such that
\[
[a_2,\ldots,a_n] \otimes (Lim\;D)(D(v_1)|a_2,\ldots,a_n) = \bigotimes_{f:v_i\rightharpoonup v_j\in \G} [a_i] \otimes D(f)(D(v_{j})|a_{i})(a_j),
\]
or equivalently,
\[
(Lim\;D)(D(v_1)|a_2,\ldots,a_n) = [a_2,\ldots,a_n] \Rightarrow \bigotimes_{f:v_i\rightharpoonup v_j\in \G} [a_i] \otimes D(f)(D(v_{j})|a_{i})(a_j),
\]
which can be interpreted as the amalgamation of classifiers associated through diagram \( D \) to predict \( D(v_1) \). This expression simplifies when \( D \) does not contain multi-arrows with source \( v_1 \), leading to
\[
(Lim\;D)(D(v_1)|a_2,\ldots,a_n) = \bigotimes_{f:v_i\rightharpoonup v_j\in \G} D(f)(D(v_{j})|a_{i})(a_j).
\]

The limit of a diagram can be perceived as a generalization of multi-morphism composition, encompassing a chain of composable multi-morphisms. This interpretation paves the way for defining a semantics for circuits, particularly when we postulate a dependency between the execution of circuit components. This perspective is instrumental in extending the conventional notion of a commutative diagram to fuzzy structures. To do so, we designate a subset of vertices, denoted as \( s(D) \), within a given diagram \( D \) as the set of \emph{sources} for \( D \).

\begin{defn}[Commutativity of Multi-diagrams]
Let \( D \) be a multi-diagram with \( s(D) \) as its set of sources. If \( V \) represents the cartesian product defined by all the vertices of \( D \) not in \( s(D) \), then the multi-diagram \( D \) is said to be commutative for \( s(D) \) if
\[
\bigvee_{\bar{n} \in V} (Lim\;D)(\bar{s},\bar{n}) = \bigvee_{\bar{n} \in V} (\prod_i\;D(i))(\bar{s},\bar{n}),
\]
for every \( \bar{s} \in \prod_{i \in s(D)} D(i) \). It is \(\lambda\)-commutative if 
\[
\left( \bigvee_{\bar{n} \in V} (Lim\;D)(\bar{s},\bar{n}) \Leftrightarrow \bigvee_{\bar{n} \in V} (\prod_i\;D(i))(\bar{s},\bar{n}) \right) \geq \lambda,
\]
for every \( \bar{s} \in \prod_{i \in s(D)} D(i) \).
\end{defn}

In simpler terms, a multi-diagram is deemed commutative if the multi-morphism defined by its limit, with the selected sources, is total.

\begin{exam} Consider \( Set([0,1]) \) defined by the product logic, \( \mathds{R} \) as the set of real numbers, and \( \oplus \) as a relation given by the multi-morphism \( \oplus: \mathds{R} \times \mathds{R} \rightharpoonup \mathds{R} \), represented by the Gaussian function:
\[
\oplus(x,y,z) = e^{-\frac{(z-x-y)^2}{2}}.
\]
The diagram \( D \), shown in Figure \ref{equation1}, with sources \( \alpha_0: \mathds{R} \) and \( \alpha_1: \mathds{R} \),
\begin{figure}[h]
\[
\small
\xymatrix @=7pt {
&&&&\\
_{\alpha_0:\mathds{R}}\ar[r]\ar `d[ddrr]`r[rru][drr] &*+[o][F-]{+}\ar `r[rrd][rrd] &&&\\
&_{\alpha_1:\mathds{R}}\ar[u]\ar[r] &*+[o][F-]{+} \ar [r]&*+[o][F-]{_{=}}\\
&&&&\\
}
\]
\caption{Multi-diagram \( D \) representing \( \alpha_0 + \alpha_1 = \alpha_1 + \alpha_0 \).}\label{equation1}
\end{figure}
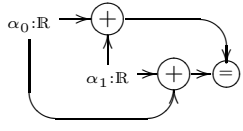
where \( = \) signifies equality in \( \mathds{R} \), is commutative for every \( x_0, x_1 \in \mathds{R} \), provided we have the densities in \( \alpha_0 \) and \( \alpha_1 \) given by
\[
\alpha_0(x,y) = e^{-\frac{(x-x_0)^2}{2} - \frac{(y-x_0)^2}{2}} \quad \text{and} \quad \alpha_1(x,y) = e^{-\frac{(x-x_1)^2}{2} - \frac{(y-x_1)^2}{2}}.
\]
Utilizing the definition of the multi-diagram limit, we find
\[
\begin{array}{rcl}
  (Lim\;D)(x,y,w) & = & \oplus(x,y,w) \otimes \oplus(y,x,w) \otimes [x,y,w] \\
                  & = & \oplus(x,y,w) \otimes \oplus(y,x,w) \otimes [x] \otimes [y] \otimes [w] \\
                  & = & e^{-\frac{(w-x-y)^2}{2}} \cdot e^{-\frac{(w-y-x)^2}{2}} \cdot e^{-\frac{(x-x_0)^2}{2}-\frac{(x-x_0)^2}{2}} \cdot e^{-\frac{(y-x_1)^2}{2}-\frac{(y-x_1)^2}{2}} \cdot 1 \\
                  & = & e^{-(w-x-y)^2-(x-x_0)^2-(y-x_1)^2}.
\end{array}
\]

Then, since \( e^{-(w-x-y)^2} \leq 1 \), we can infer that \( e^{-(w-x-y)^2-(x-x_0)^2-(y-x_1)^2} \leq e^{-(x-x_0)^2-(y-x_1)^2} \). Thus,
\[
\begin{array}{rcl}
  \bigvee_w (Lim\;D)(x,y,w) & = & e^{-(x-x_0)^2-(y-x_1)^2} \\
                             & = & [x]_{\alpha_0} \otimes [y]_{\alpha_1} \\
                             & = & \bigvee_w \alpha_0(x,x) \otimes \alpha_1(y,y) \otimes =(w,w).
\end{array}
\]
This proves the commutativity for diagram \( D \) when its sources have fixed distributions. The diagram \( D' \), shown in Figure \ref{equation2},
\begin{figure}[h]
\[
\small
\xymatrix @=7pt {
 _{0_{1_{\mathds{R}}}:\mathds{R}}\ar[r]&*+[o][F-]{+}\ar `r[rd][rd]&\\
  &&_{\alpha_0:\mathds{R}}\ar `l[lu][lu]\\
}
\]
\caption{Multi-diagram \( D' \) representing \( 0_{1_{\mathds{R}}} + \alpha_0 = \alpha_0 \).}\label{equation2}
\end{figure}
with its source as the \( [0,1] \)-set \( \alpha_0:\mathds{R} \), defined by the distribution
\[
\alpha_0(x,y) = e^{-\frac{(x-x_0)^2}{2}-\frac{(y-x_0)^2}{2}},
\]
is also commutative when the \( [0,1] \)-set \( 0_{\mathds{R}}:\mathds{R} \) is defined by the distribution
\[
0_\mathds{R}(x,y) = \lambda \cdot e^{-\frac{x^2}{2}-\frac{y^2}{2}},
\]
where the parameter \( \lambda \) is a truth value selected from \( [0,1[ \). We can then calculate
\[
\begin{array}{rcl}
  (Lim\;D')(x,y,x) & = & \oplus(x,y,x) \otimes [x,y,x] \\
                  & = & \oplus(x,y,x) \otimes [x] \otimes [y] \otimes [x] \\
                  & = & e^{-\frac{(x-x-y)^2}{2}} \cdot e^{-\frac{(x-x_0)^2}{2}-\frac{(x-x_0)^2}{2}} \cdot \lambda \cdot e^{-\frac{y^2}{2}-\frac{y^2}{2}} \cdot e^{-\frac{(x-x_0)^2}{2}-\frac{(x-x_0)^2}{2}} \\
                  & = & \lambda \cdot e^{-(x-x-y)^2-2(x-x_0)^2-y^2}.
\end{array}
\]

and
\[
\begin{array}{rcl}
  \bigvee_y (Lim\;D')(x,y,x) & = & \lambda \cdot e^{-(x-x-0)^2-2(x-x_0)^2-0^2} \\
                           & = & \lambda \cdot e^{-2(x-x_0)^2} \\
                           & = & \lambda \otimes [x] \otimes [x] \\
                           & = & \bigvee_y \alpha_0(x,x) \otimes 0_{\mathds{R}}(y,y) \otimes \alpha_0(x,x).
\end{array}
\]

However, if we modify the interpretation of \( \oplus \) in diagram \( D' \) using the new distribution
\[
\oplus(x,y,z) = \lambda' \cdot e^{-\frac{(z-x-y)^2}{2}},
\]
depending on a parameter \( \lambda' \in [0,1[ \), we obtain
\[
(Lim\;D')(x,y,x) = \lambda' \cdot \lambda \cdot e^{-(x-x-y)^2-2(x-x_0)^2-y^2},
\]
thus
\[
\bigvee_v (Lim\;D')(x,y,x) = \lambda' \cdot \bigvee_y \alpha_0(x,x) \otimes 0_{\mathds{R}}(y,y) \otimes \alpha_0(x,x).
\]

Then, since we are working in a multiplicative logic, we have
\[
\left( \bigvee_v (Lim\;D')(x,y,x) \Leftrightarrow \bigvee_v (\alpha_0(x,x) \otimes 0_{\mathds{R}}(y,y) \otimes \alpha_0(x,x) \right) \geq \lambda',
\]
which means that \( D' \) is \( \lambda' \)-commutative.

\end{exam}

Naturally, if a diagram is \( \lambda \)-commutative, it is also \( \lambda' \)-commutative when \( \lambda' < \lambda \). When, for every \( \lambda > \bot \), the diagram isn't \( \lambda \)-commutative, it is called a non-commutative diagram.

However, when we view a multi-diagram as a way to specify architectural connectors, in the sense of \cite{GibbonsBackhouse03}, we may want to interpret diagrams by collapsing the joint execution of its components, generalizing the notion of parallel composition. This can be achieved through the symmetry between the operators \( \otimes \) and \( \vee \) in classical logic. We define:

\begin{defn}[Colimit of multi-diagrams]\label{colim}
Given a multi-diagram \( D \) in \( Set(\Omega) \) with vertices \( (v_i) \), the colimit is defined by the multi-morphism
\[
coLim\; D \leq \prod_i M(v_i)
\]
i.e.,
\[
coLim\; D: \prod_i M(v_i) \rightarrow \Omega
\]
given by
\[
(coLim\;D)(\bar{x}_1, a_{i}, \bar{x}_2, a_{j}, \bar{x}_3) = [\bar{x}_1, a_{i}, \bar{x}_2, a_{j}, \bar{x}_3]_{\prod_i M(v_i)} \otimes \bigvee_{f:v_i \rightharpoonup v_j \in \G} D(f)(a_{i}, a_{j}).
\]
\end{defn}

This definition allows for the formalization of knowledge integration. Colimits capture a generalized notion of parallel composition of components in which the designer makes explicit the interconnections used between components. We can see this operation as a generalization of the notion of superimposition as defined in \cite{Bosch99}.

The colimit for a multi-diagram \( D \) can be used to define a multi-morphism by selecting a set of sources and a set of targets. The canonical multi-morphism associated with a multi-diagram \( D \), using the colimit, has its source \( s(D) \) as the union of the sources used to define the diagram's multi-morphisms and its target \( t(D) \) as the union of the targets of \( D \)'s multi-morphisms.

\begin{exam}
By definition, the multi-diagram \( D \), presented in Fig. \ref{multidiagram1}, has a colimit given by the \(\Omega\)-map
\[
coLim\;D: A_0 \times A_1 \times A_2 \times A_3 \times A_4 \times A_5 \rightarrow \Omega
\]
defined as
\[
(coLim\;D)(a_0, a_1, a_2, a_3, a_4, a_5) = [a_0, a_1, a_2, a_3, a_4, a_5] \otimes (f(a_0, a_1, a_3, a_4, a_5) \vee g(a_1, a_2, a_4, a_5) \vee h(a_2, a_3)).
\]
\end{exam}

In this context, the coequalizer of a parallel pair of multi-morphisms \( R, S: X \rightharpoonup Y \) is defined by the multi-morphism \( coLim(R=S): X \times Y \rightarrow \Omega \) given by
\[
coLim(R=S)(x, y) = [x, y]_{X \times Y} \otimes (R(x, y) \vee S(x, y)).
\]
And the pushout of \( R: X \rightharpoonup U \) and \( S: Y \rightharpoonup U \) is the multi-morphism
\[
coLim(R \oplus_U S): X \times U \times Y \rightarrow \Omega
\]
defined as
\[
coLim(R \oplus_U S)(x, u, y) = [x, y]_{X \times Y} \otimes (R(x, u) \vee S(y, u)).
\]

When \( D \) is a discrete diagram, the colimit coincides with the limit of \( D \), and in this case, we write
\[
\coprod_v D(v) = \prod_v D(v).
\]

Naturally,
\begin{prop}[Existence of Colimit in \( Set(\Omega) \)]
Every multi-diagram \( D: \mathcal{G} \rightarrow Set(\Omega) \) has a colimit.
\end{prop}

Given that \( Set(\Omega) \) has both limits and colimits for multi-diagrams, we utilize it as the "Universe of Discourse" in the subsequent chapters. This allows us to construct models for structures specified by diagrams within the monoidal logic described in \( \Omega \).

\begin{exam}[Genome]
The genomes of numerous organisms, including the human genome, have now been completely sequenced. Consequently, interest within the field of bioinformatics is gradually shifting from sequencing to understanding the genes encoded within the sequence. Genes are responsible for coding proteins, which tend to localize in various cellular compartments and interact with one another to perform essential functions. 

A dataset presented at the KDD Cup 2001 comprises a plethora of details about genes from a specific organism. The Data Analysis Challenge proposed two primary tasks: predicting the functions and localizations of the proteins encoded by these genes. It's worth noting that a single gene or protein may have multiple functions and localizations. Other relevant information for predicting function and localization includes the gene/protein class, the observable characteristics (phenotype) of individuals with a mutation in the gene (and consequently in the protein), and the other proteins known to interact with each protein.

The dependencies associated with this problem can be represented by the multi-diagram \( D \) shown in Figure \ref{genomeattrib}.

\begin{figure}[h]
\[
\small
\xymatrix @=8pt{
\text{Class} & & \text{Phenotype} \\
& \text{Gene} \ar@_{->}[lu]\ar@_{->}[ru]\ar@_{->}[ld]\ar@_{->}[rd]\ar@_{->}[r] & \text{Gene} \times \text{Interaction.type} \\
\text{Function} & & \text{Localization}
}
\]
\caption{Dependencies between attributes.}\label{genomeattrib}
\end{figure}
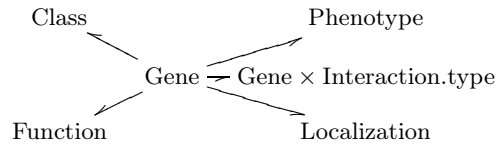

The diagram limit defines a morphism that characterizes the involved entities as:
\[
\text{Gene} \times \text{Class} \times \text{Phenotype} \times \text{Gene} \times \text{Interaction.type} \times \text{Function} \times \text{Localization}
\]
\[
\begin{array}{c}
    \downarrow \\
    \Omega \\
\end{array}
\]
This map can be conceptualized as a dataset where we can compute \(\alpha: \text{Function} \times \text{Localization}\), describing the similarity between pairs in \(\text{Functions} \times \text{Localization}\). Given a description \(\bar{x}\) for a gene, the \(\Omega\)-map
\[
(Lim\;D)(\alpha|\bar{x})(f,l)
\]
reflects the truth value in \(\Omega\) of the proposition "the class of genes characterized by \( \bar{x} \) have function \(f\) and localization \(l\)".

\end{exam}
In this context, a multi-diagram can be aggregated into a relation via its limit. In the following chapters, we will address the inverse problem: defining a multi-diagram whose limit provides an "approximation" to a given multi-morphism. By this, we mean the ability to graphically express fuzzy relations between attributes aggregated in a multi-morphism that encodes a dataset.

\chapter{Specifying libraries of components}\label{specifying libraries}

In Computer Science, a formal grammar is an abstract structure that describes a formal language. Formal grammars are typically classified into two main categories: generative and analytic.

Generative grammars are the most well-known type. They consist of a set of rules by which all possible strings in the language can be generated. This is achieved by successively rewriting strings starting from a designated start symbol. In contrast, an analytic grammar is a set of rules that takes an arbitrary string as input and successively analyzes or reduces the input string to yield a final boolean result, indicating whether or not the input string belongs to the language described by the grammar.

The languages employed in this work are expressed using a type of generative grammar where words are configurations defined using components selected from a library. Each component is associated with a set of requirements, and a configuration is considered valid in the language if all requirements for the used components are satisfied. The adoption of this type of structure and the definition of its semantics were motivated by the domain of Architectural Connectors. This domain has emerged as a powerful tool for describing the overall organization of systems in terms of components and their interactions \cite{LopesFiadeiro97} \cite{KasmanBassClements98} \cite{WolfPerry92}. According to \cite{GarlanAllen97}, an architectural connector can be characterized by a set of roles and a glue specification. The roles of a connector type can be instantiated with specific components of the system under construction, resulting in an overall system structure comprising components and connector instances that establish interactions between the components.

Let $Chains$ be the forgetful functor from the category of totally ordered sets and their homomorphisms to $Set$, the category of all sets. If we interpret a set $\Sigma$ as a set of symbols or signs, objects in the comma category $(Chains\downarrow \Sigma)$ can be viewed as words defined by strings over the \emph{alphabet} $\Sigma$. A set of signs $\Sigma$ equipped with a partial order $\leq$ is referred to as an \emph{ontology}. Given signs $\lambda_0$ and $\lambda_1$ in an ontology $(\Sigma,\leq)$ such that $\lambda_0\leq\lambda_1$, $\lambda_1$ is termed a \emph{generalization} of $\lambda_0$, and $\lambda_0$ is termed a \emph{particularization} of $\lambda_1$.

Given $w\in (Chains\downarrow \Sigma)$, we denote $w:|w|\rightarrow \Sigma$, where $|w|$ represents the chain used for the indexation of $w$. This notation is interpreted as an ordered sequence of symbols from $\Sigma$.

An ontology $(\Sigma,\leq)$ is termed a \emph{bipolarized ontology} if it possesses a nilpotent operator $(\_)^+:\Sigma\rightarrow\Sigma$ such that $\Sigma=\Sigma_I\cup\Sigma_O$, where $(\Sigma_I)^+=\Sigma_O$. This operator preserves the ontology structure, meaning that for any signs $\lambda_0$ and $\lambda_1$, if $\lambda_0\leq\lambda_1$, then $\lambda_0^+\leq\lambda_1^+$. Here, the set $\Sigma_I$ is referred to as the set of \emph{input symbols} for $\Sigma$, while $\Sigma_O$ is termed the set of \emph{output symbols}. If a symbol $\lambda$ belongs to $\Sigma_I$, its dual, denoted by $\lambda^+$, is an output symbol. We denote a bipolarized ontology as $(\Sigma^+,\leq)$.

Using the concept of lifting, we define, for every word $w\in (Chains\downarrow \Sigma^+)$, the following words:
\begin{enumerate}
  \item $o(w)\in (Chains\downarrow\Sigma_O)$, which is defined by all output symbols present in $w$, and
  \item $i(w)\in (Chains\downarrow \Sigma_I)$, which is defined by all input symbols present in $w$.
\end{enumerate}

\begin{figure}[h]
\[
\xymatrix{\ar@{}[dr]|(.3)\lrcorner {|i(w)|}\ar[r]^{i(w)} \ar[d]_{\subseteq}& \Sigma_I \ar[d]^{\subseteq} \\
          |w|\ar[r]^{w} & \Sigma\ }\quad\quad\quad
\xymatrix{\ar@{}[dr]|(.3)\lrcorner {|o(w)|}\ar[r]^{o(w)} \ar[d]_{\subseteq}& \Sigma_O \ar[d]^{\subseteq} \\
          |w|\ar[r]^{w} & \Sigma\ }
\]
\caption{Pullbacks used to select input and output signs from word $w$.}\label{inoutsign}
\end{figure}
Given a word $w$, we define $\Sigma(w)$ as the set of symbols used in $w$, denoted as $\Sigma(w)\subset\Sigma^+$. For a bipolarized ontology, let $w$ and $w'$ be two words. We can define $w\otimes w'$ as a substring resulting from the concatenation of $w$ and $w'$, which can be inductively described by the following algorithm:

\begin{alg}\label{op:StringGluing}
    \textbf{Algorithm: String Gluing (Input: $w,w'\in \Sigma^+$, Output: $w_i.w'_i$)}
    \begin{enumerate}
        \item Initialize $w_0=w$ and $w'_0=w'$.
        \item Let $\lambda$ be the first output symbol in $w_i$ that has its dual, $\lambda^+$, in $w'_i$, or one of its generalizations:
            \begin{enumerate}
                \item Generate $w_{i+1}$ by removing the first occurrence of $\lambda$ from $w_i$.
                \item Generate $w'_{i+1}$ by removing the first occurrence of $\lambda^+$ or one of its generalizations from $w'_i$.
            \end{enumerate}
        \item Repeat step 2 as long as there exist signs in $\Sigma(w_{i+1})$ that have a dual or a generalization in $\Sigma(w_{i+1})$.
    \end{enumerate}
\end{alg}

In this context, the word $w\otimes w'$ can be understood as the result of the ordered elimination of output symbols in $w$ and input symbols in $w'$ that are linked by duality.

From the definition of the operator $\otimes$, we can establish the following properties:
\begin{prop}
For every pair of words $w$, $w'$, and $w''$ in a bipolarized ontology $(\Sigma,\leq)$, the following properties hold:
    \begin{enumerate}
        \item Associativity: $w\otimes  (w'\otimes  w'') = (w\otimes   w')\otimes  w''$;
        \item Symmetry under Certain Conditions: If $w'$ (and $w$) does not have the dual, nor any of its generalizations, of signs from $w$ (and $w'$, respectively), then $w'\otimes w = w\otimes  w'$;
        \item Identity: $w\otimes \bot=\bot \otimes w  = w$.
    \end{enumerate}
\end{prop}

In our pursuit of developing a framework for library specification, we assume that the requirements for processes' inputs and outputs are codified using signs from a polarized ontology $(\Sigma^+,\leq)$. Consequently, \emph{the universe of libraries}, where component requirements are specified over the polarized ontology $\Sigma^+$, can be conceptualized as the comma category:
\[(Chains \downarrow(Chains \downarrow \Sigma^+)).\]
In this context, a library can be viewed as a list of components, each specified using words defined over $\Sigma^+$.

A \emph{library specification} is represented by a map $L:|L|\rightarrow (Chains\downarrow \Sigma^+)$, where each node in the chain $|L|$ is termed a \emph{component label} or simply a \emph{sign} within the library. For a given component label $r\in |L|$, we can interpret $L(r)=w:|w|\rightarrow \Sigma^+$ as specifying the component's input requirements, denoted by $i(w)$, and its output requirements, denoted by $o(w)$. Consequently, a library can be visualized as an oriented multi-graph. Here, multi-arrows represent selections of objects in $(Chains\downarrow \Sigma^+)$, and nodes correspond to objects from $(Chains\downarrow \Sigma_I)$.

Consider $L\in(Chains \downarrow(Chains \downarrow \Sigma^+))$. If $L(r)=w$, the label $r$ is understood as establishing a dependence between families of nodes $i(w)$ and $o^+(w)$. In this context, we denote this relationship as
\[r:i(w)\rightarrow o^+(w),\] or succinctly as $r\in L$. This notation defines a multi-arrow in the multi-graph $\G(L)$ associated with the library $L$.

A homomorphism between libraries is characterized as a morphism within $(Chains \downarrow(Chains \downarrow \Sigma^+))$. For any morphism $f:L_0\rightarrow L_1$ between libraries $L_0$ and $L_1$, there exists an associated multi-graph homomorphism $\G(f):\G(L_0)\rightarrow \G(L_1)$. This homomorphism establishes a correspondence between signs and a one-to-one mapping between component labels in $L_0$ and $L_1$, while preserving the component requirements.

Naturally, an order relation can be established between libraries. We denote $L_0\leq L_1$ if for every $r\in |L_0|$, we also have $r\in |L_1|$. That is, every component present in $L_0$ is also found in $L_1$. In such a scenario, $L_0$ is referred to as a \emph{sublibrary} of $L_1$, and the corresponding homomorphism $f:L_0\rightarrow L_1$ is termed the \emph{library inclusion}.

We define $L^\ast$ as the \emph{free monoid} on $L$ with respect to the operator $\otimes$. Formally, for a given library $L\in (Chains \downarrow(Chains \downarrow \Sigma^+))$, $L^\ast$ is characterized as the $\otimes$-closure of $L$. In other words, it is the smallest library in $(Chains \downarrow (Chains\downarrow \Sigma^+))$ satisfying the following conditions:
\begin{enumerate}
  \item Every word generated using signs from $L$ is a label in $L^\ast$.
  \item The empty word defines a label for a component with null requirements, denoted as $\perp:\perp\rightarrow\perp$.
  \item $L$ is a sublibrary of $L^\ast$.
  \item For any $s\in L^\ast$ such that $s=r_1\otimes r_2\otimes \ldots\otimes r_n$, we have
  \[L^\ast(s)=L(r_1)\otimes L(r_2)\otimes\ldots\otimes L(r_n).\]
\end{enumerate}
It's worth noting that the empty word $\bot$ serves as a label in $L^\ast$. Given that $L$ is a sublibrary of $L^\ast$, the requirements associated with a label in $L^\ast$ can be interpreted as the prerequisites for the component connections defined by the label. Hence, a word in $L^\ast$ can be conceptualized as a circuit constructed through the interconnection of components from $L$.

A library serves as a formal system that allows us to stratify circuits into different levels of abstraction. The abstraction level of a circuit is determined by the number of refinement steps required to derive an equivalent circuit using only atomic components. To elucidate the concept of circuit refinement, consider an ontology $(\Sigma^+,\leq)$ where the order defined for signs can be extended to words. We denote $\lambda_0\ldots\lambda_n\leq\lambda_0'\ldots\lambda_n'$ if and only if $\lambda_i\leq\lambda'_i$ in $(\Sigma^+,\leq)$. When two words from the ontology satisfy $w\leq w'$, we say that $w'$ is a \emph{generalization} of $w$ within the ontology.

Circuit refinement hinges on the concept of \emph{semantics for a library} denoted by $L$. It comprises a pair of equivalence relations $(\equiv_l,\equiv_w)$, where $\equiv_l$ is defined for labels in $L^\ast$, and $\equiv_w$ is defined for words in $\Sigma^\ast$. Specifically, the relations are defined such that:
\[
\text{If } l_0\equiv_l l_1 \text{ then } L^\ast(l_0)\equiv_w L^\ast(l_1).
\]
In the context of $(L^\ast,\equiv_l,\equiv_w)$, a label $s$ is termed a \emph{decomposable component} if there exist words $s_0$ and $s_1$ such that:
\[
s\equiv_l s_0\otimes s_1
\]
Labels that cannot be decomposed are referred to as \emph{atomic components}. Thus, if a label in $L^\ast$ is atomic, it corresponds to a label in library $L$.

A \emph{normal form} representation for a label \( s \in L^\ast \) is a sequence of atomic components
\[
(r_0, r_1, r_2, \ldots, r_n)
\]
such that
\[
s \equiv_l r_0 \otimes r_1 \otimes r_2 \otimes \ldots \otimes r_n.
\]

Given a library \( L \), we define the \emph{library of atomic components} for \( (L^\ast, \equiv_l, \equiv_w) \) as the library
\[
L_{at} \leq L
\]
where \( s \in L_{at} \) if and only if \( s \) is atomic within \( (L^\ast, \equiv_l, \equiv_w) \).

Our objective is to describe structures, such as Architectural connectors, using a graphical language that outlines the global organization of complex structures based on simpler ones. In this context, each component is associated with a graphical representation, and the circuit specification emerges from the interconnection between components that adhere to a set of rules and a glue specification \cite{GarlanAllen97}. The crux of our approach lies in providing a general framework that endows circuits with explicit semantic meaning. To formalize this concept, we associate a multi-graph \( \G(L) \) with a library \( L \), where the nodes are symbols from \( \Sigma_I \), and the multi-edges represent components. Each component \( s \) is characterized by its input \( i(L(s)) \) and output \( o(L(s)) \).

We denote the comma category
\[
\G^\ast(L) = (Mgraph \downarrow \G(L))
\]
where the objects are homomorphisms defined between a multi-graph and the multi-graph \( \G(L) \).

For a library \( L \) with a polarized ontology \( (\Sigma^+, \leq) \), any diagram \( D \in \G^\ast(L) \) can be represented as a library
\[
L(D) \in (Chains \downarrow (Chains \downarrow \Sigma^+)),
\]
where the component labels are multi-arcs derived from \( D \). Each multi-arc must be associated with a word defined by concatenating, into a single word, the labels of the source vertices and the dual labels of the target vertices of the multi-arc. Given a diagram \( D \in \G^\ast(L) \) defined by:
\begin{enumerate}
  \item the source map \( i: |D| \rightarrow (Chains \downarrow \Sigma_I) \), and
  \item the target map \( o: |D| \rightarrow (Chains \downarrow \Sigma_I) \),
\end{enumerate}
the library \( L(D) \) has two associated words defined using symbols from \( \Sigma_I \): its input requisites \( i(D) \) and its output structures \( o(D) \), where:
\begin{enumerate}
  \item \( i(D) \) is the word obtained by concatenating labels belonging to vertices without input multi-arcs, and
  \item \( o(D) \) is the word defined by concatenating the dual labels of vertices without output multi-arcs.
\end{enumerate}

In the category \( \G^\ast(L) \), for every pair of diagrams \( D \) and \( D' \), we define the diagram \( D \otimes D' \) by gluing together vertices with identical labels from \( o(D) \) and \( i(D') \), while treating others as distinct. The order used for gluing vertices must respect the order determined by the chain of symbols used to define the words \( o(D) \) and \( i(D') \).

When \( \otimes \) is restricted to pairs of diagrams \( D \) and \( D' \) such that \( i(D) = o(D') \), we use this operator as a "composition" between relations specified using multi-graphs. With it, we define a category with objects as words from \( (Chains \downarrow \Sigma_I) \) and morphisms as diagrams from \( \G^\ast(L) \). Given a diagram \( D \), the condition \( i(D) = w \) and \( o(D) = w' \) is denoted by \( D: w \rightharpoonup w' \). For diagrams \( D \) and \( D' \) from \( \G^\ast(L) \), if \( D' \) is a subobject of \( D \), denoted by \( D' \leq D \), then there exists an epimorphism in \( \G^\ast(L) \) from \( D' \) to \( D \). Equivalently, there is a decomposition \( D = D'' \otimes D' \otimes D''' \).

A diagram \( D \) is said to be \emph{decomposable} if there exist two non-null subobjects \( D' \) and \( D'' \) such that \( D = D' \otimes D'' \). If a diagram is not decomposable, it is termed \emph{atomic}. Let \( \G_{at}^\ast(L) \) denote the class of atomic diagrams in \( \G^\ast(L) \). We can view atomic diagrams as fundamental building blocks for generating more complex diagrams.

A functor \( F: \G^\ast(L) \rightarrow \G^\ast(L) \), where \( L \) has a semantic interpretation, is termed a \emph{diagram refinement} if:

\begin{enumerate}
    \item \( F(D \otimes D') \equiv_l F(D) \otimes F(D') \): This means that the refinement of a composite diagram is semantically equivalent to the refinement of its constituent parts.
    \item \( F(D) \equiv_l D \): This indicates that the refinement of a diagram is semantically equivalent to the original diagram itself.
    \item \( F(D) = D \) if and only if \( D \in \G_{at}^\ast(L) \): This implies that atomic elements cannot be further simplified.
\end{enumerate}

A refinement can be understood as a rewriting rule that enables the unpacking of encapsulated subdiagrams. If a diagram is a fixed point of the refinement function, we say that it is in \emph{normal form}. Given that diagrams are finite structures, for every diagram \( D \), we can find at least one representation of \( D \) in normal form within a finite number of steps.

A diagram refinement \( F: \G^\ast(L) \rightarrow \G^\ast(L) \) possesses the notable property of defining a partial order on \( \G^\ast(L) \), denoted by \( \leq_F \). Specifically, for any two diagrams \( D \) and \( D' \) in \( \G^\ast(L) \), we say \( D \leq_F D' \) if \( F(D') = D \). In other words, \( D \) is a refinement of \( D' \) via \( F \). In this context, \( D' \) is referred to as a \emph{generalization} of \( D \).

\begin{exam}[Signatures as libraries]\label{ex:signature}
A signature can be represented through a library of components, where each component symbolizes a function. The arity of each function is encoded within the component's requirements. Formally, following \cite{ParMakkai89}, a signature \( \Sigma = (S, T, \text{ar}) \), with type symbols from \( S \), is comprised of:

\begin{itemize}
    \item A finite set \( T \) of function symbols (or operators) \( f, g, \ldots \).
    \item For each function symbol \( f \), an arity \( \text{ar}(f) = (<a_i>_I, b) \) is defined, consisting of a chain of input type symbols and one output type symbol. In this context, we write \( i(f) = <a_i>_I \) and \( o(f) = <b> \).
\end{itemize}

We can organize the set \( T \) of function symbols by using these symbols as labels for components, with their requirements encoded over the polarized alphabet \( S^+ \) generated from \( S \). The set \( S^+ \) is constructed by adding a new dual symbol \( a^+ \) for each type symbol \( a \) in \( S \). The library associated with the signature \( \Sigma \) is denoted by \( L(\Sigma) \).

A constant of type \( a \) is a function symbol in a signature with arity \( (<>, a) \), meaning it has no inputs and \( a \) as its output. It is common practice to utilize a countably infinite set of variables for each type present in the signature. These variables are represented in a diagram by using inputs on components without associated links, thereby defining the set of diagram sources.

\end{exam}

Below, we present examples of libraries associated with models generated by machine learning algorithms as discussed in \cite{MichalskiMitchellCarbonell86}. These examples serve to illustrate fuzzy structures:

\begin{exam}[Binary Library \L\(_B(S,C) \)]\label{binarylibrary}
Binary libraries are defined using a set of component labels \( C \) and a set of type signs \( S \). A binary library \L\(_B(S,C) \) presupposes the existence of:

\begin{itemize}
    \item A sign \( l \in S \), which can be interpreted as the set \( \Omega \) of truth values in \( \text{Set}(\Omega) \).
    \item A constant \( \top: l^+ \) in \( C \), interpreted as 'true'.
\end{itemize}

In the set \( S \), we must have components of type \( =_s : ssl^+ \), one for each symbol \( s \in S \). These are interpreted as similarities \([\cdot = \cdot]\) corresponding to the interpretation of \( s \). Additionally, components of type \( c : s^+ \) must be defined, where \( c \in C \) and \( s \in S \). These are interpreted as constants selected based on the interpretation of \( s \).

Given that datasets or tables can be encoded using these primitives, binary libraries are also referred to as 'dataset libraries'.
\end{exam}

\begin{exam}[Linear Library \L\(_L(S,C) \)]\label{linearlibrary}
Linear libraries, denoted as \L\(_L(S,C) \), extend binary libraries by incorporating additional components:

\begin{itemize}
    \item Similarity components \( =_s : ssl^+ \) for each symbol \( s \in S \), interpreted as indicating equality.
    \item Components \( \geq_s : ssl^+ \) for each symbol \( s \in S \), which can be interpreted as specifying a total order.
    \item Constant components \( c : s^+ \), where \( c \in C \) and \( s \in S \).
\end{itemize}

Linear libraries are particularly associated with processes for the discretization of continuous domains. For this reason, they are also referred to as 'grid libraries'.
\end{exam}

\begin{exam}[Additive Library \L\(_A(S,C) \)]\label{addlibrary}
Additive libraries, denoted as \L\(_A(S,C) \), are extensions of linear libraries and are defined as follows:

\begin{itemize}
    \item Equality components \( =_s : ssl^+ \) for each symbol \( s \in S \).
    \item Order components \( \geq_s : ssl^+ \) for each symbol \( s \in S \).
    \item Addition components \( +_s : sss^+ \) for each symbol \( s \in S \), interpretable as addition.
    \item Constant components \( b : s^+ \), where \( b \in C \) and \( s \in S \).
\end{itemize}
\end{exam}

\begin{exam}[Multiplicative Library \L\(_M(S,C) \)]\label{multlibary}
Multiplicative libraries, denoted as \L\(_M(S,C) \), are extensions of additive libraries and are defined as follows:

\begin{itemize}
    \item Equality components \( =_s : ssl^+ \) for each symbol \( s \in S \).
    \item Order components \( \geq_s : ssl^+ \) for each symbol \( s \in S \).
    \item Addition components \( +_s : sss^+ \) for each symbol \( s \in S \).
    \item Multiplication components \( \times_s : sss^+ \) for each symbol \( s \in S \), interpretable as multiplication.
    \item Constant components \( b : s^+ \), where \( b \in C \) and \( s \in S \).
\end{itemize}
\end{exam}

\chapter{Modeling Libraries and Graphic Languages}\label{modeling libraries}

Given the structural compatibility between libraries and multi-graphs, it is natural to consider the soundness of library semantics in \( Set(\Omega) \) as equivalent to the preservation of library structure.

A model for a library \( (L^\ast, \equiv_l, \equiv_w) \) in \( Set(\Omega) \) is a multi-graph homomorphism \( M \) from the \emph{library parser graph} \( \G(L^\ast) \) to \( Set(\Omega) \), denoted as
\[ M: \G(L^\ast) \rightarrow Set(\Omega) \]
such that:
\begin{enumerate}
    \item Equivalent components are interpreted as the same multi-morphism, i.e., \( M(r) = M(r') \) if \( r \equiv_l r' \).
    \item Component gluing translates to multi-morphism composition, i.e.,
    \[ M(r \otimes r') = M(r) \otimes M(r'); \]
    \item It preserves component requirements and truth value distribution, i.e., if \( r: w \rightharpoonup w' \), then
    \[ M(r): M(w) \rightharpoonup M(w') \text{ and } M(r)^\circ \otimes M(w) \otimes M(r) = M(w'); \]
    \item It maintains the sign's ontological structure, i.e., if sign \( l \) is a generalization of sign \( l' \) (i.e., if \( l' \leq l \)), then \( M(l') \leq M(l) \);
    \item Words are mapped to chains of \( \Omega \)-set products defined in \ref{ProdSimil}, i.e., if \( w = s_1 s_2 \ldots s_n \), then \( M(w) = \Pi_i M(s_i) \);
    \item Equivalent words are mapped to the same \( \Omega \)-set, as defined in \ref{def:isomorphism}, i.e., if \( w \equiv_w w' \), then \( M(w) = M(w') \).
\end{enumerate}

In other words, a model transforms multi-arcs into multi-morphisms while preserving their structure and the semantics induced through relations \( \equiv_l \) and \( \equiv_w \). Property (3) stipulates the preservation of truth value distribution by component interpretation. The class of models for a library \( (L^\ast, \equiv_l, \equiv_w) \) is used subsequently in the definition of a category of models given by
\[ \text{Mod}(L^\ast, \equiv_l, \equiv_w). \]

A model for \( L^\ast \) can be defined by lifting the interpretation of atomic components to circuits. To achieve this, it is important to note that since a model preserves component gluing, it can be defined by fixing interpretations for its atomic components. This is expressed by the following completion principle:

\begin{prop}[Universal Property]
Let \( L_{at} \) be the sublibrary defined by atomic components in \( (L^\ast, \equiv_l, \equiv_w) \). Every multi-graph homomorphism
\[ M: \G(L_{at}) \rightarrow Set(\Omega), \]
defines a unique model
\[ M^\ast: \G(L^\ast) \rightarrow Set(\Omega), \]
for \( (L^\ast, \equiv_l, \equiv_w) \), such that
\[ M^\ast \circ i = M, \]
where \( i \) is the homomorphism defined by library inclusion.
\end{prop}

The proof of this result is constructed by defining \( M^*(r) = \bigotimes_i M(r_i) \) if the circuit \( r \) has a normal form given by a sequence
\[
(r_0, r_1, r_2, \ldots, r_n),
\]
i.e., the \( r_i \)'s are atomic and
\[
r \equiv_l r_0 \otimes r_1 \otimes r_2 \otimes \ldots \otimes r_n.
\]

For every component label \( r \in L \), we refer to a \emph{realization} for \( r \) through \( M \) in \( \text{Mod}(L^\ast, \equiv_l, \equiv_w) \) or a \emph{Chu representation} of \( r \) as an \emph{epi multi-morphism}
\[
M(r): M(i(r)) \rightharpoonup M(o(r)) \quad \text{such that} \quad M(r)^\circ \otimes M(i(r)) \otimes M(r) = M(o(r)).
\]
In this case, if \( r = r' \otimes r'' \) in \( L^\ast \), then \( M(r) \) can be decomposed in \( Set(\Omega) \) as
\[
M(r') \otimes M(r'').
\]

Since library refinement preserves semantics, it is idempotent with respect to library models. Given a model \( M \in \text{Mod}(L^\ast, \equiv_l, \equiv_w) \) and if \( F: L^\ast \rightarrow L^\ast \) is a refinement in the library \( (L^\ast, \equiv_l, \equiv_w) \), we have
\[
M( F^n(D) ) = M(D), \quad \text{for every configuration } D \in \G(L^\ast).
\]
This property can be used to characterize refinement as follows:

\begin{prop}
A library homomorphism \( F: L^\ast \rightarrow L^\ast \) having atomic components as fixed points is a refinement in \( (L^\ast, \equiv_l, \equiv_w) \) if and only if for every model \( M \in \text{Mod}(L^\ast, \equiv_l, \equiv_w) \), we have
\[
M \circ F^n = M \quad \text{for each natural number } n.
\]
\end{prop}

In this context, for every component \( r \) which is a refinement of \( r' \) by \( F \), i.e., \( r \leq_F r' \), we have, for every library model \( M \), that \( M(r') = M(r) \).


A library can be perceived as an analytic grammar, and we can leverage them to characterize languages. We define the \emph{graphic language} associated with a library \( L \) as the set of valid finite configurations using components indexed by \( L \). A configuration of components \( D \) is \emph{valid or allowed} in \( L \) if
\[
D \in \G^\ast(L),
\]
i.e., if \( D \) is a multi-graph homomorphism
\[
D: \G \rightarrow \G(L).
\]
Formally, given a library \( L \in (Chains \downarrow (Chains \downarrow \Sigma^+)) \), a graphic word \( D \) defined by \( L \) is a finite configuration
\[
D \in (Mgraph \downarrow \G(L)).
\]
In this context, a word in the language is a multi-graph homomorphism where the multi-arrows are library components. Since the homomorphism \( D \) has \( \G(L) \) as its codomain, it satisfies library constraints and is referred to as the \emph{parsing} of the word \( D \).

The \emph{graphic language defined by} library \( L \) is denoted by \( Lang(L) \) and is defined as the comma category
\[
\G^\ast(L) = (Mgraph \downarrow \G(L))
\]
of allowed configurations in \( L \). Given an allowed configuration \( D: \G \rightarrow \G(L) \), we refer to \( \G \) as the \emph{configuration shape}, and \( D(\G) \) is termed a \emph{word} or \emph{diagram} in the language defined by \( L \).

Given a configuration \( D \in Lang(L) \) and a model \( M \in Mod(L^\ast, \equiv_l, \equiv_w) \), we define
\begin{enumerate}
  \item \( i(M(D)) = M(i(D)) \) and
  \item \( o(M(D)) = M(o(D)) \),
\end{enumerate}
where \( M(i(D)) \) and \( M(o(D)) \) denote \( \Omega \)-sets in \( Set(\Omega) \) used to impart meaning to the multi-diagram input and output vertices.

We can represent the structure of word interpretation as a multi-morphism using limits.

\begin{defn}[Limits as Multi-morphisms]\label{def:word interpret}
Given a model \( M \in Mod(L^\ast, \equiv_l, \equiv_w) \), the \emph{interpretation for a configuration} \( D: \G \rightarrow \G(L) \) through \( M \) is denoted as \( \text{Lim } MD \) and is represented by the multi-morphism
\[ 
M(i(D)) \rightharpoonup M(o(D)).
\]
In this context, we use \( M(D) \) to denote the multi-morphism \( \text{Lim } MD \).
\end{defn}

Naturally, we also define

\begin{defn}[Co-limits as Multi-morphisms]\label{def:colim interpret}
Given a model \( M \in Mod(L^\ast, \equiv_l, \equiv_w) \) and a diagram \( D: \G \rightarrow \G(L) \), its co-limit \( \text{coLim } MD \) can be viewed as a multi-morphism
\[ 
\text{coLim } MD : M(i(D)) \rightharpoonup M(o(D)).
\]
\end{defn}

By definition, a model for a library preserves component decomposition, which can be extended to multi-diagrams when interpreted within a basic logic.

\begin{prop}
Let \( \Omega \) be a basic logic. If the multi-diagram \( D \) is a word in the language defined by the library \( L \) and can be obtained by gluing diagrams \( D_1 \) and \( D_2 \), i.e., \( D = D_1 \otimes D_2 \), then the interpretation of the word \( D \) results from composing the interpretations of \( D_1 \) and \( D_2 \), given by:
\[ 
M(D) = [\cdot = \cdot]_{\otimes H} \Rightarrow (M(D_1) \otimes M(D_2)),
\]
where \( H = M(i(D_2) \cap o(D_1)) \) represents the set of \( \Omega \)-sets that are sources for diagram \( D_2 \) and targets for \( D_1 \).
\end{prop}

Note that if \( M(D_1)^\circ \otimes \alpha \otimes M(D_1) = \beta \) and \( M(D_2)^\circ \otimes \beta \otimes M(D_2) = \gamma \), then 
\[ 
M(D_1 \otimes D_2)^\circ \otimes \alpha \otimes M(D_1 \otimes D_2) = M(D_2)^\circ \otimes M(D_1)^\circ \otimes \alpha \otimes M(D_1) \otimes M(D_2) = \gamma.
\]
Within a basic logic \( \Omega \), this strategy can be extended to the colimit of configurations:
\[ 
\text{colim } M(D_1 \otimes D_2) = [\cdot = \cdot]_{\otimes H} \Rightarrow (\text{colim } M(D_1) \otimes \text{colim } M(D_2)),
\]
where \( H = M(i(D_2) \cap o(D_1)) \).

\chapter{Library descriptive power}\label{descritive power}

Let's now define the structure for the category of models for a library, \( Mod(L^\ast,\equiv_l,\equiv_w) \). The objects in this category are models 
\[ M:\G(L^\ast)\rightarrow Set(\Omega), \]
and the morphisms are natural transformations. In this context, taking \( D = \G(L^\ast) \) as the graphic library structure, a natural transformation from model \( M_1 \) to model \( M_2 \) is represented by a pair of epi multi-morphisms \((f,g):M_1\Rightarrow M_2\), satisfying the condition
\[ f \otimes M_2(D) = M_1(D) \otimes g. \]
By the definition of an epi multi-morphism, it follows that 
\[ f^\circ \otimes M_1(i(D)) \otimes f = M_2(i(D)) \]
and 
\[ g^\circ \otimes M_1(o(D)) \otimes g = M_2(o(D)). \]

\begin{figure}[h]
\[
\xymatrix{
M_1(i(D)) \ar@_{->}[r]_f \ar@_{->}[d]_{M_1(D)} & M_2(i(D)) \ar@_{->}[d]_{M_2(D)} \\
M_1(o(D)) \ar@_{->}[r]_g & M_2(o(D))
}
\]
\caption{Natural transformation \((f,g)\).}\label{naturaltransf}
\end{figure}

We use the composition of multi-morphisms to define the composition of natural transformations. Given two natural transformations \((f_1,g_1):M_1\Rightarrow M_2\) and \((f_2,g_2):M_2\Rightarrow M_3\), we define
\[
(f_1,g_1) \otimes (f_2,g_2) = (f_1 \otimes f_2, g_1 \otimes g_2).
\]
A model \( M \) has the identity in \( Mod(L^\ast,\equiv_l,\equiv_w) \) given by the natural transformation 
\[ (1_{i(M(D))}, 1_{o(M(D))}), \]
where both epi multi-morphisms are defined using the identity relation in \( \Omega \).

The usual limits and colimits in \( Mod(L^\ast,\equiv_l,\equiv_w) \) are computed based on those made in the category of \( \Omega \)-sets and epi multi-morphisms, denoted as \( epi\text{-}Set(\Omega) \). In this context, the product (in the usual sense) exists for two \( \Omega \)-sets \( \alpha:A \) and \( \beta:B \) if
\[
\bigvee_{a,a'}\alpha(a,a') = \top \text{ and } \bigvee_{b,b'}\beta(b,b') = \top,
\]
and it is defined by the object \( \alpha\otimes\beta:A\times B \) along with the usual projections \( \pi_A:A\times A\rightarrow A \) and \( \pi_B:B\times B\rightarrow B \) in \( Set \). This product is codified as a function in \( Set(\Omega) \). Notably, for instance, \( \pi_A \) is an epi multi-morphism. This is evident since \( \pi_A^\circ\otimes(\alpha\otimes\beta)\otimes \pi_A \) defines the multi-diagram shown in Figure \ref{product}.

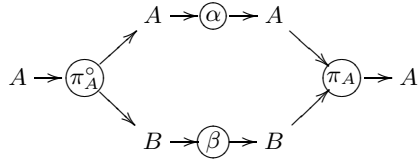
\begin{figure}[h]
\[
\small
\xymatrix @=10pt {
&& A \ar[r] & *+[o][F-]{\alpha} \ar[r] & A \ar[dr] && \\
A \ar[r] & *+[o][F-]{\pi_A^\circ} \ar[ur] \ar[dr] &&&& *+[o][F-]{\pi_A} \ar[r] & A \\
&& B \ar[r] & *+[o][F-]{\beta} \ar[r] & B \ar[ur] &&
}
\]
\caption{Multi-morphism \( \pi_A^\circ\otimes(\alpha\otimes\beta)\otimes \pi_A \).}\label{product}
\end{figure}

By composition, we have
\begin{align*}
&\bigvee_{b'}\bigvee_{b''}\bigvee_{a'}\bigvee_{a''} \pi_A^\circ(a,b',a')\otimes \alpha(a',a'')\otimes \beta(b',b'')\otimes\pi_A(a'',b'',a''') \\
&= \bigvee_{b'}\bigvee_{a'}\pi_A^\circ(a,b',a') \otimes\bigvee_{a''}\bigvee_{b''}(\alpha(a',a'')\otimes\beta(b',b'')\otimes\pi_A(a'',b'',a''') ) \\
&= \bigvee_{b'}\pi_A^\circ(a,b',a) \otimes\bigvee_{b''}(\alpha(a,a''')\otimes\beta(b',b'')\otimes\pi_A(a''',b'',a''') ) \\
&= \bigvee_{b'}\bigvee_{b''}(\alpha(a,a''')\otimes\beta(b',b'')) \\
&= \alpha(a,a''')\otimes\bigvee_{b'}\bigvee_{b''}\beta(b',b'') \\
&= \alpha(a,a''').
\end{align*}

The category \( epi\text{-}Set(\Omega) \) does not have an initial object. However, for an \( \Omega \)-object \( \alpha:A \) with a unique factorization \( \alpha = f^\circ\otimes f \), the multi-morphism \( f:\emptyset\rightharpoonup A \) is the only epi multi-morphism, as \( f^\circ\otimes f = \alpha \). On the other hand, there is only one epi multi-morphism to \( \emptyset:\emptyset \), which is given by the empty relation \( \emptyset:A\rightharpoonup \emptyset \), since \( \emptyset^\circ\otimes\alpha\otimes\emptyset = \emptyset \).

Given a pair of epi multi-morphisms \( f,g:A\rightarrow B \) from \( \alpha:A \) to \( \beta:B \), the equalizer for \( f \) and \( g \) exists and is given by \( \gamma:A \) if and only if
\[
(f\otimes g)^\circ\otimes\gamma\otimes (f\otimes g) = \beta.
\]
For every pair of epi multi-morphisms \( f \) and \( g \), there exists a coequalizer, which is given by \( \gamma:B \) such that
\[
\gamma = (f\otimes g)^\circ\otimes\alpha\otimes (f\otimes g).
\]
Furthermore, every family \( (f_i:\alpha_i\rightharpoonup \beta) \) of epi multi-morphisms has wild pushouts given by the product \( \otimes_i\alpha_i:\prod_iA_i \) and its projections \( \pi_j:\prod_iA_i\rightarrow A_j \). This is evident from the relations \( \pi_j^\circ\otimes_i\alpha_i\pi_j = \alpha_j \) and \( f_j^\circ\otimes\pi_j^\circ\otimes_i\alpha_i\pi_j\otimes f_j = f_j^\circ\otimes\alpha_j\otimes f_j = \beta \).

Since \( epi\text{-}Set(\Omega) \) possesses coequalizers and wild pushouts, it also has connected colimits (see \cite{Borceux94}). By the definition of natural transformation in \( Mod(L^\ast,\equiv_l,\equiv_w) \), we have:

\begin{prop}
The category \(Mod(L^\ast,\equiv_l,\equiv_w)\) possesses connected colimits in the conventional sense.
\end{prop}

Given that \(epi\text{-}Set(\Omega)\) and \(Mod(L^\ast,\equiv_l,\equiv_w)\) have connected colimits, they also have directed colimits. Specifically, a colimit exists for diagrams \(D:(I,\leq)\rightarrow epi\text{-}Set(\Omega)\) where \( (I,\leq) \) is a poset. If the vertices are models, then its colimit is a model (refer to the definition in \cite{RosickyAdamek94}).

Following \cite{RosickyAdamek94}, a category is said to be accessible if it has directed colimits and possesses a set \( \mathcal{A} \) of presentable objects, such that every object is a direct colimit of objects from \( \mathcal{A} \). Accessible categories can be characterized by the following proposition:

\begin{prop}\cite{RosickyAdamek94}
Every small category with split idempotents is accessible.
\end{prop}

Here, a category has split idempotents if for every morphism \( f:A\rightarrow A \) with \( f.f=f \), there exists a factorization \( f=i.p \) where \( p.i=id_A \).

\begin{exam}
If \( \Omega \) is an ML-algebra with at least three logic values \( \bot<\lambda<\top \), the multi-morphism 
\[
f=\left[
     \begin{array}{cc}
       \top   & \alpha \\
       \alpha & \top \\
     \end{array}
   \right]
\]
is an epi multi-morphism. While it is idempotent, it is non-splitable. This is because for every factorization \( f=i\otimes p \), \( p\otimes i \neq id \) as per Proposition \ref{prop:comprestriction}.
\end{exam}

The phenomena illustrated by this example are the norm in \( epi\text{-}Set(\Omega) \) and \( Mod(L^\ast,\equiv_l,\equiv_w) \). Given that most idempotents are not split idempotents, we have:

\begin{prop}
Let \( \Omega \) be an ML-algebra with more than two logic values. Then, the categories of models of libraries \( Mod(L^\ast,\equiv_l,\equiv_w) \) are not accessible.
\end{prop}

This implies that we cannot employ Ehresmann sketches to specify the category of models of a library \cite{RosickyAdamek94}. In other words, model categories cannot be axiomatized by basic theories in first-order logic.

\chapter{Sign Systems and Semiotics}\label{Signs Semiotics}

Let's now determine the fundamental structure required in a library to define useful fuzzy structures.

\begin{exam}[Signatures as Libraries]
Let \( \Sigma \) be a signature. As illustrated in Example \ref{ex:signature}, a signature can be viewed as a library. The set \( T(\Sigma) \) of terms defined by \( \Sigma \) is provided, in \cite{ParMakkai89}, as the least set that satisfies:
\begin{enumerate}
  \item Each variable \( x \) belongs to \( T(\Sigma) \), and if it has type \( b \), we denote it as \( x:b \).
  \item If \( f \in \Sigma \) with \( ar(f)=(<a_i>_I,b) \) and \( <t_i>_I \) is a list of terms from \( T(\Sigma) \) such that \( t_i:a_i \) for every \( i\in I \), then \( f(t_i)_I \) is in \( T(\Sigma) \). This term has type \( b \), and in this case, we denote it as \( f(t_i)_I:b \).
\end{enumerate}
\end{exam}

A term is considered closed if it does not contain any variables. We can establish the isomorphism
\[ T(\Sigma) \cong Lang(L(\Sigma ,S)) \]
if and only if we impose the existence of:
\begin{enumerate}
  \item Special symbols \( c \) and \( v \) in \( S \), as described in Example \ref{ex:signature};
  \item Diagonal components \( \lhd \) in \( |L(\Sigma ,S)| \) such that \( i(\lhd)=a \) and \( o(\lhd)=<a>_I \) in \( L(\Sigma) \). We require one such component for each alphabet symbol \( a \) in \( |L(\Sigma ,S)| \) and each "class" of chain equivalences \( I \). These components represent dependencies in the term structure, which arise from using the same variable multiple times in the term definition.
\end{enumerate}

If \( t \) is a term in \( T(\Sigma) \), then it can be represented as a multi-graph homomorphism
\[ t:\G\rightarrow \G(L(\Sigma,S)). \]
This graph homomorphism is typically referred to as the parsing graph for \( t \). If the variables involved in the definition of a term are distinct, then the diagram \( t \) has the shape of a tree. In Figure \ref{terms}, we present the allowed configurations defining terms, such as
\[ f(x:b,g(y:c,z:d):e):a \]
and
\[ f(x:b,g(y:c,x:b):e):a. \]

\begin{figure}[h]
\[
\small
\xymatrix @=10pt {
\ar[rrd]_b&  & & \\
\ar[rd]_c&   &*+[o][F-]{f}\ar[rr]_a  & &\\
&*+[o][F-]{g}\ar[ru]_e&  & \\
\ar[ru]_d&   & &}
\quad\quad
\xymatrix @=10pt {
\\
\ar[rrd]^c&&   &*+[o][F-]{f}\ar[rr]_a  & &\\
\ar[r]_b&*+[o][F-]{\lhd}\ar[rru]^b\ar[r]_b&*+[o][F-]{g}\ar[ru]_e&  & \\
}
\]
\caption{Multi-morphism $f(x:b,g(y:c,z:d):e):a$ and
$f(x:b,g(y:c,x:b):e):a$.}\label{terms}
\end{figure}
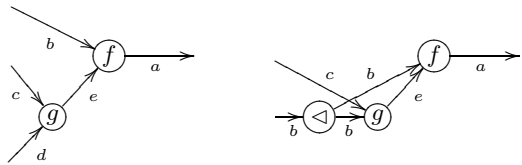

Let \( \Sigma \) be a signature equipped with a special collection of functional symbols denoted by \( =_a \), one for each data type symbol \( a \) used in the signature. Here, we define the arity of \( =_a \) as
\[ ar(=_a)=(<a,a>,l). \]
In the arity of \( =_a \), the symbol \( l \) is to be understood as an identifier for the set of truth-values. This symbol \( l \) should be associated with two constant operators \( T \) and \( F \), both with arity \( (<>,l) \), which are used to identify the true and false values within the associated logical framework. For practical reasons, when dealing with two terms of type \( a \), \( t:a \) and \( s:a \), instead of expressing the term \( =_a(t,s)=T \), we employ the conventional infix notation and write \( t=_a s \), referring to it as an equation of type \( a \). In a signature with these characteristics, we refer to every functional symbol \( f \) with arity of the form
\[ ar(f)=(<a_i>_I,l) \]
as a relational symbol. In this context, the symbol \( =_a \) is relational, and a signature \( \Sigma \) possessing this sort of structure is said to have a logical framework.

A formula \( f \) within a signature \( \Sigma \) with a logical structure is a term that is represented by a multi-graph homomorphism. The output of this homomorphism is a relation, denoted as
\[ o(f)=<l> \].

The example mentioned above necessitates the inclusion of a special symbol "$l$" in \( \Sigma \), along with the presence of specific component labels "=", and "$\lhd$", each having predefined interpretations. Such requirements are commonly observed in various other examples. For these types of labels, we will impose restrictions on the language model structures by assigning interpretations to certain signs and structures that are definable within the library. We refer to these structured systems defined by multi-graphs as specification systems.

\begin{defn}
A \emph{specification system} \( S \), utilizing a library \( L \) (or a \emph{sign system} \( S \) using \( L \)), is defined as a structure \( S = (L,\E,\U,co\U) \) where:
\begin{enumerate}
  \item \( \E \subset \text{Lang}(L) \) is a set of \textbf{finite diagrams}, interpreted as a total multi-morphism (refer to Definition \ref{total}),
  \item \( \U \) is a set of tuples \( (f,D,i(D),o(D)) \) where \( f \) is a component and \( D \) is a \textbf{finite configuration}. Here, \( f \) is interpreted as the multi-morphism defined by the limit of \( D \), with \( i(D) \) and \( o(D) \) serving as the input and output vertices, respectively (refer to Definition \ref{def:word interpret}), and
  \item \( co\U \) is a set of tuples \( (f,D,i(D),o(D)) \) where \( f \) is a component and \( D \) is a \textbf{finite configuration}. In this case, \( f \) is interpreted as the multi-morphism defined by the colimit of \( D \) from sign \( i(D) \) to \( o(D) \) (refer to Definition \ref{def:colim interpret}).
\end{enumerate}
\end{defn}

Within a specification system \( S = (L,\E,\U,co\U) \), while the set \( \E \) defines the structural properties to be preserved by its models, the sets \( \U \) and \( co\U \) impose constraints on the structure for sign interpretations.

Expanding the definition of a model from small Ehresmann sketches to interpretations of sign systems in \( Set(\Omega) \), we introduce:

\begin{defn}[Model for a specification system]\label{Modelspec}
A model \( M \) in \( Mod(L^\ast,\equiv_l,\equiv_w) \), for a library \( L \), is a model for the specification system \( S = (L,\E,\U,co\U) \) if:
\begin{enumerate}
  \item For every diagram \( D \in \E \), \( M(D) \) is a total multi-morphism.
  \item For every tuple \( (f,D,i(D),o(D)) \in \U \), \( M(f) \) is the multi-morphism defined by \( Lim\;MD \) from \( M(i(D)) \) to \( M(o(D)) \).
  \item For every tuple \( (f,D,i(D),o(D)) \in co\U \), \( M(f) \) is the multi-morphism defined by \( coLim\;MD \) from \( M(i(D)) \) to \( M(o(D)) \).
\end{enumerate}
\end{defn}

The category defined by models for a specification system in \( Set(\Omega) \), along with natural transformations between interpretations, is denoted by \( Mod(S) \). We refer to \( (\E,\U,co\U) \) as \emph{the sketch structure of} \( S \). By definition, the category \( Mod(S) \) is a full subcategory of \( Mod(L^\ast,\equiv_l,\equiv_w) \). Additionally, since the category of library models does not have split idempotent relations evaluated in \( \Omega \):

\begin{prop}
Let \( \Omega \) be a non-trivial ML-algebra with more than two truth values. Given a specification system \( S \), the category \( Mod(S) \) is not an accessible category.
\end{prop}

Since the category \( Set \) is a full subcategory of \( Set(\Omega) \), each map can be viewed as a relation. Moreover, due to our ability to encode commutative diagrams using total morphisms and to represent limit cones using the limit structure, as defined in \ref{lim}, we straightforwardly obtain:
\begin{prop}
Every model for an Ehresmann Limit sketch in \( Set \), defined using finite diagrams, can be specified by sign systems in \( Set(\Omega) \) for every non-trivial ML-algebra \( \Omega \).
\end{prop}

Consequently, every algebraic theory (as defined in \cite{RosickyAdamek94}) has a fuzzy counterpart defined in \( Set(\Omega) \).

Given that we will be working with models of specification systems in the subsequent chapters, we introduce the notion of a semiotic, which is a sign system equipped with a model. This structure aligns with the spirit of Goguen's institution \cite{BurstallGoguen83}, associating syntactic and semantic components to a language. Formally:
\begin{defn}[Semiotic system]
A \emph{semiotic system} is a pair \( (S,M) \) defined by a sign system \( S = (L,\E,\U,co\U) \) and a model \( M \in Mod(S) \).
\end{defn}
We will denote by \( Lang(S) \) the language associated with a sign system \( S = (L,\E,\U,co\U) \) or with a semiotic system \( (S,M) \).

In the realm of information systems, a system specification can be perceived as a database structure, while a semiotic defined based on this structure can be likened to a database state. Every update to the database brings about a change in its state, which consequently results in a semiotic change. This is because the change reflects alterations in the relationships between system attributes as encoded in the database tables. Therefore, an information system can be viewed as a semiotic, as it is typically defined by a database instance or state. Consequently, the information stored within the information system can be queried within the context of the associated semiotic.

Let's explore an example of a semiotic and illustrate how we can query it using limits.

\begin{exam}\cite{HandCohenAdams01}
During the IDA'01 Robot Data Challenge, a series of binary data vectors was generated by the perceptual system of a mobile robot. We hypothesize that the resulting time series comprises multiple patterns. Here, a pattern should be understood as a recurring structure in the data that appears either completely or partially more than once. However, the exact boundaries of these patterns, their number, or even their specific structures remain unknown. While some patterns may share similarities, it's possible that no two are exactly alike. The primary challenge lies in identifying these patterns and understanding their underlying structures. A supervised approach to this problem might entail learning to recognize these patterns based on known examples.

The robot dataset consists of a time series comprising 22,535 binary vectors, each of length 9. These vectors were generated as the robot executed 48 repetitions of a straightforward "approach-and-push" plan. In each trial, the robot would visually identify an object, orient itself towards it, approach it rapidly, slow down upon nearing it, and then attempt to push the object. In certain trials, the robot encountered obstacles. In one set of trials, the robot was unable to push the object and would subsequently stall and retreat. In another set, the robot would push the object until it collided with a wall, after which it would stall and retreat. In a third set of trials, the robot successfully pushed the object without any hindrances. Out of the 48 trials, two were deemed anomalous.

The sensory data from the robot was sampled at a rate of 10Hz and processed through a basic perceptual system, yielding values for nine binary variables. These variables represent the robot's state and its rudimentary perceptions of objects within its environment. The variables include: STOP, ROTATE-RIGHT, ROTATE-LEFT, MOVE-FORWARD, NEAR-OBJECT, PUSH, TOUCH, MOVE-BACKWARD, and STALL. For instance, the binary vector [0 1 0 1 1 0 1 0 0] indicates a state where the robot is rotating right, moving forward, near an object, and touching it without pushing. While most of the 512 potential states are not semantically meaningful, the robot's sensors are prone to noise, and its perceptual system can occasionally make errors.

The dataset was manually segmented into episodes. Each of the 48 episodes contains a varying number of instances of seven distinct episode types, labeled as A, B1, B2, C1, C2, D, and E. Collectively, the dataset comprises 356 instances of these episode types.

Utilizing domain knowledge, we can define a library \( L \) to characterize relations between attributes. One straightforward approach to constructing this library is by directly specifying its parsing graph \( \G(L) \). In this regard, we designate the following signs: \emph{Move}, \emph{Objects}, \emph{Path}, \emph{Node}, \emph{Episode}, \emph{Rotate}, \emph{Stalled}, and \emph{Class}. Additionally, we specify components such as \emph{pushing}, \emph{direction}, \( \text{stat}_1 \), \emph{proximity}, \emph{source}, \emph{target}, \emph{start}, \emph{end}, \( \text{direction} \), \( \text{stat}_2 \), and \emph{type}, with their constraints delineated in the graph below.

\begin{figure}[h]
    \begin{center}
    \includegraphics[width=150pt]{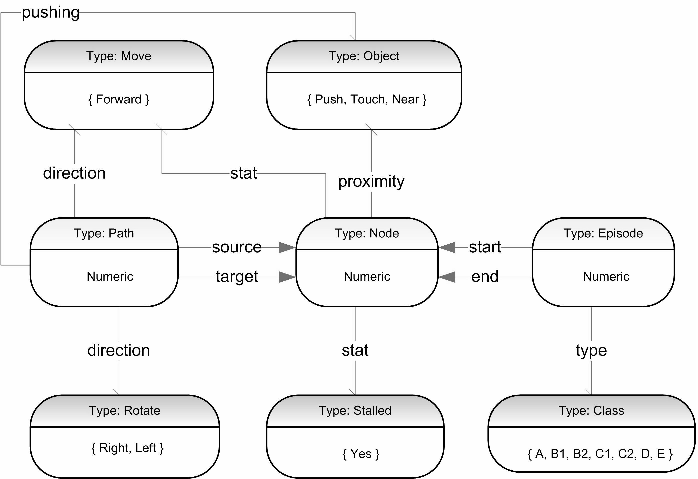}
    \includegraphics[width=150pt]{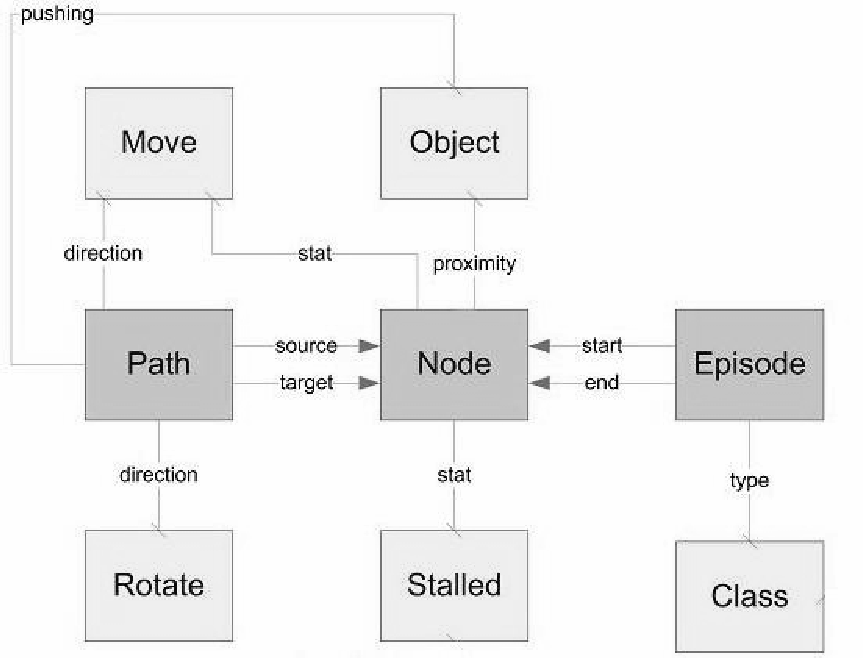}
    \end{center}
    \caption{Parsing Graph and an example query.}\label{parsinggraph}\label{query1}
\end{figure}

Assigning a model to this structure involves fixing an interpretation for each sign and component. The semiotic system defined by this structure is consistent with the available data if there exists a word in the associated graphic language that contains the provided dataset. However, the expressive limitations of the chosen language render this notion of consistency too restrictive for practical purposes. To address this, we introduce the concept of \( \lambda \)-consistency with the data for a semiotic: a specification expressed in the graphic language is \( \lambda \)-consistent with the data if there exists an interpretation that provides a "good approximation" to the given data.

For the parsing graph depicted in Figure \ref{parsinggraph}, the signs' interpretation domains come equipped with a similarity relation. The limit of the diagram in Figure \ref{query1}, where we identify the diagram sources as \( \{ \text{Move}, \text{Object}, \text{Rotate}, \text{Stalled} \} \), the target as \( \{ \text{Class} \} \), and auxiliary signs as \( \{ \text{Path}, \text{Node}, \text{Episode} \} \), serves as a "good" approximation to the dataset. This limit can be represented as an \( \Omega \)-set:
\[ \alpha: \text{Move} \times \text{Object} \times \text{Rotate} \times \text{Stalled} \times \text{Class} \]
The discrepancies between \( \alpha \) and the actual data should be viewed as information that is not semantically valid within the defined semiotic framework. Such limits provide a perspective of the data as described by the semiotic. However, the information encapsulated in this generated limit is not suitable for solving the proposed pattern detection problem using machine learning algorithms. This is because it fails to encode the temporal relations between states in the time series generated by the robot's perceptual system.

We utilized limits of admissible configurations to extract potentially valuable information from the universe modeled by the semiotic. The existence of patterns associated with the robot stalling in the first three states of each episode should be discernible in the limit for diagram \( D(a) \) in Figure \ref{query2}, using the appropriate machine learning tools.

\begin{figure}[h]
    \begin{center}
    \includegraphics[width=150pt]{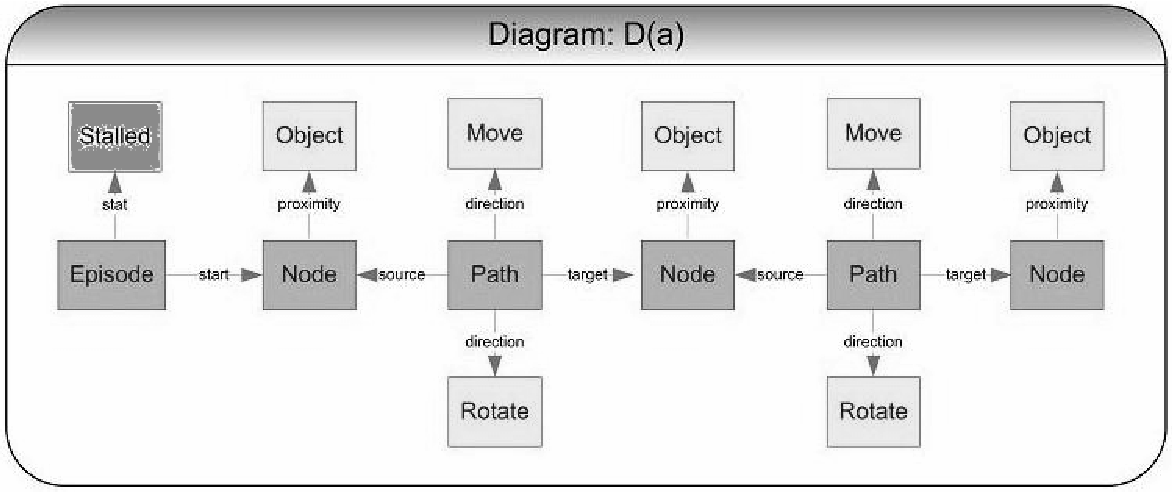}
    \includegraphics[width=170pt]{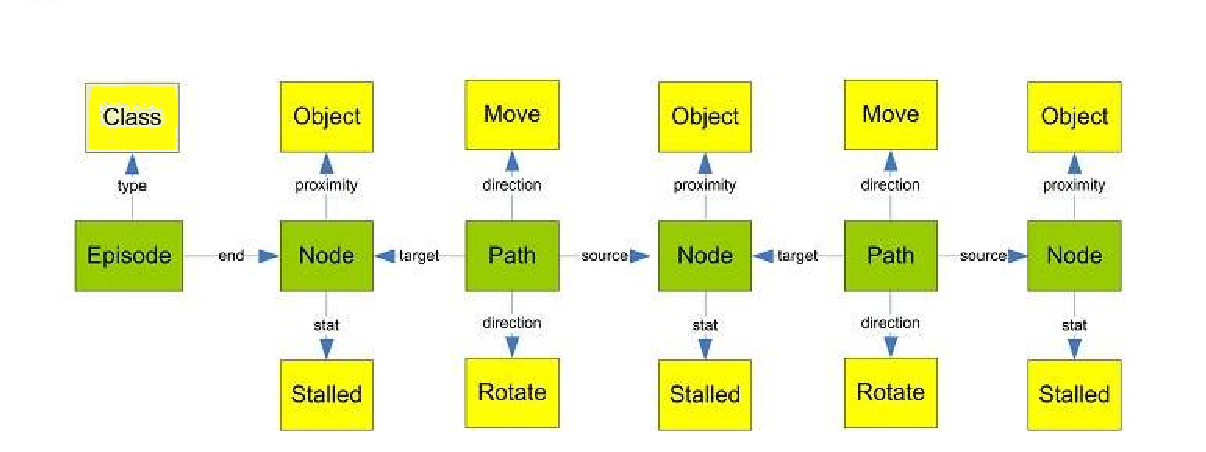}
    \end{center}
    \caption{Queries \( D(a) \) and \( D(b) \).}\label{query2}\label{query3}
\end{figure}

However, for episode classification, the last robot states appear to be more relevant. Information patterns associated with the last three states of a robot are evident in the limit for diagram \( D(b) \) from Figure \ref{query3}.

To utilize relational information available for three consecutive states in episode class prediction, we need to modify the current library. A challenge arises: the existing library structure lacks the expressive power required to describe this type of query comprehensively. Enhancing the library's expressive capacity involves adding an additional component: associating each automaton state with its respective episode. The query can then be defined by diagram \( D(c) \) in Figure \ref{query3}.

Library updates can also be implemented to constrain their interpretations. In Figure \ref{enrichsemiotic}, we enriched the sign system by imposing interpretive restrictions through the addition of two equalizers. These were employed to mandate that in a path, the source and the target must differ. This constraint can be codified by restricting the equalizer between the \emph{source} and \emph{target} components to an initial relation. If the consideration is also limited to episodes with more than one state, the limit of the diagram defined by the \emph{start} and \emph{end} components must be the initial relation. The specification below is further enriched with new components \( \text{Lim}\;D(a) \), \( \text{Lim}\;D(b) \), \( \text{Lim}\;D(c) \), and \( \text{Lim}\;D(d) \) – interpreted as the limits of the presented queries \( D(a) \), \( D(b) \), \( D(c) \), and \( D(d) \), respectively – added as signs interpreted as the source of these new components. These new signs introduce an additional layer to the sign ontology, being more general than the initially defined signs.

\begin{figure}[h]
    \begin{center}
    \includegraphics[width=150pt]{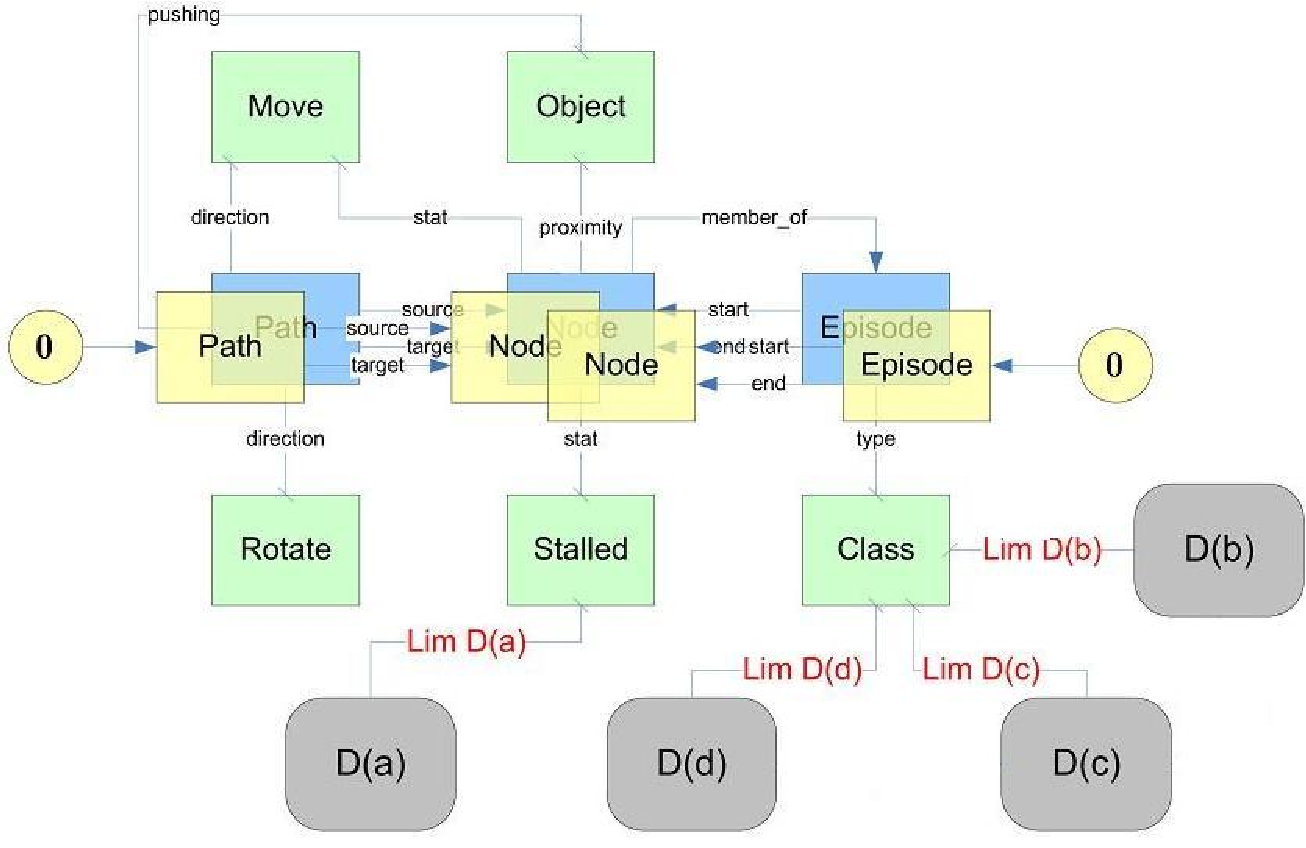}
    \end{center}
    \caption{Enriched sign system.}\label{enrichsemiotic}
\end{figure}

Note that the limit of diagram \( D(c) \), presented in Figure \ref{query3}, describes available information about signs interpreted as three consecutive states, while the limit of \( D(d) \) characterizes three consecutive states of episodes where the robot stalls.

\begin{figure}[h]
    \begin{center}
    \includegraphics[width=150pt]{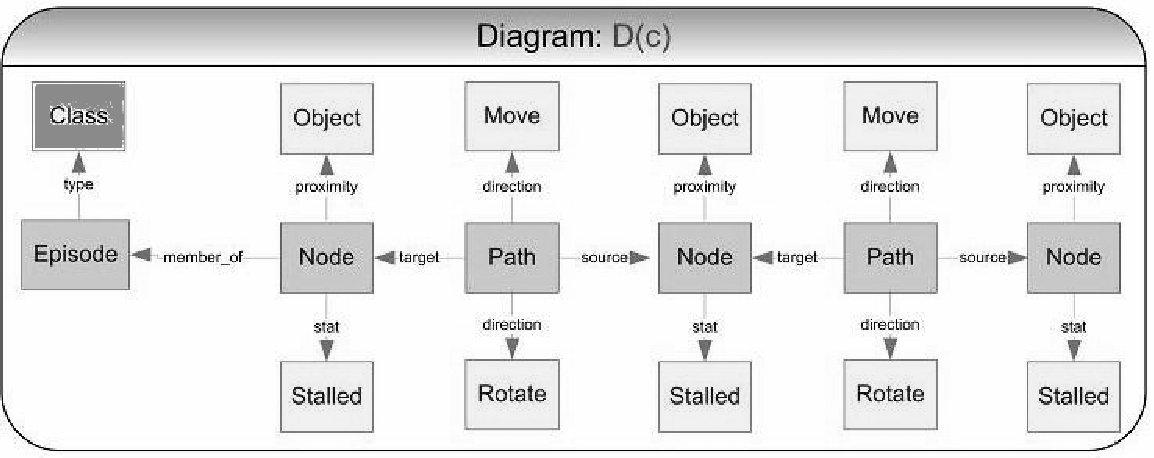}
    \includegraphics[width=170pt]{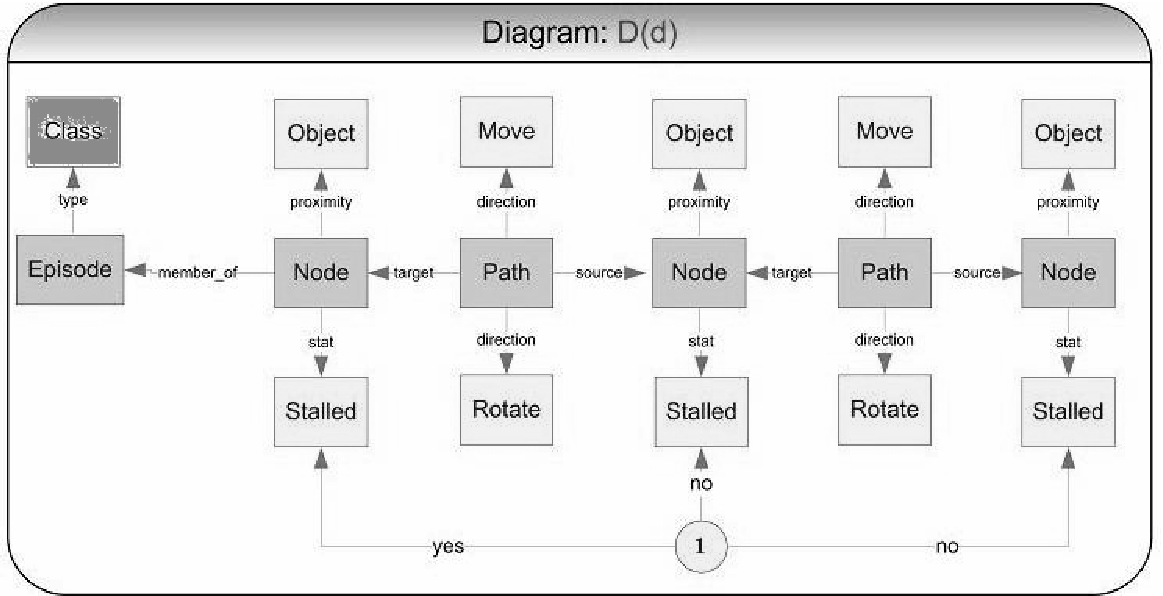}
    \end{center}
    \caption{\( D(c) \) querying three consecutive states and \( D(d) \) querying three consecutive states of episodes where the robot stalls.}\label{query4}
\end{figure}
\end{exam}

In the example provided, we utilized limits as a method to query the structure of a semiotic system. In the following chapters, we will formalize some of the concepts introduced in this example, particularly elucidating what we mean by an "approximation" to the given data.

\begin{exam}[Neural Networks \cite{MichalskiMitchellCarbonell86}]
A neural network is comprised of a network of simple processing elements, known as neurons, which can collectively exhibit complex global behavior. This behavior is determined by the connections between the processing elements and their associated parameters. Formally, a neural network can be represented as a function \( f \), defined as a composition of other functions \( g_i(x) \), which can themselves be further defined as compositions of additional functions. These dependencies within the network can be conveniently represented using a network structure, with arrows denoting the dependencies between variables. One of the most intriguing aspects of neural networks is their capacity to be parameterized for learning specific tasks. Given a particular task and a class of functions \( F \) defined using a network structure and a set of neurons, learning entails utilizing a set of observations to identify \( f^* \in F \) that optimally solves the task.

In this context, an artificial neural network can be viewed as an admissible diagram defined within a semiotic framework, which codifies every potential parametrization of each processing element. More precisely, the associated sign system describes the possible neural network structures, while a model for it represents a set of parameterizations that characterize the behavior of each processing element.

There are three major learning paradigms in neural networks, each corresponding to a distinct learning task \cite{Bishop96}. These paradigms are supervised learning, unsupervised learning, and reinforcement learning. They can be regarded as different approaches to exploring the space defined by models of admissible configurations in order to identify the optimal solution—i.e., the model that best fits the data.

Given a diagram \( D \) in a Neural Network Semiotic, the multi-morphism \( \text{Lim}\;D \) describes the functional behavior of the network \( D \) when applied to its input space. In this context, a network \( D \) has learned a concept described in a dataset if a portion of its limit \( \text{Lim}\;D \) serves as a good approximation to the dataset.

\end{exam}

Additional examples can be drawn from the realm of generic programming applications. For a comprehensive exploration of this topic, one can refer to \cite{LopesFiadeiro97}.

\chapter{Logics}\label{logics}

We can conceptualize a semiotic as a formal mechanism to specify a Universe of Discourse for a given problem. We are particularly interested in the specification of universes where entities can be characterized by propositions in monoidal logics. To achieve this, we introduce the following definition:

\begin{defn}
A semiotic system $(S,M)$ is termed a \emph{logic semiotic system} if it is characterized by the following components:
\[S=(L:|L|\rightarrow (Chains\downarrow \Sigma^+),\E,\U,co\U),\]

The system satisfies the following conditions:
\begin{enumerate}
    \item There exists a sign in $S$, denoted as the support $\Omega$, which is interpreted as an ML-algebra. This algebra has operator interpretations corresponding to the components labeled with the signs $\vee ,\wedge ,\otimes, \Rightarrow ,\bot$, and $\top$.
    
    \item For every string $w=s_0s_1\ldots s_n$ consisting of signs in $S$, there exists a label $=_w$. This label is interpreted by $M$ as the similarity measure $\bigotimes_i^n[\cdot=\cdot]_i$, where $[\cdot=\cdot]_i$ represents the similarity on the $\Omega$-set $M(s_i)$, as defined in \ref{ProdSimil}.
    
    \item For every sign $s$ in $S$ and every natural number $n$, there exists a component in $S$ labeled by $\lhd^n_s$. This component is interpreted by $M$ as a diagonal relation \[\lhd^n_{M(s)}:M(s)\rightharpoonup \prod^n_{i=1}M(s),\] given by \[\lhd^n_{M(s)}(a,a_1,a_2,\ldots,a_n)=\bigotimes_i^n[a=a_i].\]
    
    \item Analogously, for every sign $s$ in $S$ and every natural number $n$, there exists a component labeled by $\rhd^n_s$. This component is interpreted by $M$ as a codiagonal relation \[\rhd^n_{M(s)}:\prod^n_{i=1}M(s)\rightharpoonup M(s),\] and is given by \[\rhd^n_{M(s)}(a_1,a_2,\ldots,a_n,a)=\bigotimes_i^n[a_i=a].\]
    
    \item We assume the existence of a disjoint decomposition for the set of signs $\Sigma$, which is given by $\Sigma_{aux}$ and $\Sigma_{pri}$. Here, signs in $\Sigma_{aux}$ are termed \emph{auxiliary}. For every pair $(s,u)\in \Sigma_{pri}\times \Sigma_{aux}$, there exist components $r(s,u):su^+$ and $r(u,s):us^+$ in $L$, referred to as \emph{rename components}. These components are interpreted by $M$ as the identity in $M(s)$, i.e., \[M(r(s,u))=id_{M(s)} \text{ and } M(r(u,s))=id_{M(s)}.\]
\end{enumerate}
\end{defn}

In a semiotic logic, the signs $\triangleleft$ and $\triangleright$ play a crucial role in connecting similar component inputs and outputs, respectively. These signs serve as essential tools for establishing relationships between the components, ensuring a coherent flow of information within the system. Rename components, on the other hand, serve as a mechanism to codify the connections or links between the inputs and outputs of various components represented in the diagram.

Equations in this context can be elegantly represented using commutative diagrams. As illustrated in Figure \ref{identity}, we specify the property 
\[e_s + x = x + e_s = x\]
which denotes the existence of an identity element \(e_s\). This property is depicted using a commutative diagram to visually capture the relationships and operations involved.

A model \(M\) for an additive library \L\(_A(S,C)\) defines an additive operator with identity if the diagram's limit yields a total multi-morphism. This ensures that the system possesses an additive structure where the identity element interacts with the other components in a consistent and coherent manner.

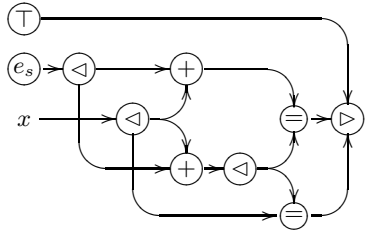
\begin{figure}[h]
\[
\small
\xymatrix @=7pt {
 *+[o][F-]{\top} \ar `r[rrrrrr][rrrrrrdd]&  & & & & & & \\
 *+[o][F-]{e_s}\ar[r]& *+[o][F-]{\lhd} \ar[rr] \ar `d[ddrr][ddrr]& &*+[o][F-]{+} \ar `r[rr][rrd]& & & & \\
 x\ar[rr]      & &*+[o][F-]{\lhd} \ar `r[ur][ur] \ar `r[dr][dr] \ar `d[ddrrr][ddrrr]& & &*+[o][F-]{=}\ar[r] &*+[o][F-]{\rhd}  \\
                & & &*+[o][F-]{+} \ar[r]& *+[o][F-]{\lhd}\ar `r[ur][ur] \ar `r[dr][dr]& & & \\
                & & & & &*+[o][F-]{=} \ar `r[ruu][ruu] & &
}
\]
\caption{Diagrams codifying $e_s+x= x+e_s= x$.}\label{identity}
\end{figure}

The diagram illustrated in Figure \ref{identity} can be represented in a string-based notation by the chain of signs:

\[
\lhd^3_sr(s,x)r(s,x)r(s,x)e_s\lhd^2_sr(s,y)r(x,s)+r(s,z)r(x,s)r(y,s)+\lhd^2_s r(s,w) r(s,w) r(z,s) r(w,s) = r(w,s)r(x,s)=\top\rhd^3_\Omega.
\]

This notation demonstrates that any diagram that defines a word in a graphical language can be codified using string-based notation with the aid of rename components.

\begin{exam}
Consider \( Set(\Omega) \) defined using the ML-algebra \(\Omega=([0,1],\otimes,\Rightarrow,\vee,\wedge,0,1)\), where \(([0,1], \otimes, \Rightarrow)\) serves as a model for product logic, and \(([0,1],\vee,\wedge,0,1)\) represents the standard complete lattice defined in \([0,1]\) by the "less than" relation. The \(\Omega\)-set, defined in \( A = \{0,1,2\} \) and characterized by the similarity relation provided in the table below:

\begin{center}
\small
\begin{tabular}{c|ccc}
  \( [\_\;=\;\_] \) & 0 & 1 & 2 \\
  \hline
  0 & 1.0 & 0.5 & 0.0 \\
  1 & 0.5 & 1 & 0.5 \\
  2 & 0.0 & 0.5 & 1.0 \\
\end{tabular}
\end{center}

\noindent and the additive relational operators depicted in the following tables:

\begin{center}
\small
\begin{tabular}{c|ccc}
  \( \_\;+\;0 \) & 0 & 1 & 2 \\
  \hline
  0 & 1.0 & 0.5 & 0.0 \\
  1 & 0.5 & 1 & 0.5 \\
  2 & 0.0 & 0.5 & 1.0 \\
\end{tabular}
\qquad
\begin{tabular}{c|ccc}
  \( \_\;+\;1 \) & 0 & 1 & 2 \\
  \hline
  0 & 0.5 & 1.0 & 0.5 \\
  1 & 0.0 & 0.5 & 1.0 \\
  2 & 1.0 & 0.5 & 0.5 \\
\end{tabular}
\qquad
\begin{tabular}{c|ccc}
  \( \_\;+\;2 \) & 0 & 1 & 2 \\
  \hline
  0 & 0.0 & 0.5 & 1.0 \\
  1 & 1.0 & 0.0 & 0.5 \\
  2 & 0.5 & 1.0 & 0.0 \\
\end{tabular}
\end{center}

\noindent provide a model for an additive library. The identity element \( e_s = 0 \) since the diagram in Figure \ref{identity} yields the following limit:

\begin{center}
\small
\begin{tabular}{c|c}
  A & \( \Omega \) \\
  \hline
  0 & 1.0=[0] \\
  1 & 1.0=[1] \\
  2 & 1.0=[2] \\
\end{tabular}.
\end{center}
\end{exam}

In a logic semiotic system \((S,M)\), the language \(Lang(L)\) is referred to as a \emph{logic language}. Logic semiotics possesses sufficient expressive power to differentiate between diagrams that define relations and those that define entities. To achieve this, we classify words as follows:

\begin{defn}(Graphic relations)
A diagram \(D \in Lang(S)\), for a logic semiotic \((S,M)\), is termed a \emph{relation} when its output \(o(D)\) is interpreted by \(M\) as the set of truth values \(\Omega\) on \(Set(\Omega)\).
\end{defn}

\begin{defn}
A relation \(D\) is designated as an \emph{equation} if the diagram \(D\) can be decomposed as
\[
D = I \otimes D_0 \otimes D_1 \otimes \lq=\lq,
\]
where \(I = \lhd^{n_1}_{s_1} \ldots \lhd^{n_m}_{s_m}\) is defined through realizations of diagonals. The diagrams \(D_0\) and \(D_1\) are diagrams that satisfy \(o(D_0) = o(D_1)\), and '=' is a diagram defined using the unique component, interpreted as a similarity relation, and satisfying \(i(\lq=\lq) = o(D_0) \cdot o(D_1)\). In this case, we simplify the notation by denoting the diagram \(D\) as \(D_0 = D_1\). It is noteworthy that the diagram \(I\) codifies the dependencies between interpretations of signs used as inputs for the diagram \(D_0 \otimes D_1 \otimes \lq=\lq\), linking together signs with similar values.
\end{defn}

In the given definition, the diagram \(I = \lhd^{n_1}_{s_1} \otimes \ldots \otimes \lhd^{n_m}_{s_m}\) encapsulates, within a graphic logic, the dependence relations defined in string-based logic by repeatedly utilizing bounded variables in a proposition.

\begin{figure}[h]
\[
\begin{picture}(65,65)(0,0)
\multiput(19,33.75)(.04435484,-.03360215){186}{\line(1,0){.04435484}}
\put(27.25,27.5){\line(0,1){41}}
\multiput(27.25,68.5)(-.05536913,-.03355705){149}{\line(-1,0){.05536913}}
\multiput(55.5,29.5)(.034911717,.033707865){623}{\line(1,0){.034911717}}
\multiput(77.25,50.5)(-.035773026,.033717105){608}{\line(-1,0){.035773026}}
\put(25.25,63.75){\line(0,-1){13}}
\multiput(25.25,50.75)(.16666667,.03358209){201}{\line(1,0){.16666667}}
\multiput(58.75,57.5)(-.17487047,.03367876){193}{\line(-1,0){.17487047}}
\multiput(25.25,42.75)(.03125,-1.3125){8}{\line(0,-1){1.3125}}
\multiput(25.5,32.25)(.21644295,.03355705){149}{\line(1,0){.21644295}}
\put(57.25,57.75){\circle*{2.06}} \put(73.25,50){\circle*{2.92}}
\put(25.75,62.75){\circle*{2.06}} \put(25.25,59.75){\circle*{2.06}}
\put(25.5,57){\circle*{2.06}} \put(25.25,54){\circle*{2.24}}
\put(25.75,40.75){\circle*{2}} \put(25.75,37.25){\circle*{1.8}}
\put(25.75,34.75){\circle*{2}} \put(20,61.25){\circle*{2.06}}
\put(19.5,58.25){\circle*{2.24}} \put(19.5,52){\circle*{2}}
\put(19.5,46.5){\circle*{2}} \put(19.5,42){\circle*{2}}
\put(23.25,47){\makebox(0,0)[cc]{$_I$}}
\put(33.75,56.25){\makebox(0,0)[cc]{$_{D_1}$}}
\put(33,37.25){\makebox(0,0)[cc]{$_{D_2}$}}
\put(65.25,49.75){\makebox(0,0)[cc]{$_=$}}
\put(76.5,44.75){\makebox(0,0)[cc]{$_{\Omega}$}}
\put(55.5,70.75){\line(0,-1){40.75}}
\multiput(25.25,43)(.19512195,-.03353659){164}{\line(1,0){.19512195}}
\put(19,63.5){\line(0,-1){29}} \put(19.5,38){\circle*{1.58}}
\put(56.75,38.25){\circle*{1.58}}
\end{picture}
\]
\caption{Sketch for a multi-morphism of type $D=I\otimes D_0\otimes D_1\otimes \lq=\lq$.}\label{identity2}
\end{figure}
The relation \(D \in Lang(S)\) is termed \emph{true} by \(M\) if the limit \(M(D \otimes \top \otimes \rhd^2_\Omega)\) is a total multi-morphism. In this context, an equation \(D\) is considered universal if the interpretation of \(D.\top.\rhd^2_\Omega\) by \(M\) is a total multi-morphism.

Given a logic semiotic system \((S,M)\), let \(Lang_R(S,M)\) be the subcategory of \(Lang(S)\) consisting of diagrams defining relations. Utilizing the operation of diagram gluing, we define the following operators for pairs of relations \(D_0, D_1 \in Lang_R(S,M)\):

\begin{enumerate}
  \item \(D_0 \otimes D_1\) corresponds to the diagram \(I \otimes D_0 \otimes D_1 \otimes \lq\otimes\lq\),
  \item \(D_0 \Rightarrow D_1\) corresponds to the diagram \(I \otimes D_0 \otimes D_1 \otimes \lq\Rightarrow\lq\),
  \item \(D_0 \wedge D_1\) corresponds to the diagram \(I \otimes D_0 \otimes D_1 \otimes \lq\wedge\lq\),
  \item \(D_0 \vee D_1\) corresponds to the diagram \(I \otimes D_0 \otimes D_1 \otimes \lq\vee\lq\),
\end{enumerate}

where \(I\) is defined through realizations of diagonals, linking together inputs that have the same meaning according to \(M\). This framework enables the definition of new relations from pairs of simpler ones, representing the lifting of the \(\Omega\) structure to diagrams in \(Lang_R(S,M)\).

In the subsequent chapters, we present some illustrative examples of logic semiotics that are pivotal for characterizing the expressive power of languages utilized by certain machine learning algorithms:

\begin{exam}[Binary semiotic \(S_B(S,C)\)]
A binary semiotic is a logic semiotic where sets \(S\) and \(C\) define a binary library \L\(_B(S,C)\) (as presented in example \ref{binarylibrary}). We refer to this type of semiotic as a dataset semiotic or table semiotic, as instantiations of relations in \(Lang_R(S_B(S,C))\) can be used to encode datasets or tables.

Binary semiotics find applications in machine learning algorithms designed to generate decision rules, such as Apriori, as described in \cite{MichalskiMitchellCarbonell86}.
\end{exam}

\begin{exam}[Linear semiotic $S_L(S,C)$]
A linear semiotic \(S_L(S,C)\) serves as an extension of a binary semiotic. It is characterized by a linear library \L\(_L(S,C)\) (as illustrated in example \ref{linearlibrary}). Within this framework, \(\geq:ssl^+\) is interpreted as a partial order in the interpretation of a sign \(s\), denoted as \(M(s)\), for all symbols \(s \in S\). This type of relation is encapsulated in the model of a linear semiotic \(S_L(S,C) = (L, \E, \U, co\U)\) if it encompasses the diagrams depicted in fig. \ref{linear}. These diagrams codify propositions represented in string-based first-order logic as follows:

\[
\forall x : x \geq x,
\]
\[
\forall x, y : (x \geq y \wedge y \geq x) \Rightarrow (x = y),
\]
\[
\forall x, y, z : (x \geq y \wedge y \geq z) \Rightarrow (x \geq z),
\]

and these propositions are "preserved" by its models.

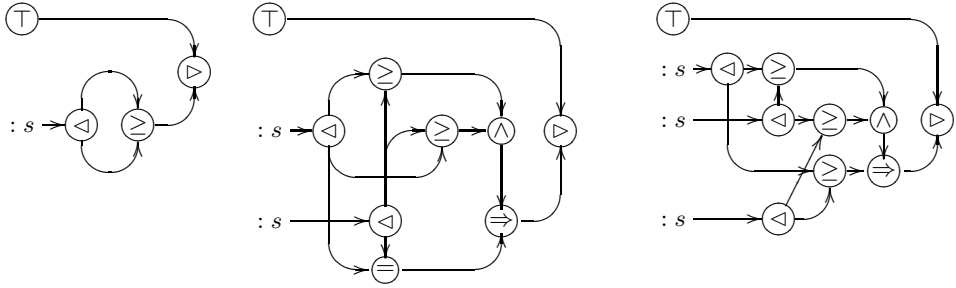
\begin{figure}[h]
\[
\small
\xymatrix @=8pt {
 *+[o][F-]{\top}\ar `r[rrrd][rrrd]&&&\\
 &&&*+[o][F-]{\rhd}&\\
 :s\ar[r]&*+[o][F-]{\lhd}\ar `u[ru]`r[rd][r] \ar `d[rd]`r[ru][r]&*+[o][F-]{\geq}\ar `r[ru][ru]&&\\
 &&&&
}
\xymatrix @=8pt{
 *+[o][F-]{\top}\ar `r[rrrrrdd][rrrrrdd]&&&&&&&\\
&&*+[o][F-]{\geq}\ar `r[rrd][rrd] &&&&&\\
:s\ar[r]&*+[o][F-]{\lhd}\ar `d[dddr][dddr] \ar `u[ur] [ur] \ar `d[drr] `r[rru] [rr]& &*+[o][F-]{\geq} \ar [r]&*+[o][F-]{\wedge} \ar [dd]&*+[o][F-]{\rhd}&&\\
&&&&&&&\\
:s\ar[rr]&&*+[o][F-]{\lhd}\ar[d] \ar[uuu] \ar `u[uur] [uur] &&*+[o][F-]{\Rightarrow} \ar `r[ruu][ruu]&&&\\
&&*+[o][F-]{=}\ar `r[rru] [rru]&&&&&\\
}
\xymatrix @=6pt{
 *+[o][F-]{\top}\ar `r[rrrrrdd][rrrrrdd]&&&&&&&\\
:s\ar[r]&*+[o][F-]{\lhd}\ar[r] \ar `d[ddrr][ddrr] &*+[o][F-]{\geq}\ar `r[rrd][rrd] &&&&&\\
:s\ar[rr]&&*+[o][F-]{\lhd}\ar[u]\ar[r] &*+[o][F-]{\geq} \ar [r]&*+[o][F-]{\wedge} \ar [d]&*+[o][F-]{\rhd}&&\\
&&&*+[o][F-]{\geq}\ar[r]&*+[o][F-]{\Rightarrow} \ar `r[ru][ru]&&&\\
:s\ar[rr]&&*+[o][F-]{\lhd} \ar[uur] \ar `r[ru] [ru]&&&&&&\\
}
\]
\caption{Diagrams codifying $\forall x:x\geq
x,\,\,\,\, \forall x,y:(x\geq y\wedge y\geq x) \Rightarrow (x=y)$ and $\forall
x,y,z:(x\geq y\wedge y\geq z) \Rightarrow (x\geq z)$.}\label{linear}
\end{figure}

 Let's consider an example on \( Set(\Omega) \) where \( \Omega \) has the structure of an ML-algebra. Here, \( ([0,1],\otimes,\Rightarrow) \) serves as a model for product logic, and \( ([0,1],\vee,\wedge,0,1) \) represents the typical complete lattice defined in \( [0,1] \). 

For the \( \Omega \)-set defined with the support \( A=\{0,1,2,3,4\} \) and the similarity relation:
\begin{center}
\small
\begin{tabular}{c|ccccc}
  \( [\_\;=\;\_] \) & 0 & 1 & 2 & 3 & 4 \\
  \hline
  0 & 1.0 & 0.5 & .25 & 0.0 & 0.0 \\
  1 & 0.5 & 1.0 & 0.5 & .25 & 0.0 \\
  2 & .25 & 0.5 & 1.0 & 0.5 & .25 \\
  3 & 0.0 & .25 & 0.5 & 1.0 & 0.5 \\
  4 & 0.0 & 0.0 & .25 & 0.5 & 1.0 \\
\end{tabular}
\end{center}
and the relational operator defined from \( A \) to \( A \) as:
\begin{center}
\small
\begin{tabular}{c|ccccc}
  \( [\_\;\geq\;\_] \) & 0 & 1 & 2 & 3 & 4 \\
  \hline
  0 & 1.0 & 0.5 & .25 & 0.0 & 0.0 \\
  1 & 1.0 & 1.0 & 0.5 & .25 & 0.0 \\
  2 & 1.0 & 1.0 & 1.0 & 0.5 & .25 \\
  3 & 1.0 & 1.0 & 1.0 & 1.0 & 0.5 \\
  4 & 1.0 & 1.0 & 1.0 & 1.0 & 1.0 \\
\end{tabular}
\end{center}

When these relations are employed for sign interpretation in the three diagrams depicted in fig. \ref{linear}, they yield the following limits, respectively:
\[
\small
\begin{tabular}{c|c}
  \( A \) & \( \Omega \) \\
  \hline
  0 & 1 \\
  1 & 1 \\
  2 & 1 \\
  3 & 1 \\
  4 & 1 \\
\end{tabular}\;\;
\begin{tabular}{c|c}
  \( A\times A \) & \( \Omega \) \\
  \hline
  (0,0) & 1 \\
  (1,0) & 1 \\
  \(\vdots\) & \(\vdots\) \\
  (3,4) & 1 \\
  (4,4) & 1 \\
\end{tabular}\;\;\text{ and }
\begin{tabular}{c|c}
  \( A\times A\times A \) & \( \Omega \) \\
  \hline
  (0,0,0) & 1 \\
  (1,0,0) & 1 \\
  \(\vdots\) & \(\vdots\) \\
  (3,4,4) & 1 \\
  (4,4,4) & 1 \\
\end{tabular}
.
\]
This confirms the commutativity of each diagram.

We refer to this type of semiotics as \textbf{grid semiotics}. They often appear in associated machine learning algorithms used to generate decision rules, such as the C4.5Rules by J.R. Quinlan, as described in \cite{MichalskiMitchellCarbonell86}.

\end{exam}

\begin{exam}[Additive semiotic $S_A(S,C)$]
An \emph{additive semiotic} \( S_A(S,C) \) is a linear semiotic \( S_L(S,C) \) such that it possesses an additive library \L\(_A(S,C) \) (as illustrated in example \ref{addlibrary}). In this context, \( (M(s), M(+:sss^+)) \) forms a monoid for every symbol \( s \in S \). Furthermore, the interpretation for \( e_s: s^+ \) corresponds to the monoid identity, denoted by \( e_s \in C \). The properties of a monoid can be effectively encoded within the semiotic structure if the model transforms each of the diagrams depicted in fig. \ref{additive} into a total multi-morphism for every sign \( s \) in the semiotic.

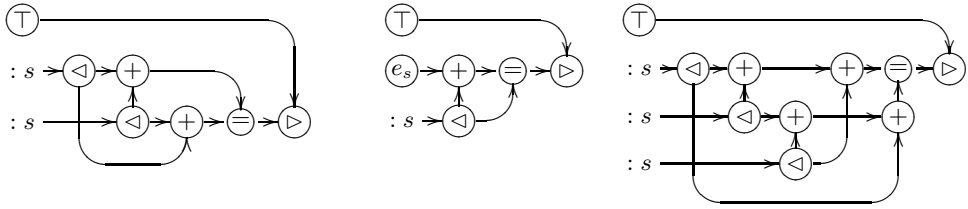
\begin{figure}[h]
\[
\small
\xymatrix @=7pt {
 *+[o][F-]{\top}\ar `r[rrrrrdd]'[rrrrrdd]&&&&&&&\\
:s\ar[r]&*+[o][F-]{\lhd}\ar[r] \ar `d[ddrr]`r[rru][drr] &*+[o][F-]{+}\ar `r[rrd][rrd] &&&&&\\
:s\ar[rr]&&*+[o][F-]{\lhd}\ar[u]\ar[r] &*+[o][F-]{+} \ar [r]&*+[o][F-]{=} \ar [r]&*+[o][F-]{\rhd}&&\\
&&&&&&&\\
}  \xymatrix @=7pt {
 *+[o][F-]{\top}\ar `r[rrrd]'[rrrd]&&&&\\
*+[o][F-]{e_s}\ar[r]&*+[o][F-]{+}\ar[r]&*+[o][F-]{=}\ar[r]&*+[o][F-]{\rhd}&\\
:s\ar[r]&*+[o][F-]{\lhd}\ar [u]\ar `r[ru][ru]&&&\\}
\xymatrix @=6pt {
 *+[o][F-]{\top}\ar `r[rrrrrrd][rrrrrrd]&&&&&&&\\
:s\ar[r]&*+[o][F-]{\lhd}\ar[r] \ar `d[dddrr]`r[rrrru][drrrr] &*+[o][F-]{+}\ar [rr] &&*+[o][F-]{+}\ar[r]&*+[o][F-]{=}\ar[r]&*+[o][F-]{\rhd}&\\
:s\ar[rr]&&*+[o][F-]{\lhd}\ar[u]\ar[r] &*+[o][F-]{+} \ar [rr]& &*+[o][F-]{+}\ar[u]&&&\\
:s\ar[rrr]&&&*+[o][F-]{\lhd}\ar[u]\ar`r[ru][ruu] &&&&\\
&&&&&&&\\
}
\]
\caption{Diagrams codifying $\forall x,y:x+y=y+x, \forall x:x+e_s=x$ and $\forall
x,y,z:(x+y)+z=x+(y+z)$.}\label{additive}
\end{figure}
\end{exam}

\begin{exam}[Multiplicative Semiotic \( S_M(S,C) \)]
An \emph{multiplicative semiotic} \( S_M(S,C) \) is an additive semiotic
\( S_L(S,C) \) defined by a multiplicative library \L\(_M(S,C) \) (as presented in example \ref{multlibary}). In this context, \( (M(s), \times) \) forms a commutative semigroup for every symbol \( s \in S \).

Within this semiotic framework, we can encode rules established through regression, akin to the rules generated by machine learning algorithms such as M5 by J.R. Quinlan, as described in \cite{Quinlan93}.
\end{exam}

The definition of a Domain of Discourse structure may impose constraints on the interpretation of signs. In a binary semiotic, we might impose sign interpretation rules defined by Horn clauses of the form:
\[
(x_1 = c_1 \wedge x_2 = c_2 \wedge x_3 = c_3) \Rightarrow y = c_4.
\]
The expressive power of a linear semiotic enables the encoding of rules like:
\[
(x_1 \leq c_1 \wedge x_2 \leq c_2 \wedge c_3 \leq x_1) \Rightarrow y \leq c_4,
\]
While in a multiplicative semiotic, sign interpretation can be confined to semiotics defined by models that satisfy regression rules such as:
\[
(x_1 \leq c_1 \wedge x_2 \leq (c_3 \times x_1 + c_2) \wedge x_3 = c_3) \Rightarrow
y \leq c_4 \times x_1 + c_5 \times x_2 + c_6 \times x_3 + c_7.
\]
This type of regression rules can be represented through diagrams, similar to the one depicted in fig. \ref{regrassion}, where frames highlight the evident subdiagrams.

\begin{figure}[h]
\[
\small
\xymatrix @=12pt {
*+[o][F-]{\top}\ar `r[rrrrrrd][rrrrrrd]&&&&\\
x_1:s\ar[r]&*+[o][F-]{\lhd}\ar[rr] \ar `d[r]`/1pt[rrd][drr] \ar `d[dddrrrrr][dddrrrrr]&&*+[F-]{_{x_1\leq c_1}}\ar `r[rrd][rrd]&&&*+[o][F-]{\rhd}&\\
x_2:s\ar[rr]&&*+[o][F-]{\lhd}\ar[r]\ar `d[ddrrrr][ddrrrr]&*+[F-]{_{c_3x_1+c_2\leq x_2}}\ar[rr]&&*+[o][F-]{\wedge}\ar[d]&\\
x_3:s\ar[rrr]&&&*+[o][F-]{\lhd}\ar[r]\ar `d[drrr][drrr]&*+[F-]{_{x_3=c_3}}\ar[r]&*+[o][F-]{\wedge}\ar[r]&*+[o][F-]{\Rightarrow}\ar[uu]\\
y:s\ar[rrrrrr]&&&&&&*+[F-]{_{y=c_4\times x_1+c_5\times x_2+c_6\times
x_3+c_7}}\ar[u]&}
\]
\caption{Diagram codifying a regression rule.}\label{regrassion}
\end{figure}
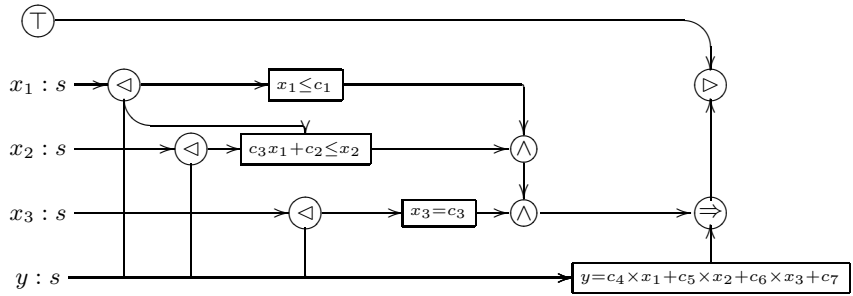

Models generated by various Data Mining and Machine Learning tools can be encoded within these sign systems. This implies that rules extracted from models produced by learning algorithms can reflect the underlying data structure. Such structures can enhance the richness of the sign system used to specify the information system. This enhancement allows for the definition of constraints on its models that are compatible with the existing data or perspectives on reality. Essentially, this suggests that the most accurate description we can have for an information system is the best generalization derived from the stored data.

\chapter{Lagrangian syntactic operators}\label{synopt}
The expressive power of our specification languages can be increased using \emph{Lagrangian syntactic operators} or \emph{sign operators}. This operator are defined at the level of sign systems signs or/and components, and must be preserved via sign systems models, allowing the generation of new signs or components based on others signs or components.

An example of this operators, with evident applicability, are the differential operators. For that we define:

\begin{defn}[Differential semiotic]
A logic semiotic system $(S,M)$ is termed a \emph{differential semiotic system} if it is a multiplicative semiotic that contains a sign $R$, which is interpreted as the support for a ring $(M(R),\times,+,0,1)$. Additionally, the system includes labels $\partial_wf$ in $S$, corresponding to components $f:i(f)\rightarrow R$ in $S$ with an output of $R$, and where $w$ is a word composed of its input symbols, subject to the following conditions:
\begin{enumerate}
    \item $w\leq i(f)$,
    \item $\partial_wf:d(f)\rightarrow R$,
    \item For the component label definitions presented below, utilizing $'\times'$ and $'+'$ in infix notation, the following conditions must be met:
    \begin{enumerate}
        \item If $f\equiv_l d_0\times d_1$, then $M(\partial_w(d_0\times d_1))=M(\partial_w(d_0)\times d_1+d_0\times \partial_w(d_1))$,
        \item If $f\equiv_l d_0+d_1$, then $M(\partial_w(d_0+d_1))=M(\partial_w(d_0)+\partial_w(d_1))$, and
        \item If $ww'\leq i(f)$, then $M(\partial_{ww'}f)=M(\partial_{w}(\partial_{w'}f))=M(\partial_{w'}(\partial_{w}f))$.
    \end{enumerate}
\end{enumerate}
\end{defn}

The component label operator $\partial$ facilitates the characterization of multi-morphisms that are otherwise impossible to represent within the associated logic semiotic. For instance, in a differential semiotic system, we can interpret a component $f:xy\rightarrow R$ as a multi-morphism that adheres to the conservative law when the following diagram is interpreted by the model as a total multi-morphism.

\begin{figure}[h]
\[
\small
\xymatrix @=7pt {
 *+[o][F-]{\top}\ar `r[rrrrd][rrrrd]&&&&\\
 &&*+[o][F-]{\partial_xf}\ar `r[rd][rd]&&*+[o][F-]{\rhd}&\\
 :xy\ar[r]&*+[o][F-]{\lhd}\ar `u[ru][ru] \ar `d[rd][rd]&&*+[o][F-]{=}\ar `r[ru][ru]&&\\
 &&*+[o][F-]{\partial_yf}\ar `r[ru][ru]&&&
}
\]
\caption{Conservative law in a diferencial semiotic.}\label{conservative}
\end{figure}
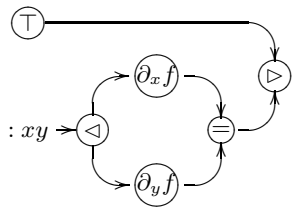

Let's consider a naive application:

\begin{exam}[Modeling traffic]
In the Lighthill-Whitham-Richards (LWR) model, as presented in \cite{WhithamLighthill55} and \cite{Richards56}, the traffic state is represented from a macroscopic viewpoint by the function $\rho(x,t)$, which denotes the density of vehicles at position $x$ and time $t$. The dynamics of the traffic are encapsulated by a conservation law expressed as
\[
\frac{\partial \rho}{\partial t} + \frac{\partial (\rho v)}{\partial x} = 0,
\]
where $v = v(x,t)$ represents the velocity of cars at $(x,t)$. The fundamental assumption of the LWR model posits that drivers adjust their speed instantaneously based on the surrounding density:
\[
v(x,t) = V(\rho(x,t)),
\]
where the function $f(\rho) = \rho V(\rho)$ serves as the "flow rate," denoting the number of vehicles per unit time. Consequently, we have
\[
\frac{\partial \rho}{\partial t} + \frac{\partial f(\rho)}{\partial x} = 0.
\]
The model is formulated for a single unidirectional road, defining a differential semiotic system with the foundational library depicted in Figure \ref{baselibrary}.
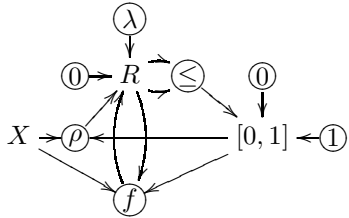
\begin{figure}[h]
\[
\xymatrix @=8pt {
\tiny
 &&*+[o][F-]{\lambda}\ar[d]&&&\\
 &*+[o][F-]{0}\ar[r]&R \ar@/^/[r] \ar@/_/[r]\ar@/^/[dd]&*+[o][F-]{\leq}\ar[dr]&*+[o][F-]{0}\ar[d]&\\
 X\ar[r]\ar[rrd]&*+[o][F-]{\rho}\ar[ru]&&&[0,1]\ar[lll]\ar[lld]&*+[o][F-]{1}\ar[l]\\
 &&*+[o][F-]{f}\ar@/^/[uu]&&&\\
 }
\]
\caption{Library for a single unidirectional road model.}\label{baselibrary}
\end{figure}
The associated sign system is characterized by the total diagrams
\[
\partial_t\rho + \partial_xf(\rho) = 0 \quad \text{and} \quad 0 \leq \rho \leq \lambda,
\]
where the latter condition defines the maximum density of the road. A model for this sign system can be conceptualized as an admissible distribution of cars on the road.

The initial library lacks the descriptive power required to specify a road network comprehensively. In Figure \ref{networklibrary}, we introduce an extension to the initial library that enables the graphical modeling of a single network.
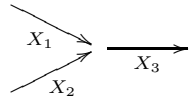
\begin{figure}[h]
\[
\xymatrix @-1pc {
\ar[rrd]_{X_1}&&&&\\
  &&\ar[rr]_{X_3}&&\\
\ar[rru]_{X_2}&&&&
}
\]
\caption{A road network defined by the junction of two incoming roads, $X_1$ and $X_2$, and one outgoing road, $X_3$, all with a single direction.}\label{network}
\end{figure}
This network is characterized by the junction of two incoming roads, $X_1$ and $X_2$, and one outgoing road, $X_3$, all with a single direction. The semiotic for this problem is defined using the extended library depicted in Figure \ref{networklibrary}.
\begin{figure}[h]
\[
\small
\xymatrix @=8pt {
 &*+[o][F-]{\lambda_1}\ar[rd]&*+[o][F-]{\lambda_2}\ar[d]&*+[o][F-]{\lambda_3}\ar[ld]&&\\
 &*+[o][F-]{0}\ar[r]&R \ar@/^/[r] \ar@/_/[r]&*+[o][F-]{\leq}\ar[dr]&*+[o][F-]{0}\ar[d]&\\
 X_1\ar[r]\ar[rdd]&*+[o][F-]{\rho_1}\ar[ru]&&&[0,1]\ar[lll]\ar[dd]\ar[ddl]&*+[o][F-]{1}\ar[l]\\
 &&&&&\\
 &*+[o][F-]{c}\ar[rrruu]\ar[rrrd]&&*+[o][F-]{\rho_2}\ar[luuu]&*+[o][F-]{\rho_3}\ar[lluuu]&\\
 X_2\ar[rrru]\ar[ru]&&&&X_3\ar[u]&
}
\]
\caption{Extended library for the presented road network.}\label{networklibrary}
\end{figure}
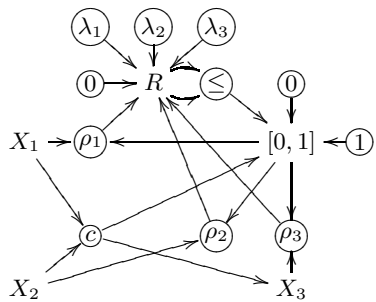
In this extended library, we have also included three "flow rate" components: 
\[
f_1 : X_1 \times R \times [0,1] \rightharpoonup R, \quad f_2 : X_2 \times R \times [0,1] \rightharpoonup R, \quad \text{and} \quad f_3 : X_3 \times R \times [0,1] \rightharpoonup R,
\]
one for each road. Additionally, we introduce initial condition constants $\rho_{1,0}, \rho_{2,0}, \rho_{3,0} : \bot \rightharpoonup R$. The restrictions of the network traffic model are described by the following conservation laws and flow restrictions for each road:

\begin{enumerate}
  \item $\partial_t\rho_1 + \partial_xf_1(\rho_1)=0$ and $0\leq \rho_1 \leq\lambda_1$,
  \item $\partial_t\rho_2 + \partial_xf_2(\rho_2)=0$ and $0\leq \rho_2 \leq\lambda_2$, and
  \item $\partial_t\rho_3 + \partial_xf_3(\rho_3)=0$ and $0\leq \rho_3 \leq\lambda_3$.
\end{enumerate}
To complete the model description, we define the mechanisms that occur at the junction. A fundamental condition is the conservation of flows:
\[
f_1(\rho_1(0,t))+f_2(\rho_2(0,t))=f_3(\rho_3(0,t)), \quad \forall t.
\]
One elementary problem we can study, which also forms the basis for a specific model, is the Riemann problem. For a Riemann problem at a junction, we consider a constant density on all three roads as the initial condition:
\begin{enumerate}
  \item $\rho_1(x,0)=\rho_{1,0}$,
  \item $\rho_2(x,0)=\rho_{2,0}$, and
  \item $\rho_3(x,0)=\rho_{3,0}$.
\end{enumerate}

Models derived from this sign system may be inherently unrealistic. For example, constants such as $\rho_1 = \lambda_1$, $\rho_2 = \lambda_2$, and $\rho_3 = 0$ can be associated with a conceivable semiotic representation. However, this model is clearly counterintuitive in most scenarios (except, of course, in the presence of a red light at the entrance of the third road). A natural approach to constructing a realistic model for a specific road junction involves incorporating rules that describe driver behavior at the junction. We can enrich the sign system by incorporating additional information obtained from models generated by machine learning algorithms for the data. This integration necessitates combining the defined semiotic with a semiotic associated with the language used to describe machine learning models, a challenge we address in the following chapter.

\end{exam}

\chapter{Concept description}\label{Concept description}
A \emph{concept description}, using attributes $(A_i)$, in the logic semiotic system $(S,M)$ is represented by an $\Omega$-map:
\[d:\prod_i A_i\rightarrow\Omega,\]
where the family $(A_i)$ forms a sequence of $\Omega$-sets. A concept description is deemed to be defined in a semiotic $(S,M)$ if the sequence $(A_i)$ is constructed using interpretations of signs within the language $Lang(S)$. For brevity, we denote $d\in M(S)$ when we aim to select a concept description within the semiotic $(S,M)$. It's important to note that $d$ defines a relation in a monoidal logic, which may not necessarily be interpretable by $M$ as a relation in $Lang(S)$. Intuitively, a concept description can be visualized as the state of knowledge about a concept at a given moment. A database specified by the sign system $S$ can encapsulate the concept $d$ if there exists a model $M\in Mod(S)$ and a diagram $D\in Lang(S)$ such that
\[M(D)=d.\]
In this scenario, we assert that the query $D$ on the information system defined by semiotic $(S,M)$ yields the answer $d$.

Given two concept descriptions \[d_0:\prod_i A_i\rightarrow\Omega \quad \text{and} \quad d_1:\prod_i A_i\rightarrow\Omega,\] we denote $d_0\leq d_1$ if $d_0(\bar{x})\leq d_1(\bar{x})$ for every $\bar{x}$ in $\Pi A_i$. It is noteworthy that every $\Omega$-set $A$ defines a concept description via the extent map $[\cdot]:A\rightarrow \Omega$. In this context, we regard every $\Omega$-set as a concept description. Within $Set(\Omega)$, every $\Omega$-set with support $\Pi A_i$ is associated with a complete lattice of concept descriptions, having a bottom element $\bot:\Pi A_i\rightarrow \Omega$ and a top element $\top:\Pi A_i\rightarrow \Omega$. Notably, the limit $M(D)$ serves as a concept description for every $D\in Lang(S)$.

While some concept descriptions can be codified semantically, others cannot. Given a pair of concept descriptions \[d_0:\Pi A_i\rightarrow\Omega \quad \text{and} \quad d_1:\Pi A_i\rightarrow\Omega,\] we define
\[\Gamma(d_0,d_1)=d_0\Leftrightarrow d_1.\]
The biimplication $\Gamma$ is a $\otimes$-\emph{similarity relation} in $\Pi A_i$ due to the following properties:
\begin{enumerate}
  \item $\Gamma(d_0,d_0)=\top$ (reflexivity),
  \item $\Gamma(d_0,d_1)=\Gamma(d_1,d_0)$ (symmetry), and
  \item $\Gamma(d_0,d_1)\otimes\Gamma(d_1,d_2)=\Gamma(d_0,d_2)$ (transitivity) (as per Proposition \ref{prop:implic}).
\end{enumerate}
When $\Omega=\{\bot,\top\}$ represents a two-valued logic, $\Gamma$ clearly emerges as an equivalence relation on $\Pi A_i$.

\begin{defn}
Given a semiotic $(S,M)$ and a concept $d$, a diagram $D$ serves as a $\lambda$-codification or a $\lambda$-description for $d$ if
\[
\Gamma(d,M(D))\geq \lambda.
\]
In this context, we also refer to $M(D)$ as an approximation to the concept $d$. Consequently, the relation $D$ acts as a hypothesis that describes the concept encapsulated by $d$, chosen from the language $Lang(S)$.
\end{defn}

This definition can be extended to concept descriptions with differing support sets. Given concept descriptions
\[d_I:\Pi_{i\in I} A_i\rightarrow\Omega \quad \text{and} \quad d_J:\Pi_{j\in J} A_j\rightarrow\Omega,\] 
such that a projection map $\pi:\Pi_{i\in I} A_i\rightarrow\Pi_{j\in J} A_j$ exists, we denote $d_J\preceq d_I$. We term $d_J$ a $\lambda$-projection of $d_I$ if
\[
\Gamma(\pi\otimes d_I, d_J)\geq \lambda.
\]

For a given concept description $d$ and a hypothesis $D$ in $Lang(S)$, the quality of $D$ as a descriptor for $d$ is quantified by:
\[[d=D]=\bigwedge_{\bar{x}}\Gamma(d,D)(\bar{x}),\]
where
\begin{enumerate}
  \item $\Gamma(d,D)(\bar{x})=(\pi\otimes M(D) \Leftrightarrow d)$ if $d\preceq M(D)$,
  \item $\Gamma(d,D)(\bar{x})=( M(D) \Leftrightarrow \pi\otimes d)$ if $M(D)\preceq d$.
\end{enumerate}

We view a model as the fuzzy responses to a query within a semiotic, which we formally define as:

\begin{defn}\label{def:lambda model}
A concept $d$ serves as a $\lambda$-model for $D$ in $Lang(S)$ if either $d\preceq M(D)$ or $M(D)\preceq d$. Additionally, considering the diagram depicted in Fig. \ref{lambdamodel}, we define a $\lambda$-model as a pullback where \[\Gamma_\lambda(d,D)=\Pi A_i,\] given that $[\lambda,\top]$ forms a chain within the lattice $\Omega$. Consequently, we denote this relationship as
\[d\models_\lambda \forall D.\]
\end{defn}

\begin{figure}[h]
\[
\xymatrix @-1pc { \ar @{} [dr] |{\lrcorner}
 \Gamma_\lambda(d,D) \ar[rr] \ar[dd]^{\subset} &&[\lambda,\top] \ar[dd]^{\subset}\\
 &&\\
 \Pi A_i \ar[rr]_{\Gamma(d,D)} &&\Omega
 }
\]
\caption{Diagram evaluation.}\label{lambdamodel}
\end{figure}

In the aforementioned pullback diagram, if
\[\Gamma_\lambda(d,D)\subseteq \Pi A_i \text{ and } \Gamma_\lambda(d,D)\neq \emptyset,\]
we express this relationship as
\[d\models_\lambda \exists D,\]
or, with more precision,
\[d\models_\lambda \forall_B D,\]
where $B=\Gamma_\lambda(d,D)$. This notation is further employed as $d\models_\lambda \forall_C D$ when $C\subseteq \Gamma_\lambda(d,D)$.

When $d\models_\top\forall D$, we simplify this to $d\models\forall D$, indicating that $d$ can be regarded as the answer to the query $D$ within the information system defined by $(S,M)$. Similarly, if $d\models_\top\exists D$, we denote it as $d\models\exists D$, suggesting that a part of $d$ is $\lambda$-consistent with the interpretation of $D$ in the semiotic $(S,M)$.

It's important to note that $\forall_B D$ might not always be interpretable as a formula within the language associated with the employed semiotic, as the language might lack the necessary expressive power to define $B$. However, if \[B=\Gamma_\lambda(d,D)=\Gamma_\top(\chi_B,D'),\] meaning if the domain $B$ can be explicitly defined using the diagram $D'$ in the language, we express this relationship as \[d\models_\lambda \forall_{D'} D.\]
Additionally, in this scenario, for every description $d$, we have
\[
d\models_\lambda \forall_{D'} D \Leftrightarrow d\models_\lambda \forall D' \Rightarrow D.
\]

Given $d_0\models_{\lambda_0} \forall_{B_0} D_0$ and $d_1\models_{\lambda_1} \forall_{B_1} D_1$, we derive:
\[
d_0\otimes d_1\models_{\lambda_0\otimes\lambda_1} \forall_{B_0\cap B_1} D_0\otimes D_1,
\]
\[
d_0\vee d_1\models_{\lambda_0\vee\lambda_1} \forall_{B_0\cap B_1} D_0\vee D_1,
\]
\[
d_0\Rightarrow d_1\models_{\lambda_0\Rightarrow\lambda_1} \forall_{B_0\cap B_1} D_0\Rightarrow D_1.
\]

Given the proposed definition, it follows that every diagram possesses a $\lambda$-model:

\begin{prop}
In a logical semiotic $(S,M)$, if $D$ is a relation defined in the language $Lang(S)$, then
\[MD\models \forall D.\]
\end{prop}

From the introduced notion of similarity, characterized by biimplication, we can identify as $\lambda$-models for $D$ the concepts that are $\lambda$-similar to its interpretation $MD$:

\begin{prop}
If $\Gamma(d,MD)\geq \lambda$, then $d\models_\lambda \forall D$.
\end{prop}

Naturally, we utilize the definition of similarity to formalize our understanding of concepts that are consistent with relations.

\begin{defn}[Consistence]
Given a semiotic $(S,M)$, a relation $D$ from $Lang(S)$ is said to be \emph{consistent with} $d\in M(S)$ if $d\models \forall D$. It is $\lambda$-\emph{consistent with} $d$ when $d\models_\lambda \forall D$. Moreover, the relation $D$ is consistent with a part of $d$ if $d\models \exists D$, and it is $\lambda$-\emph{consistent with a part of} $d$ when $d\models_\lambda \exists D$.
\end{defn}

The set of hypotheses consistent with $d$ is denoted by $Hy_{(S,M)}(d)$. For every $\lambda\in\Omega$, the set of hypotheses $\lambda$-consistent with $d$ is represented as $\lambda$-$Hy_{(S,M)}(d)$. For a chain of truth values in $\Omega$ given by
\[
\top\geq \lambda_0\geq \lambda_1\geq \ldots \geq \lambda_n,
\]
we establish the following inclusion relationships:
\[
Hy_{(S,M)}(d)\subseteq \lambda_0\text{-}Hy_{(S,M)}(d)\subseteq  \lambda_1\text{-}Hy_{(S,M)}(d)\subseteq  \ldots \subseteq  \lambda_n\text{-}Hy_{(S,M)}(d).
\]

\begin{exam}[Description Consistent with a Dataset]
Consider a binary semiotic $(S,M)$ with signs $A,B,C,D$. Let \[d:M(A)\times M(B)\times M(C)\times M(D)\rightarrow \Omega\] be a finite crisp concept description. This means that for every entity $\bar{x}$, the value of $d(\bar{x})$ is either $\top$ or $\bot$, and the number of entities $\bar{x}$ for which $d(\bar{x})=\top$ is finite. In this context, there exists a word $D$ in the language associated with the semiotic that is consistent with $d$, known as the dataset used to describe $d$.

To be more specific, suppose the signs $A,B,C,D$ share the same interpretation, such that $M(A)=M(B)=M(C)=M(D)=[0,1]$. Let's also assume the following:
\begin{enumerate}
  \item $d(1.0,0.5,0.2,0.2)=\top$,
  \item $d(1.0,1.0,0.2,0.2)=\top$, and
  \item $d(1.0,1.0,0.0,0.2)=\top$.
\end{enumerate}
These are the only tuples where $d$ evaluates to $\top$. The relation $d$ is consistent with the diagram
\[
_{(A=1.0 \otimes B=0.5 \otimes C=0.2 \otimes D=0.2)\otimes (A=1.0 \otimes B=1.0 \otimes C=0.2 \otimes D=0.2)\otimes (A=1.0 \otimes B=1.0 \otimes C=0.0 \otimes D=0.2)}
\]
or, equivalently, $d$ serves as the answer to the query defined by the diagram. This can be represented using table notation as shown in Figure \ref{dataset}.

\begin{figure}[h]
\begin{center}
\begin{tabular}{|c|c|c|c|}
  \hline
  A & B & C & D \\
  \hline
  1.0 & 0.5 & 0.2 & 0.2 \\
  1.0 & 1.0 & 0.2 & 0.2 \\
  1.0 & 1.0 & 0.0 & 0.2 \\
  \hline
\end{tabular}
\end{center}
\caption{Dataset.}\label{dataset}
\end{figure}
\end{exam}

\chapter{Fuzzy computability}\label{fuzzy computability}
When the interpretation of a diagram is consistent with a multi-morphism, we consider the multi-morphism to be computable within the semiotic. Formally:

\begin{defn}[Computability]
Given a semiotic $(S,M)$, a multi-morphism $f:A\rightharpoonup B$ is said to be computable in $(S,M)$ if there exists a diagram $D$ in $Lang(S)$ satisfying:
\begin{enumerate}
  \item $A$ serves as the input, with $A=i(D)$, and $B$ as the output, with $B=o(D)$, and
  \item $D$ codifies $f$, meaning that $f\models \forall D$.
\end{enumerate}
The multi-morphism $f$ is said to be $\lambda$-computable in $(S,M)$ if $A=i(D)$, $B=o(D)$, and $f\models_\lambda \forall D$. These definitions are quite restrictive, so we relax them by considering a diagram $D$ as a specification to compute part of $f$ if $d\models \exists D$. When the domain of the computable part of $f$ can be described by a diagram $D'$, we write \[f\models \forall_{D'} D.\]
\end{defn}

When $f\models \forall D$, with $A=i(D)$ and $B=o(D)$, we refer to the diagram $D$ as a \emph{program} or \emph{specification} in the language $Lang(S)$. Its image under $M$ serves as an implementation for the multi-morphism $f:A\rightharpoonup B$.

In this context, every interpretation of words from $Lang(S)$ is computable in the semiotic $(S,M)$. Furthermore, since words in $Lang(S)$ are generated from atomic components, we can state the following proposition:

\begin{prop}
If $f$ and $g$ are computable multi-morphisms in the semiotic $(S,M)$, then $f\otimes g$ is also computable in $(S,M)$.
\end{prop}

Moreover, since $Lang(S)$ is defined by finite diagrams, every finite diagram $D$ in $Set(\Omega)$, comprising arrows of computable multi-morphisms, will have a limit that is a computable multi-morphism.

The set of interpretations of words from $Lang(S)$ and computable multi-morphisms together define a category, denoted by $Hy_{(S,M)}$. In this category, we write $f:d_1\rightarrow d_2$ if $f$ is a computable multi-morphism and $d_1$ and $d_2$ are consistent descriptions in the semiotic, satisfying $d_1\otimes f=d_2$. It's worth noting that if $D$ is consistent with $d_1$ and $D_f$ is the specification for $f$, then the diagram $D\otimes D_f$ is consistent with $d_1\otimes f$.

More generally, if $(d\Leftrightarrow MD)\geq \lambda$ and $(f\Leftrightarrow MD_f)=\top$, then $(d\otimes f \Leftrightarrow MD\otimes f)\geq \lambda$, i.e., $(d\otimes f \Leftrightarrow MD\otimes MD_f)\geq \lambda$. Formally:

\begin{prop}
Let $d$ be a description $\lambda$-consistent with $D$, and let $f$ be a computable multi-morphism specified by $D_f$. Then, $d\otimes f$ is a description $\lambda$-consistent with the diagram $D\otimes D_f$.
\end{prop}

In this context, a computable multi-morphism is recognized as a pre-processing tool in the data mining community. This leads to the definition of $\lambda$-$Hy_{(S,M)}$, the category of concepts that are $\lambda$-consistent and computable multi-morphisms in the semiotic $(S,M)$. Naturally, both the limit and the colimit, in the conventional sense, of finite diagrams in $\lambda$-$Hy_{(S,M)}$ define computable relations. We refer to this type of finite diagrams as \emph{mining schemas}. Given a mining schema $D$ in semiotic $(S,M)$ and a $\lambda$-consistent concept $d$, the limit $Lim\;D$ defines a computable multi-morphism, and $d\otimes Lim\;D$ is a $\lambda$-consistent concept, interpreted as the output of schema $D$ when applied to concept $d$.

As is customary, we extend the notion of computability by defining:

\begin{defn}[Turing computable]
A concept $d$ is called \emph{Turing computable} in the semiotic $(S,M)$ if there exists a diagram $D$, possibly infinite but enumerable, such that \[Lim\;D=d.\]
\end{defn}

Computability is typically associated with state-based systems. In the context of a semiotic, the interpretation of a state must be time-dependent. Given the static definition of sign interpretation presented earlier, we can only capture this dynamic behavior by employing an ontological hierarchy. We envision the possible interpretation of a sign as a class of structures that could serve as potential instantiations for it during system execution. To achieve this, we utilize a syntactic operator that links together signs within the same class, representing different perspectives of the same entity. It is imperative that the class of related signs, connected by the syntactic operator, share the same generalization sign within the sign ontology. The presence of such a syntactic operator in a semiotic is what we referred to as a syntactic operator in chapter \ref{synopt}.

\begin{defn}[Temporal semiotics]
A \emph{temporal semiotic} is defined as a semiotic $(S,M)$, characterized by a library $L:|L|\rightarrow (Chains\downarrow \Sigma^+)$ and augmented with a syntactic operator \[t:\Sigma^+\rightarrow \Sigma^+\] that adheres to the following conditions:
\begin{enumerate}
  \item It maintains the polarization of signs, specifically, $t(s^+)=t(s)^+$ for every $s\in\Sigma$;
  \item It upholds concatenation, meaning $t(w_0.w_1)=t(w_0).t(w_1)$ holds true for every pair of words $w_0,w_1$;
  \item It ensures the functionality of components; for any function $f:w\rightarrow w'$, there exists a corresponding component \[t(f):t(w)\rightarrow t(w').\]
\end{enumerate}
We further require the existence of a component \[t(r):i(t(w))\rightarrow o(t(w))\] for every component $r:i(w)\rightarrow o(w)$. Additionally, we introduce an ontological hierarchy for signs that remains time-invariant when relating time-dependent signs. That is, if $s_1=t(s_0)$, then there exists a sign $s$ such that $s_1\leq s$, $s_0\leq s$, and $s=t(s)$. Consequently, every sequence of time-dependent signs \[s_0,t(s_0),t(t(s_0)),t(t(t(s_0))),\ldots\] is generalized by the same sign $s$ in the ontology. We refer to $s$ as a \emph{time-invariant sign}.
\end{defn}

In a temporal semiotic $(S,M)$, if $r:i(w)\rightarrow o(s(w))$ is a component within the semiotic, then its interpretation $M(r):M(i(w))\rightharpoonup M(o(s(w)))$ is termed a \emph{coalgebra}. A sign $s\in\Sigma$ is deemed \emph{time-invariant} in the semiotic if $M(s)=M(t(s))$.

A \emph{temporal logic semiotic} is characterized as a semiotic that is both a logic semiotic and a temporal semiotic.

\begin{exam}[Fuzzy Turing machine]
A fuzzy Turing machine, whose tape is defined using signs from $F$, can be conceptualized as a word within the language associated with a temporal logic semiotic $(S,M)$. The interpretation of this word can be viewed as its execution. The structure of the machine can be represented within a sign system $S$ with a library $L$, where the signs correspond to a set of machine states, denoted as $Q$, and the components represent Turing machine instructions labeled within a set $I$.

Each instruction in $I$ takes a conditional form: it dictates an action based on the distribution of symbols currently being scanned on the tape. Specifically, there are three categories of actions that can be executed:
\begin{enumerate}
  \item Print: Modify the symbol distribution in the square currently being scanned;
  \item Move one square to the right;
  \item Move one square to the left.
\end{enumerate}
Thus, depending on the active instruction and the symbol distribution being scanned, the machine or its operator will execute one of these actions.

An instruction defines a link between two states and is encoded as component labels with the following structure:
\begin{enumerate}
  \item $q_0[f] q_1$: If in state $q_0$ the scanned distribution is modified using the interpretation of component $f$, the machine transitions to state $q_1$;
  \item $q_0[d_0:L]_\lambda q_1$: If in state $q_0$ the machine reads a distribution $d$ and $d\otimes M(d_0)\geq\lambda$, then the machine moves left and transitions to state $q_1$;
  \item $q_0[d_0:R]_\lambda q_1$: If in state $q_0$ the machine reads a distribution $d$ and $d\otimes M(d_0)\geq\lambda$, then the machine moves right and transitions to state $q_1$.
\end{enumerate}
An instruction is executed if its condition is satisfied.

In this context, a diagram in $Lang(L)$, defined using time-invariant signs, serves as a \emph{Turing machine specification}, with states in $Q$ and tape symbols from $F$. Every refinement of a Turing machine specification in $Lang(L)$, defined using only time-variant signs, is referred to as a \emph{flow chart}, and it encodes a Turing machine's execution. To ensure the correct interpretation of an instruction, for each state $q_i\in Q$ in the sign system, we have signs 
\[
q_i^{(r)},q_i^{(m)},q_i^{(l)},q_i^{(hr)},q_i^{(tr)},q_i^{(hl)},q_i^{(tl)}
\]
where $q_i^{(r)}$ is interpreted as the tape's right half, $q_i^{(m)}$ denotes the reading square, and $q_i^{(l)}$ represents the tape's left half. Additionally, for each tape half, we select the right half head $q_i^{(hr)}$, the right tail head $q_i^{(tr)}$, the left half head $q_i^{(hl)}$, and the left tail head $q_i^{(tl)}$. The structure of the sign system is defined in such a way that the relationship between these signs and $q_i$ is preserved if a model $M$ satisfies:

\begin{enumerate}
   \item $M(q_i)=M(q_i^{(r)})\otimes_I M(q_i^{(m)})\otimes_I M(q_i^{(l)})$;
   \item $M(q_i^{(r)})=M(q_i^{(hr)})\otimes_I M(q_i^{(tr)})$;
   \item $M(q_i^{(l)})=M(q_i^{(hl)})\otimes_I M(q_i^{(tl)})$;
\end{enumerate}
This interpretation for signs reflects the relations between I-projections (see Example \ref{def:indexproduct}) as expressed in the following diagram:

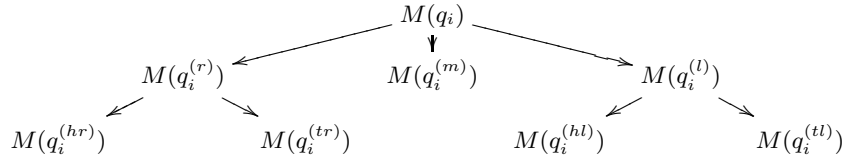
\begin{figure}[h]
\[
\small
\xymatrix @=7pt {
&&&M(q_i)\ar[lld]\ar[rrd]\ar[d]&&&\\
&M(q_i^{(r)})\ar[ld]\ar[rd]&&M(q_i^{(m)})&&M(q_i^{(l)})\ar[ld]\ar[rd]&\\
M(q_i^{(hr)})&&M(q_i^{(tr)})&&M(q_i^{(hl)})&&M(q_i^{(tl)})
}
\]
\caption{Sign interpretation structure.}\label{struct}
\end{figure}
And models of each instruction must satisfy the following conditions:
\begin{enumerate}
  \item For the print instruction $q_0[f] q_1$, we should have:
\[
_{M(q_i [f] t(q_j))^\circ\otimes M(q_i)\otimes M(q_i [f] t(q_j))=M(t(q_j))}
 \Rightarrow
\left\{
  \begin{array}{l}
    _{M(t(q_j)^{(r)})=M(q_i^{(r)})} \\
    _{M(t(q_j)^{(m)})=M(f)^\circ\otimes M(q_i^{(m)})\otimes M(f)}\\
    _{M(t(q_j)^{(l)})=M(q_i^{(l)})}\\
  \end{array}
\right.
\]
  \item For instructions of the type "Move one square to the left" $q_0[d_0:L]_\lambda q_1$, we must have:
\[
_{\left\{
  \begin{array}{l}
    _{M(q_i^m)\otimes M(d_0)\geq\lambda}\\
    _{M(q_i [d_0:L]_\lambda t(q_j))^\circ\otimes M(q_i)\otimes M(q_i [d_0:L]_\lambda t(q_j))=M(t(q_j))} \\
  \end{array}
\right.
\Rightarrow
\left\{
  \begin{array}{l}
    _{M(t(q_j)^{(r)})=M(q_j^{(r)})\otimes_I M(q_i^{(m)})} \\
    _{M(t(q_j)^{(m)})=M(q_i^{(hl)})} \\
    _{M(t(q_j)^{(l)})=M(q_i^{(tl)})}\\
  \end{array}
\right.}
\]

\item For instructions of the type "Move one square to the right" $q_0[d_0:R]_\lambda q_1$, we must have:
\[
_{\left\{
  \begin{array}{l}
    _{M(q_i^m)\otimes M(d_0)\geq\lambda}\\
    _{M(q_i [d_0:R]_\lambda t(q_j))^\circ\otimes M(q_i)\otimes M(q_i [d_0:R]_\lambda t(q_j))=M(t(q_j))} \\
  \end{array}
\right.
\Rightarrow
\left\{
  \begin{array}{l}
    _{M(t(q_j)^{(r)})=M(q_i^{(hr)})} \\
    _{M(t(q_j)^{(m)})=M(q_i^{(tr)})} \\
    _{M(t(q_j)^{(l)})=M(q_i^{(m)}) \otimes_I M(q_i^{(tl)})}\\
  \end{array}
\right.}
\]
\end{enumerate}
Thus, a model $M$ assigns to each state a \emph{fuzzy tape} with signs in $F$, which can be seen as an infinite chain of indexed products (see example \ref{def:indexproduct}):
\[
t=\underbrace{\overbrace{\cdots\otimes_I d_{5}}^{t^{(hr)}}\otimes_I \overbrace{d_{3}}^{t^{(tr)}}}_{t^{(r)}}\otimes_I \underbrace{d_{1}}_{t^{(m)}} \otimes_I \underbrace{\overbrace{d_{2}}^{t^{(hl)}}\otimes_I \overbrace{d_{4}\otimes_I\cdots}^{t^{(tl)}}}_{t^{(l)}}
\]
where we fix a component $t^{(m)}=d_{1}$, and such that each $d_i$ is a concept description $d_i:F\times I\rightarrow \Omega$. Moreover, the model $M$ associates with each possible instruction (component) a relation between fuzzy tapes $t_0$ and $t_1$, satisfying the described properties.

A fuzzy Turing machine begins its execution in an initial state and is a parallel device; at any given moment, it can be in more than one state. It finishes its execution when it is stalled in a state or set of states.

\end{exam}

\chapter{Consequence relation}\label{consequence relation}
In a semiotic $(S,M)$, we define for every relation $D$ in $Lang(S)$ the set of its $\lambda$-answers as:
\[ ans_\lambda(D)=\{ g \in  M(S): g\models_\lambda \forall_{D'} D\}\]
This set can be interpreted as the collection of concepts that are $\lambda$-consistent with $D$, defined on the domain specified by $D'\in Lang(S)$.

\begin{exam}
The examples presented in this chapter utilize a grid semiotic, which has the expressive power to encode structures on a grid, employing a three truth-values logic $\Omega=\{\bot,\frac{1}{2},\top\}$.

Let $D$ be the diagram that defines a relation between pairs of entities in a grid, as illustrated in Fig. \ref{grid1}. In this figure, white points $\bar{x}$ correspond to $M(D)(\bar{x})=\bot$, gray points represent $M(D)(\bar{x})=\frac{1}{2}$, and darker points indicate $M(D)(\bar{x})=\top$.

\begin{figure}[h]
\begin{center}
\includegraphics[width=1.9cm]{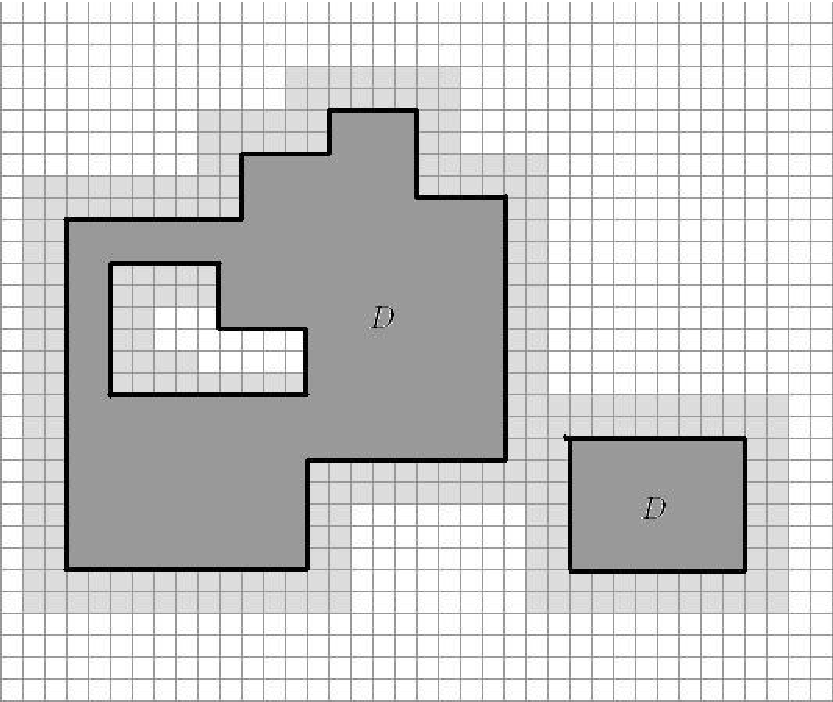} \hspace{1cm}
\includegraphics[width=1.9cm]{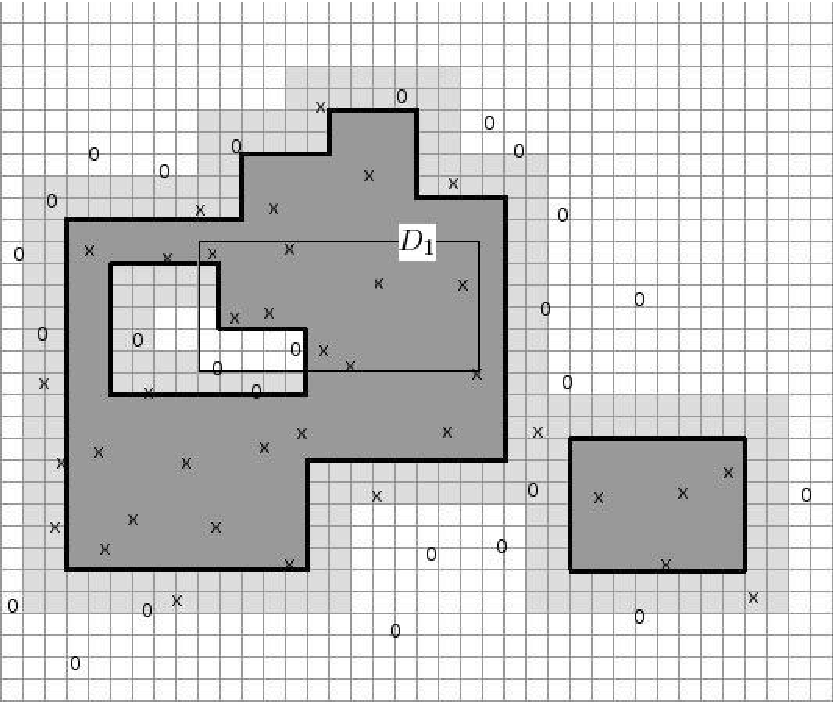}
\end{center}
\caption{Relation defined by interpreting $D$ and finite relations $g_1$. A point $\bar{x}$ marked with an X represents $g_1(\bar{x})=\top$, while a point marked with 0 corresponds to $g_1(\bar{x})=\bot$.}\label{grid1}
\end{figure}

The interior of the box, denoted as $D_1$ in the figure, can be viewed as the set of points described by this diagram. The relation $g_1$ depicted can be considered as an example that satisfies
\[
g_1\models_{\frac{1}{2}} \forall D, \quad g_1\models \forall_{D_1}D,
\]
which can be expressed concisely as
\[
g_1\in ans_{\frac{1}{2}}(D).
\]
\end{exam}

Note that given a relation $D$, the set $ans_\lambda(D)$ contains at least one element, namely $M(D)\in ans_\lambda(D)$. Naturally, we have:
\begin{thm}
If $D$ is a relation in $Lang(S)$ and $\lambda_0\leq\lambda_1$, then
\[ans_{\lambda_1}(D)\subseteq ans_{\lambda_0}(D).\]
\end{thm}
If $g\in ans_\lambda(D)$, and $g\models_\lambda \forall_{D'} D$, we denote this relation by writing $g_{D'}\in ans_\lambda(D)$.

Let $D$ be a relation defined in a semiotic by
\[f\leq_D g,\]
this implies that
\[\text{if } M(D)(\bar{x})=\top, \text{ then } f(\bar{x})\leq g(\bar{x}).\]
We utilize this relation and the operator $ans_\lambda$ to introduce two modal operators, $\diamond_\lambda g$ and $\Box_\lambda g$, which represent the weak and strong images, respectively, of the description $g\in M(S)$ along the relation $\models_\lambda$:
\[\diamond_\lambda g=\{D\in Lang_R(S): (\exists f_{D'}\in ans_\lambda(D)) (g\leq_{D'} f) \} \]
\[\Box_\lambda g=\{D\in Lang_R(S): (\forall f_{D'}\in ans_\lambda(D))(f\leq_{D'} g) \} \]
Here, $\diamond_\lambda g$ and $\Box_\lambda g$ can be interpreted, respectively, as the set of models that are $\lambda$-consistent with parts of $g$ and the set of models that are $\lambda$-consistent with $g$ in the language $Lang(S)$.

\begin{exam}
For a grid semiotic with three truth-values, we present in Fig. \ref{grid2} two possible diagrams: $D_1\in \diamond_{\frac{1}{2}} g_1$ and $D_2\in\Box_{\frac{1}{2}} g_2$.

\begin{figure}[h]
\begin{center}
\includegraphics[width=80pt]{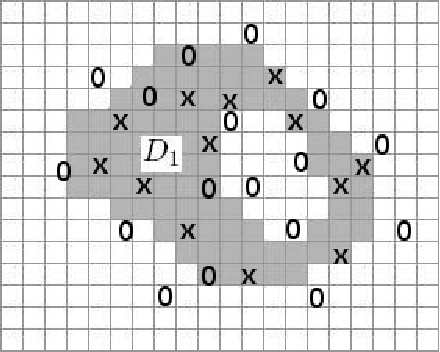} \hspace{1cm}
\includegraphics[width=80pt]{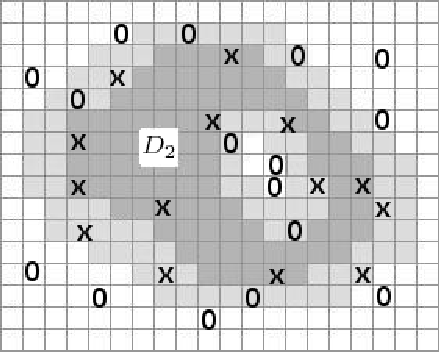}
\end{center}
\caption{Examples of $D_1\in \diamond_{\frac{1}{2}} g_1$ and $D_2\in\Box_{\frac{1}{2}} g_2$.}\label{grid2}
\end{figure}
\end{exam}

We have the following theorem:

\begin{thm}
Given relations $D_0$ and $D_1$ in $Lang_R(S)$ and a description $g$:
\begin{enumerate}
  \item If $\lambda_0\leq\lambda_1$, then $\Box_{\lambda_1} g\subseteq \Box_{\lambda_0} g$.
  \item If $\lambda_0\leq\lambda_1$, then $\diamond_{\lambda_1} g\subseteq \diamond_{\lambda_0} g$.
  \item If $D_0,D_1\in \Box_\lambda g$, then $D_0\vee D_1\in \Box_\lambda g$.
  \item If $D_0,D_1\in \diamond_\lambda g$, then $D_0\wedge D_1\in \diamond_\lambda g$.
\end{enumerate}
\end{thm}

In the other direction, we can extend $ans_\lambda$ to a set of relations $U$ in $Lang(S)$ as follows:
\[ ans_\lambda(U) = \bigvee\{g \in M(S): (\exists D \in \Box_\lambda g)(D\in U)\} \]
This provides the greatest description $\lambda$-consistent with a model from $U$. Additionally, let
\[ mod_\lambda(U) = \bigwedge\{g \in M(S): (\forall D \in \Box_\lambda g)(D\in U)\} \]
be a description that is $\lambda$-consistent with every model in $U$.

\begin{thm}
Let $U$ and $V$ be sets of relations in $S$. Then:
\begin{enumerate}
  \item If $\lambda_0\leq\lambda_1$, then $ans_{\lambda_0}(U)\geq ans_{\lambda_1}(U)$.
  \item $ans_\lambda(U\cup V) = ans_\lambda(U)\vee ans_\lambda(V)$.
  \item $mod_\lambda(U\cup V) = mod_\lambda(U)\wedge mod_\lambda(V)$.
\end{enumerate}
Moreover, if $U\subseteq V$, we have:
\begin{enumerate}
  \item $ans_\lambda(U)\leq ans_\lambda(V)$.
  \item $mod_\lambda(V)\leq mod_\lambda(U)$.
\end{enumerate}
\end{thm}

The $\lambda$-interior of a concept $g$ within the semiotic system $(S,M)$ is defined as the largest portion of $g$ that is $\lambda$-consistent with a model defined in the associated language. Mathematically, it is represented as:
\[ int_\lambda(g) = \bigvee\{h \in M(S): (\exists D \in \Box_\lambda h)(\forall f_{D'} \in ans_\lambda(D))(f \leq_{D'} g)\}, \]
This can be understood as the maximal fragment of $g$ that has a $\lambda$-consistent model within the language of the semiotic. The operator $int_\lambda$ is considered an \emph{interior operator} because:
\begin{enumerate}
  \item $int_\lambda(g) \leq g$,
  \item if $g \leq f$, then $int_\lambda(g) \leq int_\lambda(f)$, and
  \item $int_\lambda(g) = int_\lambda(int_\lambda(g))$.
\end{enumerate}
Furthermore, given $\lambda_0 \leq \lambda_1$, we have $int_{\lambda_1}(g) \leq int_{\lambda_0}(g)$. A concept description $g$ is termed $\lambda$-\emph{open} in $(S,M)$ if:
\[ int_\lambda(g) = g. \]
Both $ans_\lambda(U)$ and $mod_\lambda(U)$ for any set of relations $U$ serve as examples of $\lambda$-open concepts due to the following proposition:

\begin{prop}
In a semiotic, for any set of relations $U$ and for any $\lambda \in \Omega$:
\begin{enumerate}
  \item $int_\lambda(ans_\lambda(U)) = ans_\lambda(U)$, and
  \item $int_\lambda(mod_\lambda(U)) = mod_\lambda(U)$.
\end{enumerate}
\end{prop}

To delve into the details:

\begin{thm}
In a semiotic system, for any set of relations $U$:
\[ans_\top(U) \models \bigvee U \quad \text{and} \quad mod_\top(U) \models \bigwedge U.\]
\end{thm}

The closure of a concept $g$ within the semiotic system $(S,M)$ is defined as the minimal cover of $g$ represented in the language $L(S)$. Mathematically, it is expressed as:
\[ cl_\lambda(g) = \bigwedge\{h \in M(S): (\forall D \in \Box_\lambda h)(\exists f_{D'} \in ans_\lambda(D))(g \leq_{D'} f)\}, \]
This can be interpreted as the most concise cover that includes $g$, represented in the language associated with the semiotic system. The operator $cl_\lambda$ is termed a \emph{closure operator} because:
\begin{enumerate}
  \item $g \leq cl_\lambda(g)$,
  \item if $g \leq f$, then $cl_\lambda(g) \leq cl_\lambda(f)$, and
  \item $cl_\lambda(cl_\lambda(g)) = cl_\lambda(g)$.
\end{enumerate}
Furthermore, for $\lambda_0 \leq \lambda_1$, we have $cl_{\lambda_1}(g) \leq cl_{\lambda_0}(g)$. It is also evident that:

\begin{prop}
For a given semiotic system $(S,M)$, for every $g \in M(S)$:
\[int_\lambda(g) \leq g \leq cl_\lambda(g).\]
\end{prop}

A concept description $g$ is termed $\lambda$-\emph{close} in the semiotic system $(S,M)$ if it satisfies:
\[cl_\lambda(g) = g.\]
Both the descriptions $ans_\lambda(U)$ and $mod_\lambda(U)$ also qualify as $\lambda$-closed concepts. This concept can be generalized to encompass every $\lambda$-open description:

\begin{prop}
For a given semiotic system $(S,M)$, every $g \in M(S)$ is $\lambda$-closed if and only if it is $\lambda$-open.
\end{prop}

In this context, when a description is either $\lambda$\emph{-open} or $\lambda$\emph{-close}, we refer to it as a description that is $\lambda$-representable within the semiotic system. This leads us to the following proposition:

\begin{prop}
Let $g$ be a description within the semiotic system $(S,M)$. There exists a relation $D$ such that $g \models_\lambda D$ if and only if $g$ is $\lambda$-open or $\lambda$-close.
\end{prop}

Due to the symmetry between the left and right sides of $d \models D$, the definitions yield:
\[int_\lambda = ans_\lambda \Box_\lambda \quad \text{and} \quad cl_\lambda = mod_\lambda \diamond_\lambda\]
These definitions also possess symmetric counterparts, which can be derived by replacing each operator with its symmetric equivalent:
\[\A_\lambda = \Box_\lambda ans_\lambda \quad \text{and} \quad \C_\lambda = \diamond_\lambda mod_\lambda.\]
By this symmetry, it is immediate to conclude that $\C_\lambda$ is an interior operator and $\A_\lambda$ is a closure operator.

Elaborating on the definition of $\A_\lambda$, for every set of relations $U$:
\[\A_\lambda (U) = \{D \in Lang(S): (\forall f_{D'} \in ans_\lambda(D))(f \leq_{D'} ans_\lambda (U))\},\]
This means that all $\lambda$-answers for $D$ are $\lambda$-codified using relations from $U$. This leads us to:

\begin{thm}\label{soundness}
For every pair of relations $U$ and $V$ in the semiotic system $(S,M)$, the following properties hold:
\begin{enumerate}
  \item $\A_\lambda(U \cup V) \supseteq \A_\lambda(U) \cup \A_\lambda(V)$,
  \item if $U \subseteq V$, then $\A_\lambda(U) \subseteq \A_\lambda(V)$,
  \item if $\lambda_0 \leq \lambda_1$, then $\A_{\lambda_1}(U) \subseteq \A_{\lambda_0}(U)$,
  \item if $D \in \A_\lambda(U)$, then $\bigvee ans_\lambda(D) \leq ans_\lambda(U)$, and
  \item if $D \in \A_\lambda(U)$, then $ans_\lambda(U) \models_\lambda D$.
\end{enumerate}
\end{thm}

Expanding on the definition of the operator $\C_\lambda$, we have:
\[\C_\lambda (U) = \{D \in Lang(S): (\exists f_{D'} \in ans_\lambda(D)) (mod_\lambda(U) \leq_{D'} f)\},\]
Where $D \in \C_\lambda (U)$ implies that $D$ has a $\lambda$-answer, and every $\lambda$-codification for it is in $U$. Based on this, we can establish:

\begin{thm}
For every pair of relations $U$ and $V$ in the semiotic system $(S,M)$:
\[\C_\lambda(U \cup V) \subseteq \C_\lambda(U) \cap \C_\lambda(V)\]
and when $U \subseteq V$, we have $\C_\lambda(U) \subseteq \C_\lambda(V)$.
\end{thm}

Let's define:
\[ U \vdash_\lambda D \text{ iff } D \in \A_\lambda(U), \]
Given that $\A_\lambda$ is a closure operator, the following properties hold:

\begin{thm}
In a semiotic system $(S,M)$, for every $\lambda$, we have:
\begin{enumerate}
    \item if $D \in U$, then $U \vdash_\lambda D$ (Inclusion),
    \item if $U \vdash_\lambda D$, then $U \cup V \vdash_\lambda D$ (Monotony), and
    \item if $V \vdash_\lambda D$ and $U \cup \{D\} \vdash_\lambda D'$, then $U \cup V \vdash_\lambda D'$ (Cut).
\end{enumerate}
\end{thm}

This implies that $(Lang_R(S),\vdash_\lambda)$ serves as an \emph{inference system} \cite{MaibaumAbramskyGabbay92} for every $\lambda$.

\begin{exam}
Consider the interpretations, depicted in Fig. \ref{grid4}, for three diagrams $D_0$, $D_1$, and $D_2$ within the grid semiotic employing three-valued logic:

\begin{figure}[h]
\begin{center}
\includegraphics[width=150pt]{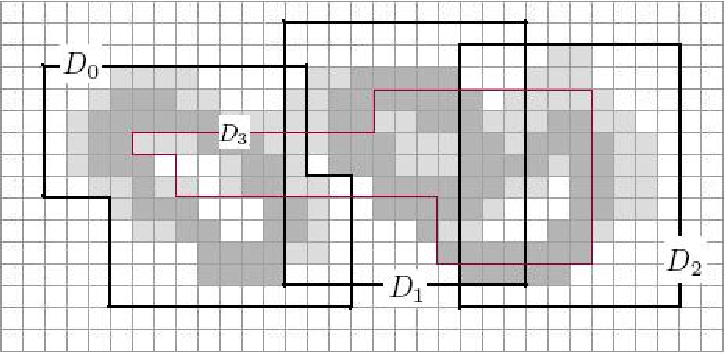}
\end{center}
\caption{Interpretations for diagrams $D_0$, $D_1$, $D_2$, and $D_3$.}\label{grid4}
\end{figure}

The diagram $D_3$, with its represented interpretation, can be viewed as the result of applying inference to the set of diagrams $\{D_0,D_1,D_2\}$. Symbolically, we express this as:
\[D_0,D_1,D_2 \vdash D_3.\]
\end{exam}

Given that $\A_{\lambda_0} \leq \A_{\lambda_1}$ whenever $\lambda_0 \geq \lambda_1$, we can represent this relationship as:

\[
\infer{U \vdash_{\lambda_1} D_0}{U \vdash_{\lambda_0} D_0 & \lambda_0 \geq \lambda_1}
\]

Utilizing the definition of $\lambda$-consistency, we can state the following theorem:

\begin{thm} 
Let $g \models_{\lambda_0} \forall_{D'_0} D_0$ and $f \models_{\lambda_1} \forall_{D'_1} D_1$:
\begin{enumerate}
  \item $g \wedge f \models_{\lambda_0 \wedge \lambda_1} \forall_{D'_0 \otimes D'_1} (D_0 \wedge D_1)$,
  \item $g \vee f \models_{\lambda_0 \vee \lambda_1} \forall_{D'_0 \otimes D'_1} (D_0 \vee D_1)$,
  \item $g \otimes f \models_{\lambda_0 \otimes \lambda_1} \forall_{D'_0 \otimes D'_1} (D_0 \otimes D_1)$, and
  \item $g \Rightarrow f \models_{\lambda_0 \Rightarrow \lambda_1} \forall_{D'_0 \otimes D'_1} (D_0 \Rightarrow D_1)$.
\end{enumerate}
\end{thm}

Applying these properties to the definition of $\lambda$-answer, we obtain:

\begin{thm} 
Within a semiotic system, we have:
\begin{enumerate}
  \item For every diagram $D \in Lang_R(S)$:
   \begin{enumerate}
     \item $ans_{\lambda_0 \vee \lambda_1}(D) = ans_{\lambda_0}(D) \vee ans_{\lambda_1}(D)$,
     \item $ans_{\lambda_0 \wedge \lambda_1}(D) = ans_{\lambda_0}(D) \wedge ans_{\lambda_1}(D)$, and
     \item $ans_{\lambda_0 \otimes \lambda_1}(D) = ans_{\lambda_0}(D) \otimes ans_{\lambda_1}(D)$;
   \end{enumerate}

  \item For every concept description $g \in M(S)$:
  \begin{enumerate}
    \item $\Box_{\lambda_0 \vee \lambda_1}(g) = \Box_{\lambda_0}(g) \vee \Box_{\lambda_1}(g)$,
    \item $\Box_{\lambda_0 \wedge \lambda_1}(g) = \Box_{\lambda_0}(g) \wedge \Box_{\lambda_1}(g)$, and
    \item $\Box_{\lambda_0 \otimes \lambda_1}(g) = \Box_{\lambda_0}(g) \otimes \Box_{\lambda_1}(g)$;
  \end{enumerate}

  \item For every set of diagrams $U \subset Lang_R(S)$:
  \begin{enumerate}
    \item $ans_{\lambda_0 \vee \lambda_1}(U) = ans_{\lambda_0}(U) \vee ans_{\lambda_1}(U)$,
    \item $ans_{\lambda_0 \wedge \lambda_1}(U) = ans_{\lambda_0}(U) \wedge ans_{\lambda_1}(U)$, and
    \item $ans_{\lambda_0 \otimes \lambda_1}(U) = ans_{\lambda_0}(U) \otimes ans_{\lambda_1}(U)$.
  \end{enumerate}
\end{enumerate}
\end{thm}

This provides support for the introduction rules' definition:

\[
\infer{U \vdash_{\lambda_0 \vee \lambda_1} (D_0 \vee D_1), \quad
       U \vdash_{\lambda_0 \wedge \lambda_1} (D_0 \wedge D_1), \quad
       U \vdash_{\lambda_0 \otimes \lambda_1} (D_0 \otimes D_1)}
      {U \vdash_{\lambda_0} D_0 \quad \text{and} \quad U \vdash_{\lambda_1} D_1}
\]

The assertion that if $D_0 \wedge D_1 \in \A_\lambda(U)$ then both $D_0 \in \A_\lambda(U)$ and $D_1 \in \A_\lambda(U)$ can be captured by the elimination rule:

\[
\infer{U \vdash_\lambda D_0, \quad U \vdash_\lambda D_1}
      {U \vdash_\lambda (D_0 \wedge D_1)}
\]

In a divisible logic setting, we observe:

\[
\infer{U \vdash_{\lambda_0 \otimes \lambda_1} (D_0 \wedge D_1)}
      {U \vdash_{\lambda_0} D_0 \quad \text{and} \quad U \vdash_{\lambda_1} (D_0 \Rightarrow D_1)}
\]

This is because, if $g \models_{\lambda_0} \exists D_0$ and $g \models_{\lambda_1} \exists (D_0 \Rightarrow D_1)$, then $g \models_{\lambda_0 \otimes \lambda_1} \exists D_1$. Specifically, if $D_0 \in \A_{\lambda_0}(U)$ and $D_0 \Rightarrow D_1 \in \A_{\lambda_1}(U)$, it follows that $D_0 \wedge D_1 \in \A_{\lambda_0 \otimes \lambda_1}(U)$. This implies that for every $f \in ans_{\lambda_0}(D_0)$, we have $f \leq ans_{\lambda_0}(U)$, and for every $h \in ans_{\lambda_0}(D_1)$, $f \Rightarrow h \in ans_{\lambda_1}(D_0 \Rightarrow D_1)$. Consequently, $f \otimes (f \Rightarrow h) \in ans_{\lambda_0 \otimes \lambda_1}(U)$. It's worth noting that in a divisible ML-algebra, $f \otimes (f \Rightarrow h) \leq f \wedge h$. Hence, $f \wedge h \in ans_{\lambda_0 \otimes \lambda_1}(U)$.

A diagram \(D\) encapsulates all the information present in a concept \(d\), employing the syntax associated with the semiotic \((S,M)\). This occurs if, for every diagram \(D_1\) such that \(d \models_\lambda \forall D_1\), we have \(M(D \vee D_1) = M(D)\). Such diagrams are termed \emph{total} and can be defined as:

\[
D = \bigvee_{d \models_\lambda \forall D_i} D_i.
\]

In the category \(Lang(S)\), whose objects are diagrams encoding relations, and where a diagram \(D\) serves as a morphism from relation \(D_0\) to relation \(D_1\) if \(D_0 \otimes D = D_1\), we conceptualize composition as the operation of diagram gluing. The consequence operator \(\vdash_\lambda\) can be viewed as a functor:

\[
\vdash_\lambda: Lang(S) \rightarrow Lang(S).
\]

A diagram \(D\) is denoted as a \emph{theory} in the \(\lambda\)-semiotic \((S,M)\) if it stands as a fixed-point for the consequence operator:

\[
\vdash_\lambda(D) = D.
\]

The semiotic model \(M\) can be interpreted as a functor:

\[
M: Lang(S) \rightarrow \lambda\text{-}Hy_{(S,M)},
\]

within the category of \(\lambda\)-consistent concepts and computable multi-morphisms. Conversely, a functor in the opposite direction can be formulated using the consistency operator:

\[
\models_\lambda: \lambda\text{-}Hy_{(S,M)} \rightarrow Lang(S),
\]

which assigns to each \(\lambda\)-consistent description a total diagram along with its codification in the semiotic.

Given \(M(D) \models_\lambda \forall D\), it follows that:

\[
M \circ \models_\lambda = \text{id},
\]

and, by the definition of the consequence relation:

\[
\models_\lambda \circ M = \vdash_\lambda.
\]

If \(D\) represents a \(\lambda\)-theory in the semiotic, then \(M(D)\) serves as the \(\lambda\)-consistent model associated with this theory.
.

\chapter{Integration}\label{integration}
Our objective is to construct an integrated semiotic foundation derived from multiple distinct semiotics. This necessity may arise, for instance, when knowledge bases are independently acquired through interactions with various domain experts. Similarly, this challenge can emerge when separate knowledge bases are produced by distinct learning algorithms. The ultimate aim of integration is to create a unified system that leverages all available knowledge, ensuring high performance, i.e., a substantial degree of consistency with the data.

It is crucial to differentiate between two forms of integration: semiotic integration and model integration within a semiotic framework. Semiotic integration seeks to define a semiotic that amalgamates the syntax and semantics of a given set of semiotics. Conversely, model integration within a semiotic refers to the enhancement of concept descriptions by integrating models derived from diverse data sources or differing perspectives on the same data. The integration of models is governed by an integration schema that delineates the relationships between various models within the same semiotic framework. In semiotic integration, we amalgamate different logics within a single semiotic, which are associated with distinct languages employed by domain experts or structural specification languages. In both contexts, knowledge integration, combined with inference, can significantly facilitate the knowledge acquisition process.

We impose a critical constraint on semiotic integration: Given a family of semiotics \((S_i, M_i)_I\), its integration is defined if and only if equal signs and components bearing identical labels across different semiotics share the same interpretation. The only potential exception to this rule pertains to the interpretations of the sign \(\Omega\), which is associated with the logic of the semiotics and its operators and may vary.

The integration of semiotics \((S_i,M_i)_I\) is denoted by \((\bigcup_I S_i, \bigcup_I M_i)\) and is represented by the sign system:
\[
\bigcup_I S_i = (\bigcup_I L_i, \bigcup_I \E_i, \bigcup_I \U_i, \bigcup_I co\U_i),
\]
where each semiotic \(S_i\) is defined by the structure \((L_i, \E_i, \U_i, co\U_i)\) for \(i \in I\). Here, \(\bigcup_I L_i\) represents the library defined as the union of libraries \((L_i: |L_i| \rightarrow (Chains \downarrow \Sigma_i^+))_I\) associated with each semiotic. This library is given by:
\[
\bigcup_I L_i: \bigcup_I |L_i| \rightarrow (Chains \downarrow \bigcup_I \Sigma_i^+),
\]
where the signs consist of the union of ontologies \(\bigcup_I \Sigma_i^+\), defined by the signs present in each library, and the component labels comprise the union of labels existing in both libraries. It's worth noting that the integration of libraries must maintain the functionalities of components. Hence, the union of libraries is only defined if components in different libraries, with identical labels, exhibit the same functionalities. The graphical language associated with \(\bigcup_I L_i\) is denoted as \(Lang(\bigcup_I L_i)\), and we have \(\bigcup_I \E_i \subset Lang(\bigcup_I L_i)\).

Recalling the description for the language associated with \(\bigcup_I L_i\), given two graphs \(G_0\) and \(G_1\), we define \(G_0 \bigcup G_1\) as the graph having vertices from both \(G_0\) and \(G_1\) and arrows from both \(G_0\) and \(G_1\). If each library \(L_i\) is associated with multi-graphs \(\G(L_i)\), then:
\[
\G(\bigcup_I L_i) = \bigcup_I \G(L_i).
\]
If we have models \((M_i: \G(L_i^\ast) \rightarrow Set(\Omega))_I\) for different libraries, the homomorphism \(\bigcup_I M_i\) serves as a model for the sign system \(\bigcup_I S_i\):
\[
\bigcup_I M_i: \G((\bigcup_I L_i)^\ast) \rightarrow Set(\Omega)
\]
constructed using the union of models \((M_i: \G(L_i^\ast) \rightarrow Set(\Omega))_I\), assigning to nodes \(\bigcup_I M_i(v) = M_i(v)\) if \(v \in \G(L_i^\ast)\) and \(v \neq \Omega\) for some \(i \in I\), and \(\bigcup_I M_i(v) = M_i(v)\) for multi-arrows \(f: w \rightarrow w'\) in \(\G(L_i^\ast)\) where \(w' \neq \Omega\) for some \(i \in I\). This implies that logical signs and multi-arrows not representing relations are interpreted as if they were in their respective libraries. This definition is meaningful only when identical signs and labels are interpreted consistently across different libraries.

For the family of logical signs \((\Omega_i)_I\) associated with the family of logical semiotics \((S_i)_I\), we define the sign \(\prod_I \Omega_i\) interpreted by \(\bigcup_I M_i\) as the ML-algebra product \(\prod_I M_i(\Omega_i)\). The interpretation of the sign \(\prod_I \Omega_i\) yields an ML-algebra, and for every relation \(r: w \rightarrow \Omega_i\) present in each semiotic \(S_i\), its interpretation by \(\bigcup_I M_i\) is given by:
\[
\bigcup_I M_i(r: w \rightarrow \Omega_i) = M_i(r) \otimes \pi^\top_{\Omega_i}
\]
where \(\pi^\top_{\Omega_i}: \Omega_i \rightarrow \prod_I \Omega_i\) is the mapping defined as:
\[
\pi^\top_{\Omega_i}(\alpha) = (\top, \ldots, \top, \alpha, \top, \ldots, \top)
\]
with \(\alpha\) appearing in the component of order \(i\) and differing from \(\top\). Formally,

\begin{prop}
If
\[
S_0 = (L_0, \E_0, \U_0, co\U_0), S_1 = (L_1, \E_1, \U_1, co\U_1), \ldots, S_n = (L_n, \E_n, \U_n, co\U_n)
\]
are sign systems with models \(M_0, M_1, \ldots, M_n\), then the mapping
\[
\bigcup_I M_i: \G((\bigcup_I L_i)^\ast) \rightarrow Set(\Omega)
\]
defined as above serves as a model for the sign system
\[
\bigcup_I S_i = (\bigcup_I L_i, \bigcup_I \E_i, \bigcup_I \U_i, \bigcup_I co\U_i).
\]
\end{prop}

Given models \(M_0, M_1, \ldots, M_n\) for sign systems \(S_0, S_1, \ldots, S_n\), by Definition \ref{Modelspec}, for every \(j = 1, \ldots, n\):

\begin{enumerate}
    \item If \(D \in \E_j\), then \(\bigcup_I M_i(D) = M_j(D)\) is a total multi-morphism.
    \item If \((s, D, i(D), o(D)) \in \U_j\), then \(\bigcup_I M_i(s) = M_j(s)\) represents the \(\Omega\)-set defined by \(Lim\; MD\).
    \item If \((s, D, i(D), o(D)) \in co\U_j\), then \(\bigcup_I M_i(s) = M_j(s)\) represents the \(\Omega\)-set defined by \(coLim\; MD\).
\end{enumerate}

Naturally, the resulting semiotic from an integration process inherits both the syntax and semantics generated by the syntax and semantics associated with the individual semiotics. A similar principle applies to certain syntactic operators. The integration of a family of semiotics, where at least one is a differential semiotic, yields a differential semiotic. The same holds true for temporal semiotics. If \((S_i)_I\) is a family where \((S_j)_J\) forms a subfamily of temporal semiotics, defined by the syntactic operator \(t_i: \Sigma^+_i \rightarrow \Sigma^+_i\), then the integrated semiotic \(\bigcup_I S_i\) is a temporal semiotic. Here, the syntactic operator \(t: \bigcup_I \Sigma^+_i \rightarrow \bigcup_I \Sigma^+_i\) is defined as follows:

\begin{enumerate}
    \item \(t(s) = t_j(s)\) if \(s \in \Sigma^+_j\) where \(j \in J\), and
    \item \(t(s) = s\) if \(s \notin \bigcup_J \Sigma^+_j\).
\end{enumerate}

An integration schema is represented by a diagram
\[
\J: \G \rightarrow Set(\Omega),
\]
defined within the category of interpretations and computable multi-morphisms, such that \(\J(i) = MD_i\) for every vertex \(i\) in \(\J\). Let \((D_i)\) be a family of diagrams used in the definition of integration schema \(\J\). The concept description \(\Omega(\J)\) defined by \(\J\) is given by the colimit of \(\J\):
\[
\Omega(\J) = colim_i \; MD_i = colim_i \; Lim \; M(D_i),
\]
where the colimit is computed as defined in \ref{colim}, i.e.,
\[
_{(coLim \; \J)(\ldots, \bar{x}_i, \ldots, \bar{x}_j, \ldots) = \ldots \otimes Lim \; M(D_i)(\bar{x}_i) \otimes \ldots \otimes Lim \; M(D_j)(\bar{x}_j) \otimes \ldots \otimes \bigvee_{f: MD_i \rightarrow MD_j \in \J} f(\bar{x}_i, \bar{x}_j).}
\]

\chapter{Reasoning about models of concepts}\label{reasoning}

The language \(\lambda\)-RL\((S)\) of \(\lambda\)-representable logic serves as a formalism for discussing structures that are \(\lambda\)-representable within a semiotic \((S,M)\). Fundamentally, it is a string-based modal logic, defined by a generative grammar where propositional variables are interpreted as diagrams belonging to the language associated with the sign system \(S\).

\(\lambda\)-RL\((S)\) is constructed using relations from \(Lang(S)\), modal operators such as limit, closure, interior, and the lifting of the monoidal logic connectives \(\otimes\), \(\Rightarrow\), \(\wedge\), and \(\vee\) to relations.

Every semiotic \((S,M)\) provides semantics for \(\lambda\)-RL\((S)\) through the truth-relation 
\[
g \models_\lambda \varphi,
\] 
which is defined as follows for every formula \(\varphi \in \lambda\)-RL\((S)\) and every concept description \(g\) in \((S,M)\):

\begin{enumerate}
    \item \(g \models_\lambda \varphi\) if and only if \(\varphi\) corresponds to the diagram \(D\) and \(\Gamma(g, MD) \geq \lambda\),
    \item \(g \models_\lambda [I]\varphi\) if and only if \(int(g) \models_\lambda \varphi\),
    \item \(g \models_\lambda [C]\varphi\) if and only if \(cl(g) \models_\lambda \varphi\).
\end{enumerate}

For any formulas \(\varphi_0\) and \(\varphi_1\) in \(\lambda\)-RL\((S)\), if 
\[
g \models_{\lambda_0} \varphi_0 \text{ and } g \models_{\lambda_1} \varphi_1,
\]
then:

\begin{enumerate}
    \item \(g \models_{\lambda_0 \otimes \lambda_1} (\varphi_0 \otimes \varphi_1)\),
    \item \(g \models_{\lambda_0 \Rightarrow \lambda_1} (\varphi_0 \Rightarrow \varphi_1)\),
    \item \(g \models_{\lambda_0 \wedge \lambda_1} (\varphi_0 \wedge \varphi_1)\), and
    \item \(g \models_{\lambda_0 \vee \lambda_1} (\varphi_0 \vee \varphi_1)\).
\end{enumerate}

Utilizing the structural compatibility between multi-morphism composition and diagram gluing, we have:

\begin{prop}
Given multi-morphisms \( g \) and \( h \) such that
\[
g \models_{\lambda_0} \varphi_0 \quad \text{and} \quad h \models_{\lambda_1} \varphi_1,
\] 
we can conclude that
\[
g \otimes h \models_{\lambda_0 \otimes \lambda_1} (\varphi_0 \otimes \varphi_1).
\]
\end{prop}

By extending the ML-algebra structure to the set of concept descriptions, we derive:
\begin{prop}
Given concept descriptions \( g \) and \( h \) in \( \oplus_i A_i \) such that
\[
g \models_{\lambda_0} \varphi \quad \text{and} \quad h \models_{\lambda_1} \varphi,
\] 
we have:
\begin{enumerate}
    \item \( (g \otimes h) \models_{\lambda_0 \otimes \lambda_1} \varphi \),
    \item \( (g \Rightarrow h) \models_{\lambda_0 \Rightarrow \lambda_1} \varphi \),
    \item \( (g \wedge h) \models_{\lambda_0 \wedge \lambda_1} \varphi \), and
    \item \( (g \vee h) \models_{\lambda_0 \vee \lambda_1} \varphi \).
\end{enumerate}
\end{prop}

Given a set of relations \( U \) from \( \text{Lang}_R(S) \) and \( \varphi \) a relation in \( \lambda \)-RL\((S)\), we define:
\[
U \models_\lambda \varphi \quad \text{iff} \quad \text{ans}_\lambda(U) \models_\lambda \varphi.
\]
Utilizing Theorem \ref{soundness}, we can state:

\begin{thm}[Soundness]
Given a set of relations \( U \) from \( \text{Lang}_R(S) \) and \( \varphi \) a relation in \( \lambda \)-RL\((S)\),
\[
\text{if } U \vdash_\lambda \varphi \text{ then } U \models_\lambda \varphi 
\]
for every \( \lambda \).
\end{thm}

Naturally, completeness does not hold universally. If \( U \models_\lambda \varphi \), it does not necessarily mean that we can prove \( U \vdash_\lambda \varphi \) by deduction.

\chapter{Conclusions and Future Work}\label{conclusions}

The use of semiotics appears to be an appropriate formalism for defining both the syntax and semantics of graphic languages. This is especially true when these languages are based on a library of functional components interpreted as relations evaluated in multi-valued logic. Such an approach simplifies the integration of knowledge expressed in different languages and facilitates the inference of new knowledge.


\bibliographystyle{amsplain}
\bibliography{main}
\end{document}